\documentclass{article} 
\usepackage{configuration/iclr2026_conference,times}

\usepackage{hyperref}
\usepackage{url}
\usepackage{xcolor}

\definecolor{color1}{HTML}{406058}
\definecolor{color2}{HTML}{78A0C8}
\definecolor{color3}{HTML}{3850A0}
\definecolor{color4}{HTML}{6860A8}
\definecolor{color5}{HTML}{A03850}
\definecolor{color6}{HTML}{AE7694}
\definecolor{color7}{HTML}{8488B4}
\hypersetup{
    colorlinks=true,
    linkcolor=black,
    citecolor=color1,
    urlcolor=color2,
    pdfborder={0 0 0}
}
\usepackage[T1]{fontenc}

\title{Self-Adversarial One Step Generation via Condition Shifting}

\author{%
  Deyuan Liu$^{1*}$ \,
  Peng Sun$^{2,1}\thanks{Equal contribution}$ \,
  Yansen Han$^{2,1}$ \,
  Zhenglin Cheng$^{3,2,1}$ \,
  Chuyan Chen$^{4,1}$ \,
  \textbf{Tao Lin}$^{1,}\thanks{Corresponding author.}$ \\ [0.5ex] 
  $^{1}$Westlake University \quad
  $^{2}$Zhejiang University \quad
  $^{3}$Shanghai Innovation Institute \quad
  $^{4}$Peking University \\
}

\usepackage{amsmath,amssymb,amsthm}

\providecommand{\R}{\mathbb{R}} %

\providecommand{\E}{{\mathbb E}}
\providecommand{\E}[1]{{\mathbb E}\left.#1\right. }        %
\providecommand{\EEb}[2]{{\mathbb E}_{#1}\left[#2\right] } %

\providecommand{\dm}{\mathrm{d}}

\providecommand{\cc}{\mathbf{c}}

\renewcommand{\ss}{\mathbf{s}}

\providecommand{\vv}{\mathbf{v}}

\providecommand{\xx}{\mathbf{x}}
\providecommand{\yy}{\mathbf{y}}
\providecommand{\zz}{\mathbf{z}}

\providecommand{\mI}{\mathbf{I}}

\providecommand{\mf}{\boldsymbol{f}}

\providecommand{\mmF}{\boldsymbol{F}}

\providecommand{\mtheta}{\boldsymbol{\theta}}

\providecommand{\cC}{\mathcal{C}}

\providecommand{\cG}{\mathcal{G}}

\providecommand{\cL}{\mathcal{L}}

\providecommand{\cN}{\mathcal{N}}

\newenvironment{talign*}
{\csname align*\endcsname}
{\endalign}

\usepackage[utf8]{inputenc}         %
\usepackage[T1]{fontenc}            %
\usepackage{url}                    %
\usepackage{booktabs}               %
\usepackage{amsfonts}               %
\usepackage{nicefrac}               %
\usepackage{microtype}              %
\usepackage{xcolor}                 %
\usepackage{algorithm}
\usepackage{algorithmic}
\usepackage{graphicx}
\usepackage{subcaption}
\usepackage[flushleft]{threeparttable}
\usepackage{float}
\usepackage{multirow}
\usepackage{xspace}
\usepackage{natbib}
\usepackage{enumitem}
\usepackage[font=small]{caption}
\usepackage{autobreak}
\usepackage{sidecap}
\usepackage{wrapfig}
\usepackage{bbding}
\usepackage[toc, page, header]{appendix}
\usepackage{tikz}
\usepackage{xcolor}
\usepackage{pifont}
\usepackage{mdframed}
\usepackage{colortbl}
\usepackage{mathrsfs}
\usepackage{footmisc}

\definecolor{coral}{RGB}{255,127,80}
\definecolor{darkgreen}{RGB}{0,100,0}
\definecolor{darkyellow}{RGB}{204,153,0}
\definecolor{salmon}{RGB}{250,128,114}
\definecolor{champagne}{RGB}{247,231,206}
\definecolor{kleinblue}{RGB}{0,47,167}
\definecolor{bauhiniapurple}{RGB}{193,84,193}
\definecolor{darkred}{RGB}{150,0,0}

\newcommand{\darkredtext}[1]{{\color{darkred}#1}}

\newcommand{\thmref}[1]{\hyperref[#1]{\darkredtext{Thm.~\ref*{#1}}}}
\newcommand{\defref}[1]{\hyperref[#1]{\darkredtext{Def.~\ref*{#1}}}}
\newcommand{\propref}[1]{\hyperref[#1]{\darkredtext{Prop.~\ref*{#1}}}}
\newcommand{\assumpref}[1]{\hyperref[#1]{\darkredtext{Assump.~\ref*{#1}}}}
\newcommand{\remarkref}[1]{\hyperref[#1]{\darkredtext{Rem.~\ref*{#1}}}}
\newcommand{\hypref}[1]{\hyperref[#1]{\darkredtext{Hyp.~\ref*{#1}}}}
\newcommand{\conjref}[1]{\hyperref[#1]{\darkredtext{Conj.~\ref*{#1}}}}
\newcommand{\lemref}[1]{\hyperref[#1]{\darkredtext{Lem.~\ref*{#1}}}}
\newcommand{\corref}[1]{\hyperref[#1]{\darkredtext{Cor.~\ref*{#1}}}}
\newcommand{\noteref}[1]{\hyperref[#1]{\darkredtext{Nota.~\ref*{#1}}}}
\newcommand{\claimref}[1]{\hyperref[#1]{\darkredtext{Clm.~\ref*{#1}}}}
\newcommand{\obsref}[1]{\hyperref[#1]{\darkredtext{Obs.~\ref*{#1}}}}

\newcommand{\transparentteal}[1]{%
  \tikz[baseline=(X.base)] \node[fill=teal, fill opacity=0.1, text opacity=1, inner sep=2pt, outer sep=0pt] (X) {#1};%
}
\newcommand{\figref}[1]{\hyperref[#1]{\transparentteal{Figure~\ref*{#1}}}}

\newcommand{\transparentdarkgreen}[1]{%
  \tikz[baseline=(X.base)] \node[fill=darkgreen, fill opacity=0.1, text opacity=1, inner sep=2pt, outer sep=0pt] (X) {#1};%
}
\newcommand{\tabref}[1]{\hyperref[#1]{\transparentdarkgreen{Table~\ref*{#1}}}}

\newcommand{\transparentdarkyellow}[1]{%
  \tikz[baseline=(X.base)] \node[fill=darkyellow, fill opacity=0.1, text opacity=1, inner sep=2pt, outer sep=0pt] (X) {#1};%
}
\newcommand{\secref}[1]{\hyperref[#1]{\transparentdarkyellow{Section~\ref*{#1}}}}

\newcommand{\transparentcoral}[1]{%
  \tikz[baseline=(X.base)] \node[fill=coral, fill opacity=0.1, text opacity=1, inner sep=2pt, outer sep=0pt] (X) {#1};%
}
\newcommand{\appref}[1]{\hyperref[#1]{\transparentcoral{Appendix~\ref*{#1}}}}

\newcommand{\transparentkleinblue}[1]{%
  \tikz[baseline=(X.base)] \node[fill=kleinblue, fill opacity=0.1, text opacity=1, inner sep=2pt, outer sep=0pt] (X) {#1};%
}
\newcommand{\algref}[1]{\hyperref[#1]{\transparentkleinblue{Algorithm~\ref*{#1}}}}

\newcommand{\transparentbauhiniapurple}[1]{%
  \tikz[baseline=(X.base)] \node[fill=bauhiniapurple, fill opacity=0.1, text opacity=1, inner sep=2pt, outer sep=0pt] (X) {#1};%
}
\newcommand{\equref}[1]{\hyperref[#1]{\transparentbauhiniapurple{Equation~\ref*{#1}}}}

\newtheoremstyle{custom}
{1pt} 
{1pt} 
{\itshape} 
{} 
{\bfseries} 
{} 
{ } 
{\thmname{#1} \thmnumber{#2} \thmnote{(#3)} . } 

\theoremstyle{custom}

\newtheorem{innerdefinition}{Definition}
\newtheorem{innerproposition}{Proposition}
\newtheorem{innerassumption}{Assumption}
\newtheorem{innerremark}{Remark}
\newtheorem{innertheorem}{Theorem}
\newtheorem{innerhypothesis}{Hypothesis}
\newtheorem{innerconjecture}{Conjecture}
\newtheorem{innerlemma}{Lemma}
\newtheorem{innercorollary}{Corollary}
\newtheorem{innerexample}{Example}
\newtheorem{innernotation}{Notation}
\newtheorem{innerclaim}{Claim}
\newtheorem{innerproblem}{Problem}

\newtheorem{innerobservation}{Observation}

\newmdenv[
  backgroundcolor=gray!10,
  linecolor=gray!100,
  linewidth=0.01pt,
  skipabove=2pt,
  skipbelow=2pt,
  innertopmargin=10pt,
  innerbottommargin=5pt,
  innerleftmargin=5pt,
  innerrightmargin=5pt,
]{definitionframe}

\newmdenv[
  backgroundcolor=blue!10,
  linecolor=blue!100,
  linewidth=0.01pt,
  skipabove=2pt,
  skipbelow=2pt,
  innertopmargin=10pt,
  innerbottommargin=5pt,
  innerleftmargin=5pt,
  innerrightmargin=5pt,
]{propositionframe}

\newmdenv[
  backgroundcolor=green!10,
  linecolor=green!100,
  linewidth=0.01pt,
  skipabove=2pt,
  skipbelow=2pt,
  innertopmargin=10pt,
  innerbottommargin=5pt,
  innerleftmargin=5pt,
  innerrightmargin=5pt,
]{assumptionframe}

\newmdenv[
  backgroundcolor=yellow!10,
  linecolor=yellow!100,
  linewidth=0.01pt,
  skipabove=2pt,
  skipbelow=2pt,
  innertopmargin=10pt,
  innerbottommargin=5pt,
  innerleftmargin=5pt,
  innerrightmargin=5pt,
]{remarkframe}

\newmdenv[
  backgroundcolor=red!10,
  linecolor=red!100,
  linewidth=0.01pt,
  skipabove=2pt,
  skipbelow=2pt,
  innertopmargin=10pt,
  innerbottommargin=5pt,
  innerleftmargin=5pt,
  innerrightmargin=5pt,
]{theoremframe}

\newmdenv[
  backgroundcolor=purple!10,
  linecolor=purple!100,
  linewidth=0.01pt,
  skipabove=2pt,
  skipbelow=2pt,
  innertopmargin=10pt,
  innerbottommargin=5pt,
  innerleftmargin=5pt,
  innerrightmargin=5pt,
]{hypothesisframe}

\newmdenv[
  backgroundcolor=orange!10,
  linecolor=orange!100,
  linewidth=0.01pt,
  skipabove=2pt,
  skipbelow=2pt,
  innertopmargin=10pt,
  innerbottommargin=5pt,
  innerleftmargin=5pt,
  innerrightmargin=5pt,
]{conjectureframe}

\newmdenv[
  backgroundcolor=cyan!10,
  linecolor=cyan!100,
  linewidth=0.01pt,
  skipabove=2pt,
  skipbelow=2pt,
  innertopmargin=10pt,
  innerbottommargin=5pt,
  innerleftmargin=5pt,
  innerrightmargin=5pt,
]{lemmaframe}

\newmdenv[
  backgroundcolor=magenta!10,
  linecolor=magenta!100,
  linewidth=0.01pt,
  skipabove=2pt,
  skipbelow=2pt,
  innertopmargin=10pt,
  innerbottommargin=5pt,
  innerleftmargin=5pt,
  innerrightmargin=5pt,
]{corollaryframe}

\newmdenv[
  backgroundcolor=lime!10,
  linecolor=lime!100,
  linewidth=0.01pt,
  skipabove=2pt,
  skipbelow=2pt,
  innertopmargin=10pt,
  innerbottommargin=5pt,
  innerleftmargin=5pt,
  innerrightmargin=5pt,
]{exampleframe}

\newmdenv[
  backgroundcolor=pink!10,
  linecolor=pink!100,
  linewidth=0.01pt,
  skipabove=2pt,
  skipbelow=2pt,
  innertopmargin=10pt,
  innerbottommargin=5pt,
  innerleftmargin=5pt,
  innerrightmargin=5pt,
]{notationframe}

\newmdenv[
  backgroundcolor=violet!10,
  linecolor=violet!100,
  linewidth=0.01pt,
  skipabove=2pt,
  skipbelow=2pt,
  innertopmargin=10pt,
  innerbottommargin=5pt,
  innerleftmargin=5pt,
  innerrightmargin=5pt,
]{claimframe}

\newmdenv[
  backgroundcolor=salmon!10,
  linecolor=salmon!100,
  linewidth=0.01pt,
  skipabove=2pt,
  skipbelow=2pt,
  innertopmargin=10pt,
  innerbottommargin=5pt,
  innerleftmargin=5pt,
  innerrightmargin=5pt,
]{problemframe}

\newmdenv[
  backgroundcolor=lavender!10,
  linecolor=lavender!100,
  linewidth=0.01pt,
  skipabove=2pt,
  skipbelow=2pt,
  innertopmargin=10pt,
  innerbottommargin=5pt,
  innerleftmargin=5pt,
  innerrightmargin=5pt,
]{observationframe}

\newenvironment{proposition}
{\begin{propositionframe}\begin{innerproposition}}
{\end{innerproposition}\end{propositionframe}}

\newenvironment{theorem}
{\begin{theoremframe}\begin{innertheorem}}
{\end{innertheorem}\end{theoremframe}}

\newenvironment{lemma}
{\begin{lemmaframe}\begin{innerlemma}}
{\end{innerlemma}\end{lemmaframe}}

\newenvironment{corollary}
{\begin{corollaryframe}\begin{innercorollary}}
{\end{innercorollary}\end{corollaryframe}}

\usepackage[textwidth=3.5cm]{todonotes}

\newcommand{\method}{\textsc{APEX}\xspace}

\newcommand{\sg}{\mathop{{\color{red}\operatorname{sg}}}\nolimits}

\usepackage{multirow}
\usepackage{makecell}
\usepackage{colortbl}

\usepackage{bm}

\definecolor{ForestGreen}{RGB}{34,139,34}
\definecolor{OrangeRed}{RGB}{255,69,0}

\iclrfinalcopy 
\begin{document}

\maketitle

\begin{abstract}
    The push for efficient text to image synthesis has moved the field toward one step sampling, yet existing methods still face a three way tradeoff among fidelity, inference speed, and training efficiency.
    Approaches that rely on external discriminators can sharpen one step performance, but they often introduce training instability, high GPU memory overhead, and slow convergence, which complicates scaling and parameter efficient tuning.
    In contrast, regression based distillation and consistency objectives are easier to optimize, but they typically lose fine details when constrained to a single step.
    We present \textcolor{color3}{\textbf{\method}}, built on a key theoretical insight: adversarial correction signals can be extracted \emph{endogenously} from a flow model through \emph{condition shifting}.
    Using a transformation creates a shifted condition branch whose velocity field serves as an independent estimator of the model's current generation distribution, yielding a gradient that is provably GAN aligned, replacing the sample dependent discriminator terms that cause gradient vanishing.
    This discriminator free design is architecture preserving, making \method a plug and play framework compatible with both full parameter and LoRA based tuning.
    Empirically, our \textcolor{color3}{\textbf{0.6B}} model surpasses FLUX-Schnell \textcolor{color3}{\textbf{12B}} (20$\times$ more parameters) in one step quality.
    With LoRA tuning on Qwen-Image \textcolor{color3}{\textbf{20B}}, \method reaches a GenEval score of \textcolor{color3}{\textbf{0.89}} at \textcolor{color3}{\textbf{NFE\,{=}\,1}} in \textcolor{color3}{\textbf{6 hours}}, surpassing the original 50-step teacher (0.87) and providing a \textcolor{color3}{\textbf{15.33$\times$}} inference speedup.
    Code is available \href{https://github.com/LINs-lab/APEX}{here}.
\end{abstract}

\begin{figure*}[h]
    \centering
    \includegraphics[width=\linewidth]{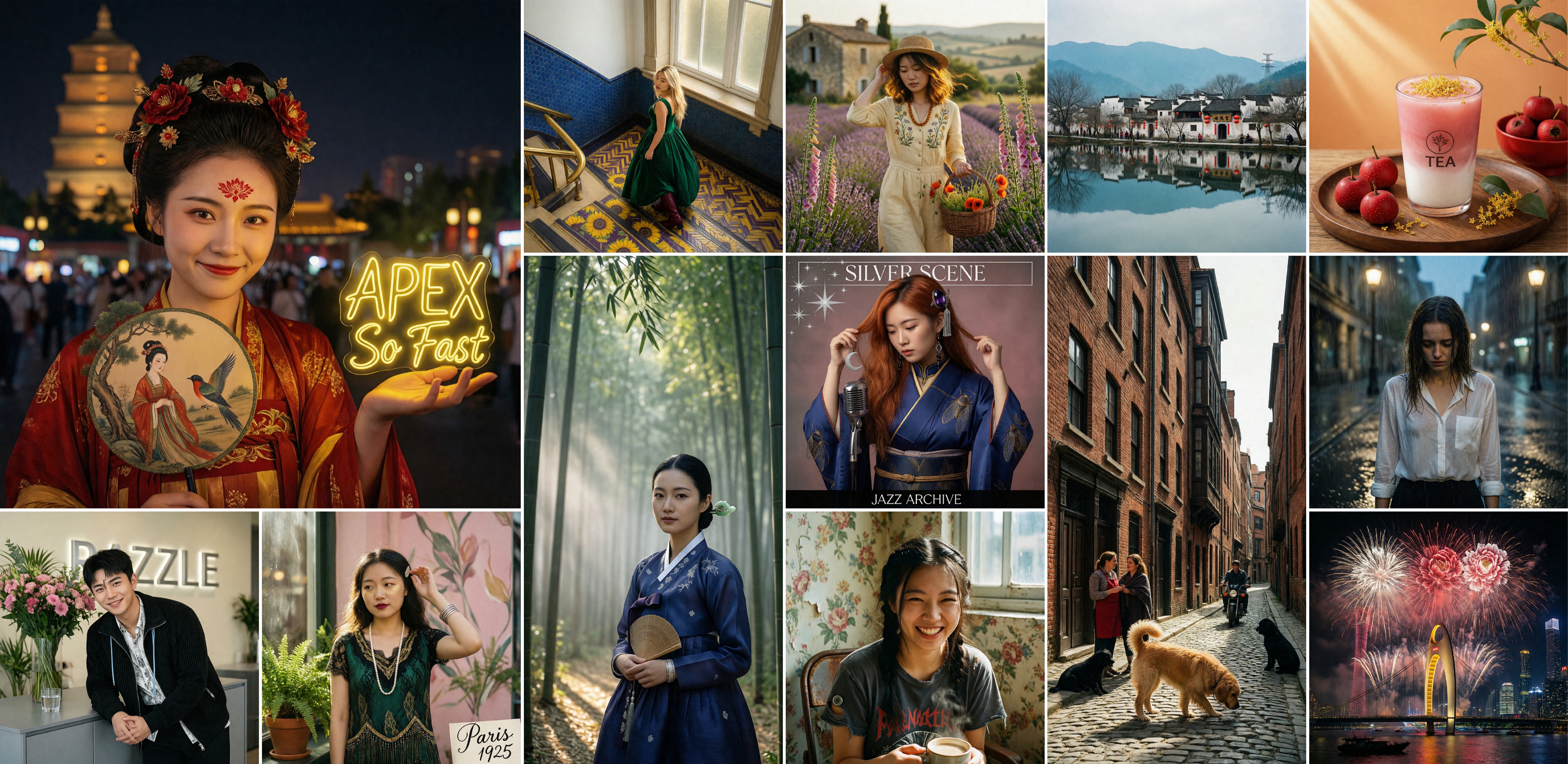}
    \caption{\textbf{An overview of generated images.}}
    \label{fig:overview}
\end{figure*}

\section{Introduction}
Continuous generative models now achieve strong fidelity across domains, from photorealistic image synthesis~\citep{dhariwal2021diffusion,karras2024analyzing} to video generation~\citep{ho2022video,chen2025sana}.
This progress is largely driven by diffusion models~\citep{ho2020denoising,dhariwal2021diffusion} and flow matching frameworks~\citep{lipman2022flow,ma2024sit}, which sample by integrating a Probability Flow Ordinary Differential Equation (PF-ODE), from noise to data~\citep{song2020score}.
The same iterative paradigm also dominates inference cost: multi step integration often requires tens of function evaluations and can be prohibitively expensive~\citep{karras2024analyzing,nichol2021improved}, motivating sustained interest in one step synthesis~\citep{song2023consistency,salimans2022progressive,yin2024improved}.
\looseness=-1

Achieving number of function evaluations \textcolor{color3}{(NFE) = 1} at high resolution exposes a persistent trilemma among generation quality, inference efficiency, and training efficiency~\citep{song2023consistency,lu2024simplifying,yin2024improved,sauer2024fast}.
External adversarial components like a discriminator or auxiliary critic can improve one step realism, but they often hurt scalability by introducing training instability and additional system overhead~\citep{yin2024improved,kim2023refining,zheng2025direct}.
This overhead becomes especially costly when scaling pretrained backbones or doing parameter efficient tuning.
In contrast, regression based distillation~\citep{yin2024one} and consistency style objectives~\citep{song2023consistency,sun2025anystep} are typically easier to optimize, yet they often struggle to match adversarial realism in one step, especially for high frequency textures and fine details~\citep{song2023consistency,lu2024simplifying,geng2025mean,sun2025unified}.
Complementary to these lines, a recent work, TwinFlow~\citep{cheng2025twinflow}, also explores \emph{self adversarial} methods that build adversarial signals by model itself.

\begin{mdframed}[
    hidealllines=true,
    backgroundcolor=color4!10,
    innerleftmargin=20pt,
    innerrightmargin=20pt,
    innertopmargin=1pt,
    innerbottommargin=8pt,
    roundcorner=1pt,
    linewidth=1pt,
    linecolor=color3
  ]
  \noindent
  \begin{center}
    \textbf{$\mathbb{Q}$: How can we obtain GAN level one step fidelity at \textcolor{color3}{\textbf{NFE=1}} \emph{without} an external discriminator, while remaining scalable to large pretrained backbones and parameter efficient tuning?}
  \end{center}
\end{mdframed}

\textbf{Our approach.} We introduce \textcolor{color3}{\textbf{\method}}, built on a key theoretical insight: the adversarial correction signal that GANs derive from an external discriminator can be generated \emph{endogenously} within a flow model by separating real and fake scores in \emph{condition space}.
Concretely, \method constructs a \emph{shifted condition} $\cc_{\text{fake}} = \mathbf{A}\cc + \mathbf{b}$ via an affine transformation and trains the model under $\cc_{\text{fake}}$ to fit trajectories toward its current one step outputs.
This shifted condition branch provides an independent estimator of the fake distribution's velocity field, enabling the main branch under the true condition $\cc$ to receive an adversarial correction signal.

We also show that \method admits a GAN aligned gradient interpretation.
Under the Optimal Transport path, the score velocity duality connects velocity regression to score matching, allowing us to express \method's update in the same canonical score difference form as GANs.
Crucially, while GANs weight the score difference using sample dependent discriminator terms such as $D^*$ or $1-D^*$, \method corresponds to a \emph{constant weight} with a \emph{target score} induced by condition shifting.
This yields stable, discriminator free signals while preserving an adversarial force toward photorealism.

\textbf{Our main contributions are:}
\begin{enumerate}[label=\alph*., nosep, leftmargin=16pt]
  \item \textbf{Theoretical Foundation --- GAN Aligned Gradient with Constant Weight:} We establish a formal gradient level equivalence between \method and GAN dynamics via score velocity duality (\secref{subsec:gan_connection}), proving that \method's training gradient takes the canonical score difference form $(\ss_\theta - \ss_{\text{mix}}) \cdot \partial\xx_t/\partial\theta$ with \emph{constant} weight $w \equiv 1$ and an implicit \emph{score interpolation} target $\ss_{\text{mix}} = (1{-}\lambda)\ss_{\text{data}} + \lambda\ss_{\text{fake}}$, connecting \method to Fisher divergence minimization and explaining why it avoids the gradient instability of sample dependent discriminator weights.
  \item \textbf{Methodology --- Self Adversarial Framework via Condition Shifting:} We propose \textcolor{color3}{\textbf{\method}}, a discriminator free framework using an affine condition shift $\cc_{\text{fake}}=\mathbf{A}\cc+\mathbf{b}$ to generate an endogenous adversarial signal for one step, high resolution text to image synthesis. This design makes \method a plug and play replacement fully compatible with LoRA and other parameter efficient fine tuning pipelines.
  \item \textbf{SOTA Performance and Scalability:} Our \textcolor{color3}{\textbf{0.6B}} model surpasses FLUX-Schnell \textcolor{color3}{\textbf{12B}}  in one step quality at \textcolor{color3}{\textbf{NFE=1}}.
        With LoRA tuning on Qwen-Image \textcolor{color3}{\textbf{20B}}, \method reaches GenEval \textcolor{color3}{\textbf{0.89}} in \textcolor{color3}{\textbf{6 hours}}, surpassing the original 50 step teacher model (0.87).
\end{enumerate}

\section{Preliminaries}
\label{sec:preliminaries}

\paragraph{Continuous Generative Models.}
\label{sec:continuous-path}

Diffusion generative models~\citep{ho2020denoising,song2020score} and flow matching models~\citep{lipman2022flow} both describe a continuous time evolution that transports a simple prior $p(\zz) = \cN(\mathbf{0}, \mI)$ toward a complex data distribution $p_{\text{data}}(\xx)$.
While classical diffusion is formulated as a stochastic forward noising process and a reverse time SDE, it admits an equivalent deterministic sampler given by the Probability Flow ODE (PF-ODE) associated with the same score field~\citep{song2020score}.
We define a time dependent random variable $\xx_t$, $t \in [0, 1]$, as a linear interpolant between noise $\zz$ and data $\xx$:
\begin{equation}
  \xx_t = \alpha(t)\zz + \gamma(t)\xx \,.
  \label{eq:xt_interpolant}
\end{equation}
Typically, we adopt the Optimal Transport (OT) path with $\alpha(t)=t,\ \gamma(t)=1-t$, which satisfies the boundary conditions $\xx_1=\zz$ for pure noise and $\xx_0=\xx$ for pure data.
This interpolation path induces a velocity field $\vv(\xx_t,t)$, defining the PF-ODE for sample generation:
\begin{equation}
  \frac{\mathrm{d}\xx_t}{\mathrm{d}t} = \vv(\xx_t,t) \,.
  \label{eq:pf_ode}
\end{equation}
Given an estimate of $\vv_t$, we can numerically integrate Eq.~\eqref{eq:pf_ode} from $t=1$ to $t=0$ using standard ODE solvers (e.g., Euler~\citep{karras2022edm}) to generate samples.
For conditional generation with condition $\cc$, flow matching trains a neural network $\mmF_{\mtheta}(\xx_t, t, \cc)$ to approximate a target velocity field.
Along the OT path, conditional velocity of a particular pair $(\xx,\zz)$ is defined as the time derivative:
\begin{equation}
  \frac{\mathrm{d}}{\mathrm{d}t}(t\zz + (1-t)\xx) = \zz - \xx \,.
\end{equation}
This quantity is an unbiased regression target; minimizing a squared error loss recovers the population optimal conditional mean $\vv^*(\xx_t, t)$.
The standard FM loss is:
\begin{equation}
  \cL_{\text{FM}}({\mtheta}) = \EEb{\xx_t,\, \zz,\, t}{ \| \mmF_{\mtheta}(\xx_t, t, \cc) - (\zz - \xx) \|^2 } \,,
  \label{eq:fm_loss}
\end{equation}
where the expectation is taken over the joint distribution of $(t, \xx, \zz)$, ensuring that $\mmF_{\mtheta}$ recovers the vector field $\vv^*$ as the conditional expectation of the per sample velocity targets $\zz - \xx$ given $\xx_t$.

\paragraph{Score Velocity Duality.}
\label{sec:score-velocity}

Under the OT path, the score function of any marginal density $p_t$ and its population optimal velocity field are related by (proof in~\appref{app:score_velocity}):
\begin{equation}
  \ss_t(\xx_t) = - \frac{\xx_t + (1-t)\vv^*(\xx_t, t)}{t} \,.
  \label{eq:score_velocity_duality}
\end{equation}
Here $\vv^*(\xx_t,t)$ denotes the OT induced \emph{conditional} velocity field.
This \emph{Score Velocity Duality} provides a bidirectional bridge between score functions and the velocity field parameterized by $\mmF_\theta$.
We will apply it in \secref{subsec:kl} to convert the KL divergence gradient into velocity space, and in \secref{subsec:gan_connection} to express \method's gradient in score space and connect it to GAN dynamics.

\paragraph{Few Step Generation.}
To overcome the inference latency caused by ODE numerical integration requiring tens of steps (NFE=50\textasciitilde250), a series of few step generation techniques have emerged~\citep{song2023consistency,lu2024simplifying,frans2024one,geng2025mean}.

\emph{\textbf{(i) Endpoint consistency methods}} like Consistency Models (CM)~\citep{song2023consistency} attempt to directly learn the mapping from ODE trajectory to origin.
A consistency function $\mf_{\mtheta}(\xx_t, t)$ is trained to satisfy the self consistency property: for any two points $t, t'$ on the same trajectory, $\mf_{\mtheta}(\xx_t, t) = \mf_{\mtheta}(\xx_{t'}, t') = \xx_0$.
This uses a first order Taylor expansion to approximate the trajectory integral.

\emph{\textbf{(ii) Higher order methods}} generalize this approach.
RCGM~\citep{sun2025anystep} shows that CM and MeanFlow~\citep{geng2025mean} are first order special cases ($N=1$) of a more general framework.
RCGM introduces $N$-th order recursive integral approximation, using future multi step trajectory information to more accurately estimate the current velocity field.

\emph{\textbf{(iii) Self adversarial methods.}}
TwinFlow~\citep{cheng2025twinflow} introduces twin trajectories by extending the time domain to $t\in[-1,1]$: the positive half maps noise to real data, while the negative half maps noise to the model's current fake data.
First, it trains the model on fake trajectories via:
\begin{equation}
  \cL_{\text{TF}} = \EEb{\xx_t,\, \zz,\, t}{ \| \mmF_{\mtheta}(\xx^{\text{fake}}_t, t) - (\zz - \xx^{\text{fake}}) \|^2 } \,.
  \label{eq:twinflow_adv_loss}
\end{equation}
Then minimizes the velocity discrepancy between the real score $+t$ and the fake score $-t$ via a rectification loss, steering generation toward higher fidelity without an external discriminator:
\begin{small}
  \begin{equation}
    \cL_{\text{TF-rect}} = \EEb{\xx_t,\, \zz,\, t}{ \left\| \mmF_{\mtheta}(\xx_t, t) - \sg\!\left( \mmF_{\mtheta}(\xx_t, -t) + \Delta \vv \right) \right\|^2 } \,,
    \label{eq:twinflow_rectification}
  \end{equation}
\end{small}%
where $\Delta \vv$ accounts for the gap between real and fake velocity targets.
The two branches are separated by the \emph{sign} of the time input $t$ vs.\ $-t$; \method achieves the same structure via a simpler separation in \emph{condition space} $\cc$ vs.\ $\cc_{\text{fake}}$, as developed in \secref{sec:apex}.

\paragraph{GAN Dynamics and Score Difference Gradients.}
GAN generator updates take the form of a score difference signal $(\ss_\theta(\xx) - \ss_{\text{data}}(\xx))$ modulated by a sample dependent weight from the discriminator; we review this structure, as \method's gradient admits the same form; see \secref{subsec:gan_connection}.
Let $p_{\mtheta}(\xx)$, $p_{\text{data}}(\xx)$, $D(\xx)$ be the generator, data, and discriminator distributions, with $\ss(\xx) := \nabla_{\xx} \log p(\xx)$.
In the analysis below, $\xx$ denotes clean samples; in \secref{sec:apex} we generalize to time marginal scores $\ss_t(\xx_t)$.
Under the optimal discriminator $D^*(\xx) = \nicefrac{p_{\text{data}}(\xx)}{p_{\text{data}}(\xx) + p_{\mtheta}(\xx)}$~\citep{mohamed2016learning,goodfellow2014generative}, both GAN variants yield a generator gradient of the unified form:
\begin{equation}
  \nabla_{\mtheta} \cL_{\text{GAN}} \propto \EEb{\xx \sim p_{\mtheta}}{ w(\xx) \cdot (\ss_{\mtheta}(\xx) - \ss_{\text{data}}(\xx)) \cdot \frac{\partial \xx}{\partial \mtheta} } \,,
  \label{eq:gan_grad}
\end{equation}
where $w(\xx) = D^*(\xx)$ or $1{-}D^*(\xx)$ for the saturating and non saturating variants respectively.
This \emph{sample dependent} weight encodes discriminator confidence: it vanishes when samples are highly realistic, causing gradient vanishing, and varies unpredictably across training, introducing instability.
In \secref{subsec:gan_connection} we show that \method's gradient takes exactly this score difference form but with a \emph{constant} weight $w\equiv1$, achieving adversarial level correction without a discriminator.

\section{\method}
\label{sec:apex}

\method achieves discriminator free, architecture preserving, self adversarial training by separating the real and fake scores in \emph{condition space} rather than time space: an affine transformation $\cc_{\text{fake}} = \mathbf{A}\cc + \mathbf{b}$ creates the fake score entirely within $t\in[0,1]$, requiring no modification to time embeddings or model architecture.
We develop the method in three stages:
\begin{enumerate}[label=(\roman*), nosep, leftmargin=18pt]
  \item \textbf{Building the fake reference}: define $\cc_{\text{fake}}$ and the fake sample $\xx^{\text{fake}}$; train the shifted condition via $\cL_{\text{fake}}$ so that $\vv_{\text{fake}}$ serves as an independent estimator of $p_{\text{fake}}$'s velocity field.
  \item \textbf{KL descent and practical loss}: show that the velocity discrepancy $\Delta\vv_{\text{\method}}$ is the exact descent direction on $D_{\text{KL}}(p_{\text{fake}}\|p_{\text{real}})$; convert it into the consistency loss $\cL_{\text{mix}}$ via endpoint equivalence.
  \item \textbf{GAN aligned gradient structure}: analyze the gradient in score space and show it is a GAN style score difference update with weight $w\equiv1$, connecting to Fisher divergence minimization.
\end{enumerate}
\looseness=-1

\subsection{Building the Adversarial Reference via Condition Shifting} \label{subsec:key_concepts}

\paragraph{Condition Space as the Separation Dimension.}
The two branch self adversarial structure requires a signal that distinguishes the real score from the fake score.
TwinFlow uses the sign of the time input $t$ vs.\ $-t$ for this purpose; \method instead uses the \emph{condition input} $\cc$ vs.\ $\cc_{\text{fake}}$.
Both achieve the same structure, but the condition space choice means the time domain, positional encodings, and time scheduling of any pretrained backbone remain completely unchanged, making \method a plug and play replacement that is fully compatible with LoRA and other parameter efficient fine tuning pipelines without any adaptation of time embedding.

\paragraph{Condition Space Shifting and the Fake Sample.}
In particular, we use the OT interpolant in Eq.~\eqref{eq:xt_interpolant} with $\alpha(t)=t$ and $\gamma(t)=1-t$, so that $\xx_1=\zz$ and $\xx_0=\xx$.
We denote the conditional velocity field by $\vv_{\mtheta}(\xx_t, t, \cc)$, parameterized by a neural network $\mmF_{\mtheta}(\xx_t, t, \cc)$.
We denote $\sg(\cdot)$ as the stop gradient operator.
Unless otherwise specified, all flows share the same interpolant family $\alpha(t),\gamma(t)$ and time weighting $\omega(t)$.
We introduce a fake condition $\cc_{\text{fake}}$, obtained through Self Condition Shifting of the original condition $\cc$:
\begin{equation}
  \cc_{\text{fake}} = \mathbf{A}\cc + \mathbf{b} \,,
  \label{eq:cond_shift}
\end{equation}
where $\mathbf{A}$ and $\mathbf{b}$ can be learnable parameter matrices/vectors or preset transformations.

\textbf{Why affine shifting?}
The self adversarial design requires two properties of $\cc_{\text{fake}}$: (i) it must be \emph{sufficiently distinct} from $\cc$ so that the network's internal representations under the two conditions decouple, allowing $\vv_{\text{fake}}$ to serve as an independent estimator of $p_{\text{fake}}$'s velocity; and (ii) it must remain \emph{within the pretrained condition embedding space} so that the network can produce semantically coherent outputs.
An affine map $\cc_{\text{fake}} = \mathbf{A}\cc + \mathbf{b}$ is the most general linear class of transformations satisfying both: it preserves the algebraic structure of the embedding space while enabling strong representational decoupling when $\mathbf{A}$ reverses or attenuates the condition's semantic direction.
In particular, negative scaling $\mathbf{A} = -a\mI$, $a > 0$ approximately inverts the condition embedding, creating a maximally contrastive branch that is consistent with our ablation finding that $a \in \{-1.0, -0.5\}$ yields the most robust performance in \tabref{tab:condshift}.

\paragraph{Self Adversarial Objective.}
\method's first stage trains the shifted condition branch to become an independent velocity estimator of the model's current generation distribution $p_{\text{fake}}$.
We require the model to reconstruct its currently generated outputs when receiving the shifted condition $\cc_{\text{fake}}$.
Under the OT path, we define an endpoint predictor that maps a velocity estimate at $(\xx_t,t)$ to its implied clean sample:
\begin{equation}
  \mf^{\xx}(\mmF, \xx_t, t) := \xx_t - t \cdot \mmF \,.
  \label{eq:endpoint_predictor}
\end{equation}
Given a noisy sample $\xx_t$ at time $t$ along the OT path in Eq.~\eqref{eq:xt_interpolant}, the model's implied clean data estimate under the real condition $\cc$ is:
\begin{equation}
  \xx^{\text{fake}} = \mf^{\xx}\!\left(\mmF_{\mtheta}(\xx_t, t, \cc),\, \xx_t,\, t\right) = \xx_t - t \cdot \mmF_{\mtheta}(\xx_t, t, \cc) \,.
  \label{eq:one_step_fake}
\end{equation}
When the model is imperfect, $\xx^{\text{fake}}$ deviates from the true $\xx$, capturing the model's current generation error.
We train the network under the shifted condition $\cc_{\text{fake}}$ to fit the trajectory toward $\xx^{\text{fake}}$.
Construct fake trajectory: $\xx^{\text{fake}}_t = \alpha(t)\zz + \gamma(t)\xx^{\text{fake}}$.
The fake flow loss is defined as:
\begin{small}
  \begin{equation}
    \cL_{\text{fake}}(\mtheta) = \EEb{\xx_t,\, \zz,\, t}{ \| \mmF_{\mtheta}(\xx^{\text{fake}}_t, t, \cc_{\text{fake}}) - (\zz - \xx^{\text{fake}}) \|^2 } \,.
    \label{eq:apex_fake_loss}
  \end{equation}
\end{small}%
Concretely, ${\partial \xx^{\text{fake}}}/{\partial \mtheta} = -t\cdot{\partial \mmF_{\mtheta}(\xx_t, t, \cc)}/{\partial \mtheta}$, so $\cL_{\text{fake}}$ simultaneously trains the $\cc_{\text{fake}}$ branch and injects a direct adversarial gradient into $\mmF_{\mtheta}(\cdot,\cdot,\cc)$.
The stop gradient in \method is applied separately in $\cL_{\text{cons}}$, where $\vv_{\text{fake}} := \sg\!\left(\mmF_{\mtheta}(\xx_t, t, \cc_{\text{fake}})\right)$ serves as a correction reference.
When $\cL_{\text{fake}}$ is minimized, $\vv_{\text{fake}}(\cdot,\cdot,\cc_{\text{fake}})$ approximates the velocity field of the fake distribution $p_{\text{fake}}$.
By training $\vv_{\text{fake}}$ on \emph{fake sample trajectories} $\xx^{\text{fake}}_t$, we obtain an estimator of $p_{\text{fake}}$'s velocity.
Second, we show how this independence is exploited to construct a KL descent signal.

\subsection{From Velocity Discrepancy to KL Descent and Practical Loss}

\label{subsec:kl}
\paragraph{KL Gradient in Velocity Space.}
Let $p_{\text{fake}}(\xx|\cc) := p_{\mtheta}(\xx|\cc)$ denote the model's current generation distribution and $p_{\text{real}}(\xx|\cc) := p_{\text{data}}(\xx|\cc)$ the true data distribution.
Our ultimate goal is to close the gap between $p_{\text{fake}}$ and $p_{\text{real}}$ by minimizing KL divergence $\min_{\mtheta} D_{\text{KL}}(p_{\text{fake}} \| p_{\text{real}})$.
The gradient of the KL divergence between $p_{\text{fake}}$ and $p_{\text{real}}$ admits a score difference form:
\looseness=-1
\begin{small}
  \begin{equation}
    \footnotesize
    \nabla_{\mtheta} D_{\text{KL}} \!=\! \EEb{\xx_t,\, \zz,\, t}{ (\nabla_{\xx_t} \log p_{\text{fake}}(\xx_t) \!-\! \nabla_{\xx_t} \log p_{\text{real}}(\xx_t)) \!\cdot \frac{\partial \xx_t}{\partial {\mtheta}} } \,.
    \label{eq:apex_kl_grad}
  \end{equation}
\end{small}%
Here, $\ss_t(\xx_t) := \nabla_{\xx_t} \log p_t(\xx_t)$ is the score function of the marginal density $p_t$ at time $t$.
We use the shorthand $\vv_{\text{data}}(\xx_t) := (\zz-\xx)$ for the supervised FM target velocity, and distinguish the two velocity fields by their gradient status:
\begin{equation}
  \begin{split}
    \vv_{\text{fake}}(\xx_t,t,\cc_{\text{fake}}) & := \sg\big(\mmF_{\mtheta}(\xx_t, t, \cc_{\text{fake}})\big) \,, \quad \vv_{\mtheta}(\xx_t, t, \cc) := \mmF_{\mtheta}(\xx_t, t, \cc) \,.
  \end{split}
  \label{eq:apex_v_def}
\end{equation}
By invoking the Score Velocity Duality defined in Eq.~\eqref{eq:score_velocity_duality}, we can analytically map the aforementioned velocity fields into the score space.
This transformation yields the following induced score for both the original and fake signal:
\begin{equation}
  \ss_{\mtheta}(\xx_t) := -\frac{\xx_t + (1-t)\vv_{\mtheta}(\xx_t,t,\cc)}{t},\quad
  \ss_{\text{fake}}(\xx_t) := -\frac{\xx_t + (1-t)\vv_{\text{fake}}(\xx_t,t,\cc_{\text{fake}})}{t} \,.
  \label{eq:score_approx_def}
\end{equation}
Substituting into Eq.~\eqref{eq:apex_kl_grad} (see~\appref{app:kl_velocity}), the KL gradient in velocity space is:
\begin{equation}
  \nabla_{\mtheta} D_{\text{KL}} = -\frac{1}{\omega(t)}\,\EEb{\xx_t,\, \zz,\, t}{ \bigl( \vv_{\mtheta}(\xx_t, t, \cc) - \vv_{\text{data}}(\xx_t) \bigr) \cdot \frac{\partial \xx_t}{\partial {\mtheta}} } \,,
  \label{eq:apex_kl_grad_velocity}
\end{equation}
where $\omega(t) = \frac{t}{1-t} > 0$.
The apparent equivalence dissolves once we recognize that the derivation treats $\vv_\theta$ itself as a proxy for the score of $p_{\text{fake}}$, its descent signal degenerates into self regression.
We replace this proxy with $\vv_{\text{fake}}$ the independent estimator of $p_{\text{fake}}$'s velocity field constructed in \secref{subsec:key_concepts}.
Because $\vv_{\text{fake}}$ was trained on \emph{fake sample trajectories}, it carries information about where $p_{\text{fake}}$ currently lies, providing a correction signal that goes beyond pure regression.
Substituting $\vv_{\text{fake}}$ for the fake score proxy in Eq.~\eqref{eq:apex_kl_grad_velocity}, we define the \method velocity correction signal:
\begin{equation}
  \Delta \vv_{\text{\method}}(\xx_t) := \vv_{\text{fake}}(\xx_t,t,\cc_{\text{fake}}) - \vv_{\mtheta}(\xx_t, t, \cc) \,.
  \label{eq:delta_v_apex}
\end{equation}

This difference measures the velocity discrepancy between $\vv_\theta$ under $\cc$ and $\vv_{\text{fake}}$ under $\cc_{\text{fake}}$, evaluated at the same $(\xx_t,t)$.
Because $\vv_{\text{fake}}$ is trained to track $p_{\text{fake}}$, $\Delta\vv_{\text{\method}}$ encodes the current deviation of the model's generation from the data.
We next construct a practical loss that combines this correction signal with data supervision, where the supervised component drives $\vv_\theta \to \vv_{\text{data}}$ and the fake correction component drives $\vv_\theta \to \vv_{\text{fake}}$; together they form an objective that steers $p_\theta$ toward $p_{\text{real}}$.

\begin{figure*}[!t]
  \centering
  \begin{subfigure}{\textwidth}
    \centering
    \includegraphics[width=\linewidth]{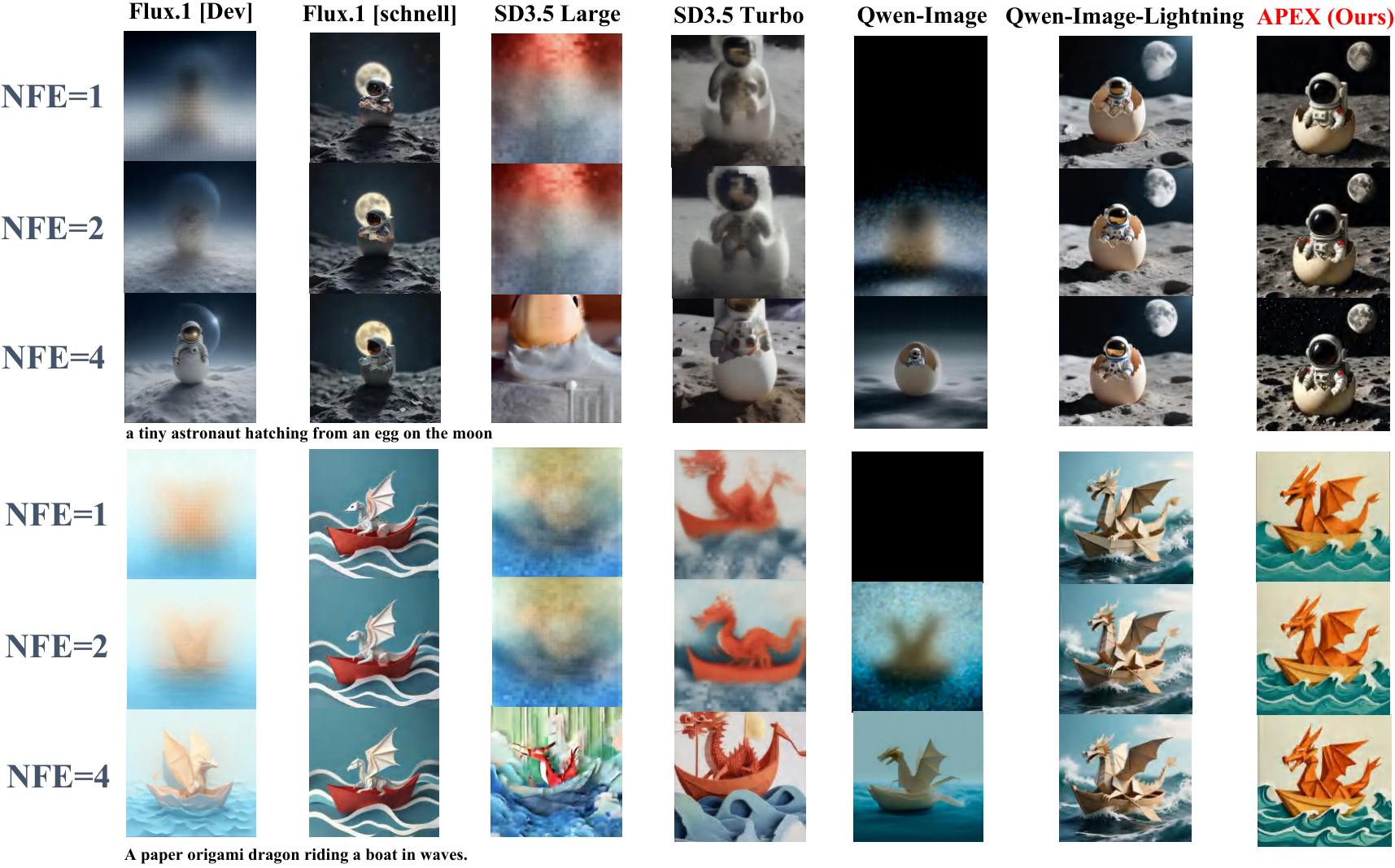}
  \end{subfigure}
  \caption{\small
    \textbf{Qualitative Analysis between \method and existing methods under different NFEs.}
  }
  \label{fig:nfe}
\end{figure*}

\paragraph{From Velocity Correction to Mixed Consistency Loss.}
$\Delta\vv_{\text{\method}}(\xx_t)$ is the KL descent direction: driving $\Delta\vv_{\text{\method}}\to\mathbf{0}$ minimizes $D_{\text{KL}}(p_{\text{fake}}\|p_{\text{real}})$.
$\vv_{\text{fake}}$ is trained on fake trajectories but queried at real trajectory points $\xx_t$; this deliberate asymmetry encodes $p_{\text{fake}}$'s current structure at real trajectory locations, providing a correction signal that breaks the self referential loop.
We convert the velocity objective to endpoint space: one can verify in~\appref{app:endpoint-proof} that velocity matching and endpoint matching are exactly interchangeable:
\begin{small}
  \begin{equation}
    \bigl\| \mf^{\xx}\!\left(\mmF_{\mtheta}, \xx_t, t\right) - \xx \bigr\|_2^2
    = t^2\bigl\| \mmF_{\mtheta} - \vv_{\text{data}}(\xx_t) \bigr\|_2^2,
    \label{eq:apex_endpoint_to_fm}
  \end{equation}
\end{small}
\begin{equation}
  \begin{split}
    \bigl\| \mf^{\xx}\!\left(\mmF_{\mtheta}, \xx_t, t\right) - \mf^{\xx}\!\left(\vv_{\text{fake}}, \xx_t, t\right) \bigr\|_2^2 = t^2\bigl\| \mmF_{\mtheta} - \vv_{\text{fake}}(\xx_t,t,\cc_{\text{fake}}) \bigr\|_2^2 \,.
  \end{split}
  \label{eq:apex_endpoint_to_align}
\end{equation}
Thus matching velocities or matching their induced endpoints are exactly interchangeable up to the scalar factor $t^2$.
We therefore define two endpoint space objectives corresponding to the supervised FM branch and the fake branch, respectively:
\begin{small}
  \begin{equation}
    \textstyle
    \cL_{\text{sup}}(\mtheta)
    \!=\!
    \EEb{\xx_t,\,\zz,\,t}{
      \frac{1}{\omega(t)}
      \bigl\| \mf^{\xx}\!\left(\mmF_{\mtheta}, \xx_t, t\right) - \xx \bigr\|_2^2
    },
    \label{eq:l_sup}
  \end{equation}
\end{small}
\begin{small}
  \begin{equation}
    \textstyle
    \cL_{\text{cons}}(\mtheta)
    \!=\!
    \EEb{\xx_t, \zz, t}{
      \frac{1}{\omega(t)}
      \bigl\| \mf^{\xx}\!\left(\mmF_{\mtheta}, \xx_t, t\right) \!-\! \mf^{\xx}\!\left(\vv_{\text{fake}}, \xx_t, t\right) \bigr\|_2^2
    },
    \label{eq:l_cons}
  \end{equation}
\end{small}
and combine them into the alternative loss:
\begin{equation}
  \cG_{\text{\method}}(\mtheta)
  =
  (1-\lambda)\,\cL_{\text{sup}}(\mtheta)
  +
  \lambda\,\cL_{\text{cons}}(\mtheta), \quad \lambda\in[0,1].
  \label{eq:g_apex}
\end{equation}
Here $\lambda \in [0,1]$ controls the balance between data supervision and self adversarial correction: $\lambda{=}0$ recovers the standard FM objective, $\lambda{=}1$ yields purely adversarial consistency training, and intermediate values blend both signals.
For later convenience we introduce the mixed endpoint target
\begin{equation}
  \mathbf{T}_{\text{mix}}(\xx_t, t)
  := (1-\lambda)\,\xx + \lambda\,\mf^{\xx}\!\left(\vv_{\text{fake}}, \xx_t, t\right),
  \label{eq:tmix}
\end{equation}
where $\vv_{\text{fake}} := \vv_{\text{fake}}(\xx_t,t,\cc_{\text{fake}})$.
Its score space counterpart the \emph{score interpolation} $\ss_{\text{mix}}$ defined in \secref{subsec:gan_connection} will reveal that $\mathbf{T}_{\text{mix}}$ corresponds to an implicit training target.
The corresponding mixed consistency loss is:
\begin{small}
  \begin{equation}
    \textstyle
    \cL_{\text{mix}}(\mtheta)
    \!=\!
    \EEb{\xx_t,\, \zz,\, t}{
      \frac{1}{\omega(t)}
      \bigl\| \mf^{\xx}\!\left(\mmF_{\mtheta}, \xx_t, t\right) \!-\! \mathbf{T}_{\text{mix}}(\xx_t, t) \bigr\|_2^2
    }.
    \label{eq:l_mix}
  \end{equation}
\end{small}%
A direct gradient calculation with detailed steps in~\appref{app:alt-loss-proof} shows that for any ${\mtheta}$ we have
$\nabla_{\mtheta}\cL_{\text{mix}}(\mtheta) = \nabla_{\mtheta}\cG_{\text{\method}}(\mtheta)$,
so optimizing the mixed endpoint regression in Eq.~\eqref{eq:l_mix} is exactly equivalent, in parameter space, to following the KL inspired alternative loss in Eq.~\eqref{eq:g_apex}.

\begin{table*}[t]
  \centering
  \caption{\textbf{System level comparison of efficiency and quality.}
    Speeds are on a single A100 (\textbf{BF16}).
    Throughput is samples/s (batch=10); latency is seconds (batch=1).
    \textbf{GenEval} is the \emph{primary} quality metric; FID/CLIP are reported for completeness.
    The \textbf{best} and \underline{second best} entries are highlighted.
    \dag~indicates methods requiring distinct models per NFE. \textbf{Notation:} \textcolor{color3}{Blue}=full tuning; \textcolor{color5}{Red}=LoRA; \textcolor{color3}{X.B}=trainable params (B); $r$=LoRA rank.
  }
  \small
  \vspace{-1em}
  \label{tab:main_comparison}{
    \scalebox{0.77}{
      \begin{tabular}{l|l|ccc|cccc}
        \toprule
        \multicolumn{2}{c|}{\multirow{2}{*}{\textbf{Methods}}} & \multirow{2}{*}{\textbf{NFEs}}                    & \textbf{Throughput}   & \textbf{Latency}         & \textbf{Params}          & \multirow{2}{*}{\textbf{FID~$\downarrow$}}                                                                      & \multirow{2}{*}{\textbf{CLIP~$\uparrow$}} & \multirow{2}{*}{\textbf{GenEval~$\uparrow$}}                                     \\
        \multicolumn{2}{c|}{}                                  &                                                   & \textbf{(samples/s)}  & \textbf{(s)}             & \textbf{(B)}             &                                                                                                                 &                                           &                                                                                  \\
        \midrule
        \multicolumn{1}{c|}{\multirow{32}{*}{\rotatebox{90}{\textbf{Few Step Distillation Models}}}}
                                                               & SDXL-LCM~\cite{luo2023latent}                     & 2                     & 2.89                     & 0.40                     & \textcolor{color3}{\textbf{0.9}}                                                                                & 18.11                                     & 27.51                                        & 0.44                              \\
                                                               & PixArt-LCM~\cite{chen2024pixartdelta}             & 2                     & 3.52                     & 0.31                     & \textcolor{color3}{\textbf{0.6}}                                                                                & 10.33                                     & 27.24                                        & 0.42                              \\
                                                               & SD3.5-Turbo~\cite{esser2024scaling}               & 2                     & 1.61                     & 0.68                     & \textcolor{color3}{\textbf{8.0}}                                                                                & 51.47                                     & 25.59                                        & 0.53                              \\
                                                               & PCM~\cite{wang2024phased}\dag                     & 2                     & 2.62                     & 0.56                     & \textcolor{color3}{\textbf{0.9}}                                                                                & 14.70                                     & 27.66                                        & 0.55                              \\
                                                               & SDXL-DMD2~\cite{yin2024improved}\dag              & 2                     & 2.89                     & 0.40                     & \textcolor{color3}{\textbf{0.9}}                                                                                & \textit{7.61}                             & 28.87                                        & 0.58                              \\
                                                               & FLUX-schnell~\citep{FLUX}                         & 2                     & 0.92                     & 1.15                     & \textcolor{color3}{\textbf{12.0}}                                                                               & 7.75                                      & 28.25                                        & \textit{0.71}                     \\
                                                               & Sana-Sprint~\citep{chen2025sana}                  & 2                     & 6.46                     & 0.25                     & \textcolor{color3}{\textbf{0.6}}                                                                                & \underline{6.54}                          & \textit{28.40}                               & \underline{0.76}                  \\
                                                               & Sana-Sprint~\citep{chen2025sana}                  & 2                     & 5.68                     & 0.24                     & \textcolor{color3}{\textbf{1.6}}                                                                                & 6.50                                      & \underline{28.45}                            & 0.77                              \\
                                                               & Qwen-Image-Lightning~\citep{Qwen-Image-Lightning} & 2                     & 3.15                     & 0.48                     & \textcolor{color5}{\textbf{20}} (\textcolor{color2}{r=64},\textcolor{color3}{\textbf{0.4}})                     & 6.76                                      & 28.37                                        & 0.85                              \\
                                                               & RCGM~\citep{sun2025anystep}                       & 2                     & 3.15                     & 0.48                     & \textcolor{color5}{\textbf{20}} (\textcolor{color2}{r=64},\textcolor{color3}{\textbf{0.4}})                     & 6.80                                      & 28.63                                        & 0.82                              \\
                                                               & TwinFlow~\citep{cheng2025twinflow}                & 2                     & 3.15                     & 0.48                     & \textcolor{color5}{\textbf{20}} (\textcolor{color2}{r=64},\textcolor{color3}{\textbf{0.4}})                     & 6.73                                      & 28.57                                        & 0.87                              \\
        \cmidrule{2-9}
                                                               & \bf \cellcolor{gray!20}\method                    & \cellcolor{gray!20} 2 & \cellcolor{gray!20} 6.50 & \cellcolor{gray!20} 0.25 & \cellcolor{gray!20} \textcolor{color3}{\textbf{0.6}}                                                            & \cellcolor{gray!20} 6.75                  & \cellcolor{gray!20} \textit{28.33}           & \cellcolor{gray!20} 0.84          \\
                                                               & \bf \cellcolor{gray!20}\method                    & \cellcolor{gray!20} 2 & \cellcolor{gray!20} 5.72 & \cellcolor{gray!20} 0.23 & \cellcolor{gray!20} \textcolor{color3}{\textbf{1.6}}                                                            & \cellcolor{gray!20} 6.42                  & \cellcolor{gray!20} \textit{28.24}           & \cellcolor{gray!20} 0.85          \\
        \cmidrule{2-9}
                                                               & \bf \cellcolor{gray!20}\method                    & \cellcolor{gray!20} 2 & \cellcolor{gray!20} 3.21 & \cellcolor{gray!20} 0.49 & \cellcolor{gray!20} \textcolor{color5}{\textbf{20}} (\textcolor{color2}{r=32},\textcolor{color3}{\textbf{0.2}}) & \cellcolor{gray!20} 6.72                  & \cellcolor{gray!20} \textit{28.71}           & \cellcolor{gray!20} 0.87          \\
                                                               & \bf \cellcolor{gray!20}\method                    & \cellcolor{gray!20} 2 & \cellcolor{gray!20} 3.17 & \cellcolor{gray!20} 0.47 & \cellcolor{gray!20} \textcolor{color5}{\textbf{20}} (\textcolor{color2}{r=64},\textcolor{color3}{\textbf{0.4}}) & \cellcolor{gray!20} 6.51                  & \cellcolor{gray!20} \textit{28.42}           & \cellcolor{gray!20} 0.89          \\
        \cmidrule{2-9}
                                                               & \bf \cellcolor{gray!20}\method                    & \cellcolor{gray!20} 2 & \cellcolor{gray!20} 3.30 & \cellcolor{gray!20} 0.45 & \cellcolor{gray!20} \textcolor{color3}{\textbf{20}}                                                             & \cellcolor{gray!20} 6.44                  & \cellcolor{gray!20} 28.51                    & \cellcolor{gray!20} \textbf{0.90} \\
        \cmidrule{2-9}
                                                               & SDXL-LCM~\cite{luo2023latent}                     & 1                     & 3.36                     & 0.32                     & \textcolor{color3}{\textbf{0.9}}                                                                                & 50.51                                     & 24.45                                        & 0.28                              \\
                                                               & PixArt-LCM~\cite{chen2024pixartdelta}             & 1                     & 4.26                     & 0.25                     & \textcolor{color3}{\textbf{0.6}}                                                                                & 73.35                                     & 23.99                                        & 0.41                              \\
                                                               & PixArt-DMD~\cite{chenpixartsigma}\dag             & 1                     & 4.26                     & 0.25                     & \textcolor{color3}{\textbf{0.6}}                                                                                & 9.59                                      & 26.98                                        & 0.45                              \\
                                                               & SD3.5-Turbo~\cite{esser2024scaling}               & 1                     & 2.48                     & 0.45                     & \textcolor{color3}{\textbf{8.0}}                                                                                & 52.40                                     & 25.40                                        & 0.51                              \\
                                                               & PCM~\cite{wang2024phased}\dag                     & 1                     & 3.16                     & 0.40                     & \textcolor{color3}{\textbf{0.9}}                                                                                & 30.11                                     & 26.47                                        & 0.42                              \\
                                                               & SDXL-DMD2~\cite{yin2024improved}\dag              & 1                     & 3.36                     & 0.32                     & \textcolor{color3}{\textbf{0.9}}                                                                                & \underline{7.10}                          & 28.93                                        & 0.59                              \\
                                                               & FLUX-schnell~\citep{FLUX}                         & 1                     & 1.58                     & 0.68                     & \textcolor{color3}{\textbf{12.0}}                                                                               & \textit{7.26}                             & \underline{28.49}                            & \textit{0.69}                     \\
                                                               & Sana-Sprint~\citep{chen2025sana}                  & 1                     & 7.22                     & 0.21                     & \textcolor{color3}{\textbf{0.6}}                                                                                & 7.04                                      & 28.04                                        & 0.72                              \\
                                                               & Sana-Sprint~\citep{chen2025sana}                  & 1                     & 6.71                     & 0.21                     & \textcolor{color3}{\textbf{1.6}}                                                                                & 7.69                                      & \textit{28.27}                               & 0.76                              \\
                                                               & Qwen-Image-Lightning~\citep{Qwen-Image-Lightning} & 1                     & 3.29                     & 0.40                     & \textcolor{color5}{\textbf{20}} (\textcolor{color2}{r=64},\textcolor{color3}{\textbf{0.4}})                     & 7.06                                      & 28.35                                        & 0.85                              \\
                                                               & RCGM~\citep{sun2025anystep}                       & 1                     & 3.29                     & 0.40                     & \textcolor{color5}{\textbf{20}} (\textcolor{color2}{r=64},\textcolor{color3}{\textbf{0.4}})                     & 11.38                                     & 27.69                                        & 0.52                              \\
                                                               & TwinFlow~\citep{cheng2025twinflow}                & 1                     & 3.29                     & 0.40                     & \textcolor{color5}{\textbf{20}} (\textcolor{color2}{r=64},\textcolor{color3}{\textbf{0.4}})                     & 7.32                                      & 28.29                                        & 0.86                              \\
        \cmidrule{2-9}
                                                               & \bf \cellcolor{gray!20}\method                    & \cellcolor{gray!20} 1 & \cellcolor{gray!20} 7.30 & \cellcolor{gray!20} 0.20 & \cellcolor{gray!20} \textcolor{color3}{\textbf{0.6}}                                                            & \cellcolor{gray!20} 6.99                  & \cellcolor{gray!20} \textit{28.36}           & \cellcolor{gray!20} 0.84          \\
                                                               & \bf \cellcolor{gray!20}\method                    & \cellcolor{gray!20} 1 & \cellcolor{gray!20} 6.84 & \cellcolor{gray!20} 0.20 & \cellcolor{gray!20} \textcolor{color3}{\textbf{1.6}}                                                            & \cellcolor{gray!20} 6.78                  & \cellcolor{gray!20} \textit{28.12}           & \cellcolor{gray!20} 0.84          \\
        \cmidrule{2-9}
                                                               & \bf \cellcolor{gray!20}\method                    & \cellcolor{gray!20} 1 & \cellcolor{gray!20} 3.29 & \cellcolor{gray!20} 0.39 & \cellcolor{gray!20} \textcolor{color5}{\textbf{20}} (\textcolor{color2}{r=32},\textcolor{color3}{\textbf{0.2}}) & \cellcolor{gray!20} 7.22                  & \cellcolor{gray!20} \textit{28.62}           & \cellcolor{gray!20} 0.88          \\
                                                               & \bf \cellcolor{gray!20}\method                    & \cellcolor{gray!20} 1 & \cellcolor{gray!20} 3.27 & \cellcolor{gray!20} 0.39 & \cellcolor{gray!20} \textcolor{color5}{\textbf{20}} (\textcolor{color2}{r=64},\textcolor{color3}{\textbf{0.4}}) & \cellcolor{gray!20} 7.14                  & \cellcolor{gray!20} \textit{28.45}           & \cellcolor{gray!20} \textbf{0.89} \\
        \cmidrule{2-9}
                                                               & \bf \cellcolor{gray!20}\method                    & \cellcolor{gray!20} 1 & \cellcolor{gray!20} 3.50 & \cellcolor{gray!20} 0.34 & \cellcolor{gray!20} \textcolor{color3}{\textbf{20}}                                                             & \cellcolor{gray!20} 6.87                  & \cellcolor{gray!20} 28.66                    & \cellcolor{gray!20} \textbf{0.89} \\
        \bottomrule
      \end{tabular}}
  }
\end{table*}

\subsection{Complete Objective and GAN Gradient Structure}
\label{subsec:gan_connection}

\paragraph{Complete Training Objective.}
The full \method objective combines the fake flow fitting $\cL_{\text{fake}}$ with the mixed consistency loss $\cL_{\text{mix}}$:
\begin{equation}
  \cL_{\text{\method}}(\mtheta)
  =
  \lambda_p\,\cL_{\text{fake}}(\mtheta) + \lambda_e\,\cL_{\text{mix}}(\mtheta), \quad \lambda_p,\lambda_e \geq 0 \,.
  \label{eq:l_apex_total}
\end{equation}
$\cL_{\text{fake}}$ is a prerequisite: it trains the shifted condition branch as an independent estimator of $p_{\text{fake}}$'s velocity field so that $\vv_{\text{fake}}$ can serve as a valid correction reference.
The KL descent interpretation of \secref{subsec:kl} applies to $\cL_{\text{mix}}$, which uses $\vv_{\text{fake}}$ to form the mixed target.
We now analyze the gradient of $\cL_{\text{mix}}$ in score space to reveal its formal connection to GAN dynamics.

\paragraph{GAN Aligned Gradient Structure.}
Via Score Velocity Duality Eq.~\eqref{eq:score_velocity_duality}, velocity differences translate to score differences by the time dependent factor $-\frac{t}{1-t}$.
Applying this to $\cG_{\text{\method}}$, we define:
\begin{equation}
  \ss_{\text{mix}}(\xx_t) := (1{-}\lambda)\,\ss_{\text{data}}(\xx_t) + \lambda\,\ss_{\text{fake}}(\xx_t) \,,
  \label{eq:score_mix}
\end{equation}
where $\ss_{\text{data}}(\xx_t) = \nabla_{\xx_t}\log p_{\text{data},t}(\xx_t)$ and $\ss_{\text{fake}}(\xx_t) = \nabla_{\xx_t}\log p_{\text{fake},t}(\xx_t)$.
This yields:
\begin{mdframed}[
    hidealllines=true,
    backgroundcolor=color4!5,
    innerleftmargin=12pt,
    innerrightmargin=12pt,
    innertopmargin=6pt,
    innerbottommargin=6pt,
    roundcorner=1pt,
  ]
  \noindent\textbf{Proposition (GAN-Aligned Gradient).}
  \emph{The gradient of $\cG_{\text{\method}}$ takes the GAN canonical score difference form:}
  \begin{small}
    \begin{equation}
      \nabla_{\mtheta}\cG_{\text{\method}}(\mtheta)
      \propto
      \EEb{\xx_t\sim p_{\mtheta,t}}{\underbrace{\vphantom{D^*(\xx_t)}1}_{w\equiv 1} \cdot \left(\ss_{\mtheta}(\xx_t) - \ss_{\text{mix}}(\xx_t)\right) \cdot \frac{\partial \xx_t}{\partial\mtheta}} \,,
      \label{eq:apex_grad_score}
    \end{equation}
  \end{small}%
  \emph{with constant weight $w \equiv 1$, corresponding to minimizing the Fisher divergence $D_F(p_{\mtheta}\|p_{\text{mix}})$.}
\end{mdframed}

\noindent The Fisher divergence is:
\begin{equation}
  D_F(p_{\mtheta}\|p_{\text{mix}})
  := \int \bigl\|\ss_{\mtheta}(\xx_t) - \ss_{\text{mix}}(\xx_t)\bigr\|_2^2\,p_{\mtheta}(\xx_t)\,\dm\xx_t \,.
  \label{eq:fisher_div}
\end{equation}
Here $\ss_{\text{mix}}$ is a convex combination of score functions and need not correspond to a proper probability distribution; we interpret $p_{\text{mix}}$ as an implicit training target, analogous to the implicit distribution induced by score interpolation in classifier free guidance~\citep{ho2022classifier}.
Eq.~\eqref{eq:apex_grad_score} reveals that \method follows a GAN-aligned gradient with a constant weight $w{\equiv}1$: the time factor $-\frac{2t^3}{1-t}$ is absorbed into $\omega(t)$ and is uniform across all samples at each $t$.

\section{Experiments}
\label{sec:exp}

\subsection{Experimental Setup}
\label{sec:exp_set}

\textbf{\emph{$\bullet$ Backbones and tuning.}}
We consider three capacities: \method 0.6B and \method 1.6B (full parameter tuning), and \method 20B using LoRA on Qwen-Image~\citep{wu2025qwen}.

\textbf{\emph{$\bullet$ Datasets.}}
Our training data comprises both open source and newly synthesized datasets. We utilize ShareGPT-4o~\citep{chen2025sharegpt} and BLIP-3o~\citep{chen2025blip3} as our part of open source resources.
Additionally, we construct two synthetic datasets using the Qwen-Image-20B model.
Part of the data includes 600K samples generated from prompts in the Flux-Reasoning-6M dataset~\citep{fang2025flux}, and another 200K samples synthesized from poster prompts.

\textbf{\emph{$\bullet$ Training and hardware.}}
Training uses BF16 precision.
For LoRA, we vary the rank $r \in \{32, 64\}$ and keep all other settings identical across ranks.
We use 16$\times$NVIDIA H800 80GB, 8$\times$A100 80GB GPUs for training and evaluation.

\textbf{\emph{$\bullet$ Evaluation metrics.}}
Our \emph{primary} metric is GenEval Overall~\citep{ghosh2023geneval}.
We also report FID and CLIP on MJHQ-30K~\citep{li2024playgroundv2.5}, DPGBench~\citep{hu2024ella} and WISE~\citep{niu2025wise} for completeness.
Unless noted, results are with NFE=1.

\begin{table*}[!t]\centering
  \caption{\textbf{Quantitative Evaluation results on GenEval.}}
  \scalebox{0.78}{
    \begin{tabular}{l|cccccc|c}
      \toprule
      \multirow{2}{*}{\textbf{Model}}                          & \textbf{Single} & \textbf{Two} & \multirow{2}{*}{\textbf{Counting}} & \multirow{2}{*}{\textbf{Colors}} & \multirow{2}{*}{\textbf{Position}} & \textbf{Attribute} & \multirow{2}{*}{\textbf{Overall$\uparrow$}} \\
                                                               & \bf Object      & \bf Object   &                                    &                                  &                                    & \bf Binding        &                                             \\
      \midrule
      Show-o~\citep{xie2024show}                               & 0.95            & 0.52         & 0.49                               & 0.82                             & 0.11                               & 0.28               & 0.53                                        \\
      Emu3-Gen~\citep{wang2024emu3}                            & 0.98            & 0.71         & 0.34                               & 0.81                             & 0.17                               & 0.21               & 0.54                                        \\
      PixArt-$\alpha$~\citep{chen2024pixartalpha}              & 0.98            & 0.50         & 0.44                               & 0.80                             & 0.08                               & 0.07               & 0.48                                        \\
      SD3 Medium~\citep{esser2024scaling}                      & 0.98            & 0.74         & 0.63                               & 0.67                             & 0.34                               & 0.36               & 0.62                                        \\
      FLUX.1 [Dev]~\citep{flux2024}                            & 0.98            & 0.81         & 0.74                               & 0.79                             & 0.22                               & 0.45               & 0.66                                        \\
      SD3.5 Large~\citep{esser2024scaling}                     & 0.98            & 0.89         & 0.73                               & 0.83                             & 0.34                               & 0.47               & 0.71                                        \\
      JanusFlow~\citep{ma2025janusflow}                        & 0.97            & 0.59         & 0.45                               & 0.83                             & 0.53                               & 0.42               & 0.63                                        \\
      Lumina-Image 2.0~\citep{qin2025lumina}                   & -               & 0.87         & 0.67                               & -                                & -                                  & 0.62               & 0.73                                        \\
      Janus-Pro-7B~\citep{chen2025janus}                       & 0.99            & 0.89         & 0.59                               & 0.90                             & 0.79                               & 0.66               & 0.80                                        \\
      HiDream-I1-Full~\citep{cai2025hidream}                   & \bf 1.00        & 0.98         & 0.79                               & 0.91                             & 0.60                               & 0.72               & 0.83                                        \\
      GPT Image 1 [High]~\citep{gptimage}                      & 0.99            & 0.92         & 0.85                               & 0.92                             & 0.75                               & 0.61               & 0.84                                        \\
      Seedream 3.0~\citep{gao2025seedream}                     & 0.99            & \bf 0.96     & 0.91                               & \bf 0.93                         & 0.47                               & 0.80               & 0.84                                        \\
      BAGEL~\citep{deng2025bagel}                              & 0.98            & 0.95         & 0.84                               & 0.95                             & 0.78                               & 0.77               & 0.88                                        \\
      Qwen-Image~\citep{wu2025qwen}                            & 0.99            & 0.92         & 0.89                               & 0.88                             & 0.76                               & 0.77               & \underline{0.87}                            \\
      Hyper-BAGEL~\citep{lu2025hyperbagel}                     & 0.97            & 0.86         & 0.75                               & 0.90                             & 0.67                               & 0.62               & 0.80                                        \\
      Qwen-Image-Lightning~\citep{Qwen-Image-Lightning}        & 0.99            & 0.89         & 0.85                               & 0.87                             & 0.75                               & 0.76               & 0.85                                        \\
      TwinFlow~\citep{cheng2025twinflow}  (1-NFE)              & 1.00            & 0.91         & 0.84                               & 0.90                             & 0.75                               & 0.74               & 0.86                                        \\
      \midrule
      \rowcolor{gray!20}
      \bf \method 0.6B (1-NFE)                                 & 0.99            & 0.91         & 0.75                               & \bf 0.93                         & 0.76                               & 0.69               & 0.84                                        \\
      \rowcolor{gray!20}
      \bf \method 1.6B (1-NFE)                                 & 0.99            & 0.91         & 0.75                               & \bf 0.93                         & 0.76                               & 0.68               & 0.84                                        \\
      \midrule
      \rowcolor{gray!20}
      \bf \method 20B (LoRA\&\textcolor{color2}{r=32}) (1-NFE) & 0.99            & 0.95         & 0.85                               & 0.90                             & 0.79                               & 0.78               & 0.88                                        \\
      \rowcolor{gray!20}
      \bf \method 20B (LoRA\&\textcolor{color2}{r=64}) (1-NFE) & 0.99            & 0.94         & 0.88                               & 0.90                             & \bf 0.85                           & \bf 0.78           & \bf 0.89                                    \\
      \midrule
      \rowcolor{gray!20}
      \bf \method 20B (\textcolor{color3}{SFT}) (1-NFE)        & 0.99            & 0.92         & 0.83                               & 0.91                             & 0.86                               & 0.81               & \textbf{0.89}                               \\
      \bottomrule
    \end{tabular}\label{tab:geneval}
  }
  \vspace{-1em}
\end{table*}

\begin{table*}[!t]
  \centering
  \caption{\textbf{Quantitative evaluation results on DPGBench.}}
  \vspace{-1em}
  \scalebox{0.78}{
    \begin{tabular}{l|ccccc|c}
      \toprule
      \textbf{Model}                                             & \bf Global     & \bf Entity     & \bf Attribute  & \bf Relation   & \bf Other      & \bf Overall$\uparrow$ \\
      \midrule
      SD v1.5~\citep{rombach2022high}                            & 74.63          & 74.23          & 75.39          & 73.49          & 67.81          & 63.18                 \\
      PixArt-$\alpha$~\citep{chen2024pixartalpha}                & 74.97          & 79.32          & 78.60          & 82.57          & 76.96          & 71.11                 \\
      Lumina-Next~\citep{zhuo2024lumina}                         & 82.82          & 88.65          & 86.44          & 80.53          & 81.82          & 74.63                 \\
      SDXL~\citep{podell2023sdxl}                                & 83.27          & 82.43          & 80.91          & 86.76          & 80.41          & 74.65                 \\
      Hunyuan-DiT~\citep{li2024hunyuan-dit}                      & 84.59          & 80.59          & 88.01          & 74.36          & 86.41          & 78.87                 \\
      Janus~\citep{wu2025janus}                                  & 82.33          & 87.38          & 87.70          & 85.46          & 86.41          & 79.68                 \\
      PixArt-$\Sigma$~\citep{chen2024pixartsigma}                & 86.89          & 82.89          & 88.94          & 86.59          & 87.68          & 80.54                 \\
      Emu3-Gen~\citep{wang2024emu3}                              & 85.21          & 86.68          & 86.84          & 90.22          & 83.15          & 80.60                 \\
      Janus-Pro-1B~\citep{chen2025janus}                         & 87.58          & 88.63          & 88.17          & 88.98          & 88.30          & 82.63                 \\
      DALL-E 3~\citep{Dalle-3}                                   & 90.97          & 89.61          & 88.39          & 90.58          & 89.83          & 83.50                 \\
      FLUX.1 [Dev]~\citep{flux2024}                              & 74.35          & 90.00          & 88.96          & 90.87          & 88.33          & 83.84                 \\
      SD3.5-Medium~\cite{esser2024scaling}                       & 84.08          & 87.90          & 91.01          & 88.83          & 80.70          & 88.68                 \\
      SD3.5-Turbo~\cite{sauer2024sd3turbo}                       & 79.03          & 80.12          & 86.13          & 84.73          & 91.86          & 78.29                 \\
      SD3.5-Large~\cite{esser2024scaling}                        & 83.21          & 84.27          & 88.99          & 87.35          & 93.28          & 80.35                 \\
      FLUX.1-schnell~\cite{FLUX}                                 & 84.94          & 86.62          & 90.82          & 88.35          & 93.45          & 82.00                 \\
      Janus-Pro-7B~\citep{chen2025janus}                         & 86.90          & 88.90          & 89.40          & 89.32          & 89.48          & 84.19                 \\
      HiDream-I1-Full~\citep{cai2025hidream}                     & 76.44          & 90.22          & 89.48          & 93.74          & 91.83          & 85.89                 \\
      Lumina-Image 2.0~\citep{qin2025lumina}                     & -              & 91.97          & 90.20          & \textbf{94.85} & -              & 87.20                 \\
      Seedream 3.0~\citep{gao2025seedream}                       & \textbf{94.31} & \textbf{92.65} & 91.36          & 92.78          & 88.24          & 88.27                 \\
      GPT Image 1 [High]~\citep{gptimage}                        & 88.89          & 88.94          & 89.84          & 92.63          & 90.96          & 85.15                 \\
      Qwen-Image~\citep{wu2025qwen}                              & 91.32          & 91.56          & \textbf{92.02} & 94.31          & \textbf{92.73} & 88.32                 \\
      Playground v3~\citep{liu2024playground}                    & 87.04          & 91.94          & 85.71          & 90.90          & 90.00          & \textbf{92.72}        \\
      TwinFlow~\citep{cheng2025twinflow}  (1-NFE)                & 92.34          & 92.12          & 92.45          & 92.86          & 92.63          & 86.52                 \\
      \midrule
      \rowcolor{gray!20}
      \bf \method 0.6B (1-NFE)                                   & 90.58          & 90.36          & 90.44          & 90.77          & 90.73          & 82.66                 \\
      \rowcolor{gray!20}
      \bf \method 1.6B (1-NFE)                                   & 90.77          & 90.56          & 90.63          & 90.98          & 90.94          & 83.22                 \\
      \midrule
      \rowcolor{gray!20}
      \bf \method 20B (LoRA\&\textcolor{color2}{r=32}) (1-NFE)   & 93.12          & 90.95          & 91.38          & 90.65          & 91.73          & 86.17                 \\
      \rowcolor{gray!20}
      \bf \method 20B (LoRA\&\textcolor{color2}{r=64}) (1-NFE)   & 92.46          & 91.14          & 90.71          & 91.30          & 91.98          & 85.77                 \\
      \midrule
      \rowcolor{gray!20}
      \bf \method 20B (\textcolor{color3}{\textbf{SFT}}) (1-NFE) & 93.25          & 89.76          & 90.65          & 91.17          & 90.75          & 84.59                 \\
      \bottomrule
    \end{tabular}\label{tab:dpg}
  }
  \vspace{-1em}
\end{table*}

\subsection{Efficiency and Performance Comparison}
We profile \method under \textbf{NFE=1/2} and contrast it with the strongest prior distilled models at each setting, summarized in \tabref{tab:main_comparison}.
\textbf{GenEval Overall} is our headline metric, with throughput and latency reported to highlight practical applicability.

At \textbf{NFE=1}, \method 0.6B sustains \textbf{7.3} samples/s at \textbf{0.20}s latency while achieving \textbf{0.84} GenEval a \(\approx\!0.15\) absolute improvement over FLUX-Schnell 12B (GenEval 0.69), a model with \textbf{20$\times$ more parameters}.
This result suggests that the endogenous adversarial signal from condition shifting is more parameter efficient than scaling model capacity under standard distillation.
Scaling to \method 1.6B keeps latency flat with similar throughput.
Our LoRA-tuned \method 20B further lifts GenEval to \textbf{0.89} (\textcolor{color2}{r=64}) at only \textbf{0.39}s latency state of the art at NFE=1.
Notably, this quality level is reached after only \textbf{6 hours} of LoRA training (2K steps, global batch size 64), while the original Qwen-Image 20B requires 50 integration steps to achieve GenEval 0.87.
\method thus simultaneously improves quality and reduces both training and inference cost.

Moving to \textbf{NFE=2}, \method 1.6B rises to \textbf{0.85} GenEval, an \(\sim\!8\)-point margin over the strongest two-step baseline (Sana-Sprint 1.6B at 0.77) while running more than twice as fast.
The 20B LoRA variant sustains \textbf{0.89} GenEval with a modest latency bump to \textbf{0.47}s.
Taken together, these results demonstrate that \method closes the quality gap to multi-step generators without sacrificing the latency advantage that makes distilled models practical in production pipelines.

\begin{table*}[!t]
  \centering
  \caption{\textbf{Quantitative evaluation results on WISE.}}
  \vspace{-1em}
  \scalebox{0.78}{
    \begin{tabular}{l|cccccc|c}
      \toprule
      \textbf{Model}                                             & \bf Cultural & \bf Time & \bf Space & \bf Biology & \bf Physics & \bf Chemistry & \bf Overall$\uparrow$ \\
      \midrule
      SD v1.5~\citep{rombach2022high}                            & 0.34         & 0.35     & 0.32      & 0.28        & 0.29        & 0.21          & 0.32                  \\
      SDXL~\citep{podell2023sdxl}                                & 0.43         & 0.48     & 0.47      & 0.44        & 0.45        & 0.27          & 0.43                  \\
      SD3.5-Large~\cite{esser2024scaling}                        & 0.44         & 0.50     & 0.58      & 0.44        & 0.52        & 0.31          & 0.46                  \\
      PixArt-$\alpha$~\citep{chen2024pixartalpha}                & 0.45         & 0.50     & 0.48      & 0.49        & 0.56        & 0.34          & 0.47                  \\
      Playground-v2.5~\citep{li2024playgroundv2.5}               & 0.49         & 0.58     & 0.55      & 0.43        & 0.48        & 0.33          & 0.49                  \\
      FLUX.1 [Dev]~\citep{flux2024}                              & 0.48         & 0.58     & 0.62      & 0.42        & 0.51        & 0.35          & 0.50                  \\
      Janus~\citep{wu2025janus}                                  & 0.16         & 0.26     & 0.35      & 0.28        & 0.30        & 0.14          & 0.23                  \\
      VILA-U~\citep{wu2024vila}                                  & 0.51         & 0.51     & 0.51      & 0.49        & 0.51        & 0.49          & 0.50                  \\
      Show-o~\citep{xie2024show}                                 & 0.95         & 0.52     & 0.49      & 0.82        & 0.11        & 0.28          & 0.53                  \\
      Janus-Pro-7B~\citep{chen2025janus}                         & 0.30         & 0.37     & 0.49      & 0.36        & 0.42        & 0.26          & 0.35                  \\
      Emu3-Gen~\citep{wang2024emu3}                              & 0.34         & 0.45     & 0.48      & 0.41        & 0.45        & 0.47          & 0.39                  \\
      MetaQuery-XL~\citep{pan2025transfer}                       & 0.56         & 0.55     & 0.62      & 0.49        & 0.63        & 0.41          & 0.55                  \\
      BAGEL~\citep{deng2025bagel}                                & 0.44         & 0.55     & 0.68      & 0.44        & 0.60        & 0.39          & 0.52                  \\
      GPT-4o                                                     & 0.81         & 0.71     & 0.89      & 0.83        & 0.79        & 0.74          & \textbf{0.80}         \\
      Qwen-Image~\citep{wu2025qwen}                              & -            & -        & -         & -           & -           & -             & 0.62                  \\
      Qwen-Image-Lightning~\citep{Qwen-Image-Lightning}          & -            & -        & -         & -           & -           & -             & 0.51                  \\
      TwinFlow~\citep{cheng2025twinflow}                         & 0.52         & 0.51     & 0.67      & 0.48        & 0.61        & 0.40          & 0.54                  \\
      \midrule
      \rowcolor{gray!20}
      \bf \method 20B (\textcolor{color3}{\textbf{SFT}}) (1-NFE) & 0.53         & 0.54     & 0.66      & 0.48        & 0.61        & 0.41          & 0.54                  \\
      \bottomrule
    \end{tabular}\label{tab:wise}
  }
  \vspace{-1em}
\end{table*}

\begin{table}[t]
  \caption{\textbf{Effect of training data and steps on GenEval Overall (NFE=1).}
    We compare \textbf{ShareGPT-4o} and \textbf{BLIP-3o} across training steps for \method 0.6B/1.6B, and LoRA tuned Qwen-Image 20B with ranks \textcolor{color2}{r=32}/\textcolor{color2}{r=64}. All runs use global batch size \textbf{64}.}
  \centering
  \scalebox{0.62}{
    \begin{tabular}{l|ccc|ccc}
      \toprule
      \multirow{2}*{Model}                   & \multicolumn{3}{c|}{ShareGPT-4o} & \multicolumn{3}{c}{Blip-3o}                                                     \\
                                             & {2Ksteps}                        & {8Ksteps}                   & {10Ksteps} & {2Ksteps}   & {8Ksteps} & {10Ksteps} \\
      \midrule
      \method 0.6B                           & 0.37                             & 0.67                        & 0.73       & 0.71        & 0.77      & 0.81       \\
      \method 1.6B                           & 0.36                             & 0.70                        & 0.73       & 0.27        & 0.78      & 0.83       \\
      \midrule
                                             & {0.4Ksteps}                      & {1Ksteps}                   & {2Ksteps}  & {0.4Ksteps} & {1Ksteps} & {2Ksteps}  \\
      \midrule
      \method 20B (\textcolor{color2}{r=32}) & 0.19                             & 0.33                        & 0.62       & 0.83        & 0.84      & 0.83       \\
      \method 20B (\textcolor{color2}{r=64}) & 0.21                             & 0.35                        & 0.61       & 0.73        & 0.85      & 0.84       \\
      \bottomrule
    \end{tabular}}
  \label{tab:dataset}
  \vspace{-1em}
\end{table}

\subsection{Ablations}
\label{sec:ablation}

We present controlled ablations to isolate the effects of key design choices in \method.
Unless otherwise stated, all results are reported with \textbf{NFE=1} and the \textbf{GenEval Overall} metric, using identical prompts, seeds, and resolution.

\paragraph{Balancing $\cL_{\text{fake}}$ and $\cL_{\text{mix}}$.}
We dissect the contribution of the fake flow fitting objective $\cL_{\text{fake}}$ (Eq.~\eqref{eq:apex_fake_loss}) and the mixed consistency objective $\cL_{\text{mix}}$ (Eq.~\eqref{eq:l_mix}) by ablating their outer relative weights $\lambda_p{:}\lambda_e$ in $\cL_{\text{\method}} = \lambda_{p}\,\cL_{\text{fake}} + \lambda_{e}\,\cL_{\text{mix}}$ on three models: \method 0.6B, 1.6B, and 20B (LoRA).
Here $\lambda_p, \lambda_e \geq 0$ are the \emph{outer} loss weights (distinct from the inner mixing ratio $\lambda \in [0,1]$ in Eq.~\eqref{eq:g_apex}); the default setting Eq.~\eqref{eq:l_apex_total} corresponds to $\lambda_p{=}\lambda_e{=}1$.
As shown in \tabref{tab:path_endpoint_ablation}, either component alone underperforms the balanced settings.
A mild endpoint emphasis (e.g., $1.0{:}0.5$) or equal weighting ($1.0{:}1.0$) yields the highest GenEval, whereas excessive endpoint emphasis ($1.0{:}2.0$) slightly harms path integrability and overall score.
This validates our design: the fake flow fitting $\cL_{\text{fake}}$ is necessary to retain one step stability, whereas $\cL_{\text{mix}}$ is critical to reach high fidelity endpoints.

\begin{table}[!t]
  \centering
  \caption{\textbf{Ablation on the weights of $\cL_{\text{fake}}$ vs.\ $\cL_{\text{mix}}$.}
    We report GenEval Overall (NFE=1) for different weighting ratios ($\lambda_p{:}\lambda_e$). The dataset is BLIP-3o.
    Training steps are 8K for 0.6B/1.6B models and 0.4K for the 20B (LoRA) model.
    Best per model in \textbf{bold}.
  }
  \label{tab:path_endpoint_ablation}
  \vspace{-0.5em}
  \scalebox{0.7}{
    \begin{tabular}{l|c|c|c}
      \toprule
      \textbf{Weighting Ratio ($\lambda_p:\lambda_e$)} & \textbf{\method 0.6B} & \textbf{\method 1.6B} & \textbf{\method 20B ($r{=}32$)} \\
      \midrule
      1.0 : 0.0 ($\cL_{\text{fake}}$ Only)             & 0.32                  & 0.35                  & 0.42                            \\
      0.0 : 1.0 ($\cL_{\text{mix}}$ Only)              & 0.63                  & 0.66                  & 0.69                            \\
      \midrule
      1.0 : 0.5                                        & 0.72                  & 0.71                  & 0.81                            \\
      \bf 1.0 : 1.0 (Ours)                             & \bf 0.77              & \bf 0.76              & \bf 0.83                        \\
      1.0 : 2.0                                        & 0.74                  & 0.75                  & 0.82                            \\
      \bottomrule
    \end{tabular}
  }
  \vspace{-1em}
\end{table}

\textbf{\emph{$\bullet$ Condition shifting hyperparameters $a$ and $b$.}}
To probe the self conditioned contrast, we vary the scale $a$ and bias $b$ in $\cc_{\text{fake}}\;=\;\mathbf{A}\,\cc + \mathbf{b}$ (setting $\mathbf{A}=a\mathbf{I}$ and $\mathbf{b}=b\mathbf{1}$, i.e.\ scalar multiples of the identity and all ones vector) and report GenEval on a $(a,b)$ grid in \tabref{tab:condshift}.
Results show a broad optimum around $a\in\{-1.0,-0.5\}$ with small positive biases ($b\in[0.1,1.0]$), consistent with the principled justification in \secref{subsec:key_concepts}: negative scaling inverts the condition embedding direction, creating maximal representational contrast between the real and shifted branches, which enables $\vv_{\text{fake}}$ to function as a more independent estimator of $p_{\text{fake}}$'s velocity.
Positive scaling ($a{=}0.5$) is generally suboptimal unless paired with a larger bias ($b{=}10.0$) to compensate for the reduced decoupling.

\begin{wraptable}{r}{0.45\textwidth}
  \centering
  \small
  \vspace{-1em}
  \caption{\textbf{Effect of condition-shifting hyperparameters on GenEval Overall (NFE=1). Moderate negative scaling ($a\in\{-1.0,-0.5\}$) yields the most robust gains.}}
  \label{tab:condshift}
  \vspace{-6pt}
  \scalebox{0.8}{
    \begin{tabular}{l|cccc}
      \toprule
      \textbf{$a$} $\backslash$ \textbf{$b$} & 0.0  & 0.1  & 1.0  & 10.0 \\
      \midrule
      $-1.0$                                 & 0.76 & 0.73 & 0.74 & 0.74 \\
      $-0.5$                                 & 0.75 & 0.79 & 0.81 & 0.70 \\
      $\phantom{-}0.5$                       & 0.29 & 0.37 & 0.30 & 0.73 \\
      \bottomrule
    \end{tabular}
  }
  \vspace{-1em}
\end{wraptable}

\textbf{\emph{$\bullet$ Datasets vs. training steps.}}
We first study data and compute scaling by varying one factor at a time.
The dataset ablation \tabref{tab:dataset} compares ShareGPT-4o and BLIP-3o across fixed steps, evaluated on \method 0.6B and 1.6B, and extends to Qwen-Image 20B (LoRA) at shorter step budgets.
BLIP-3o consistently yields higher GenEval at larger step counts for both 0.6B and 1.6B (e.g., \textbf{0.81}/\textbf{0.83} vs \textbf{0.73} at 10K). For the 20B LoRA model, BLIP-3o reaches \textbf{0.84}--\textbf{0.85} by 1--2K steps, whereas ShareGPT-4o improves steadily with more steps (\textbf{0.19}\,$\to$\,\textbf{0.62}).

\section{Conclusion}
We presented \textcolor{color3}{\textbf{\method}}, a discriminator free one step generative framework built on self condition shifting.
\method introduces a fake condition $\cc_{\text{fake}} = \mathbf{A}\cc + \mathbf{b}$ and uses the model itself to generate a fake signal under $\cc_{\text{fake}}$, replacing the need for an external discriminator or a frozen teacher network.
The fake flow fitting loss $\cL_{\text{fake}}$ (Eq.~\eqref{eq:apex_fake_loss}) trains the fake condition branch to track the model's current generation so that $\vv_{\text{fake}}$ serves as an independent estimator of $p_{\text{fake}}$'s velocity.
The mixed consistency loss $\cL_{\text{mix}}$ then uses $\vv_{\text{fake}}$ as a correction reference, with the supervised component driving $\vv_\theta \to \vv_{\text{data}}$ and the fake correction component providing an adaptive signal that evolves as $p_\theta$ improves.
We showed that the resulting gradient takes the same score difference form as GAN objectives but with a constant weight $w \equiv 1$, connecting \method to Fisher divergence minimization without sample dependent discriminator terms.
\method attains state of the art one step quality with low latency.
At \textbf{NFE=1}, the 0.6B/1.6B models reach \textbf{0.84} GenEval at \textbf{0.20}s latency (\textbf{7.3}/\textbf{6.84} samples/s), and the 20B LoRA variant achieves \textbf{0.89} GenEval at \textbf{0.39}s latency.
At \textbf{NFE=2}, the 20B LoRA model sustains \textbf{0.89} GenEval at \textbf{0.47}s latency.
These results confirm that endogenous adversarial training via condition shifting closes the quality gap to multi-step generators while preserving the throughput advantage of one step synthesis.

\newpage
\bibliography{resources/reference}
\bibliographystyle{configuration/iclr2026_conference}

\clearpage
\appendix

\onecolumn
{
    \hypersetup{linkcolor=black}
    \parskip=0em
    \renewcommand{\contentsname}{Contents}
    \setcounter{tocdepth}{3} 
    \tableofcontents
}

\section{Related Work}

\subsection{From Macro level to Local Control}
\label{sec:rw_path}
The foundational paradigm in continuous generative modeling, including diffusion~\citep{ho2020denoising,song2020score,karras2022edm} and flow matching~\citep{lipman2022flow, liu2025efficient}, involves learning an \emph{instantaneous} velocity field. 
While effective for multi step integration, this first order approach is brittle under coarse discretization, as high path curvature causes truncation errors that degrade few step generation quality~\citep{karras2022edm}.
To address this, a significant body of work has shifted focus from instantaneous dynamics to supervising the model's behavior over a \emph{time interval}. 
These methods attempt to ensure path integrability at a macro level. 
For instance, Consistency Models (CMs)~\citep{song2023consistency,lu2024simplifying} enforce a \emph{relative} constraint, requiring that endpoint predictions remain consistent across different points on the same trajectory. While effective, this does not directly address the geometric properties of the path that cause discretization errors. 
More recent approaches such as MeanFlow~\citep{geng2025mean} and Transition Models (TiM)~\citep{wang2025transition} go a step further by directly modeling the average velocity or state transition over an interval. They learn the \emph{result} of a large step, but the constraint remains on the interval's endpoints rather than its internal geometry.
UCGM~\citep{sun2025unified} unifies different paradigms by interpolating between their respective training objectives with a hyperparameter. 
APEX takes a different approach. Rather than enforcing consistency constraints between trajectory endpoints, the fake flow fitting loss $\cL_{\text{fake}}$ (Eq.~\ref{eq:apex_fake_loss}) trains the shifted condition branch to track the model's current generation errors, providing an adaptive self adversarial signal without requiring an external network.
This internal adversarial signal, combined with data supervision in $\cL_{\text{mix}}$, drives $p_\theta$ toward $p_{\text{real}}$ in a self contained, architecture preserving manner.

\subsection{From External Discriminators to Self Adversarial Conditioning}
\label{sec:rw_endpoint}
Achieving high one step fidelity requires strong, absolute anchoring of the endpoint prediction to the data manifold, a property that relative consistency constraints alone do not guarantee.
A primary approach involves incorporating external adversarial signals. Distillation methods like DMD/DMD2~\citep{yin2024improved} and other GAN based refiners~\citep{kim2023refining,sauer2024fast,zheng2025direct} use an auxiliary discriminator to sharpen outputs, even allowing a student to surpass its teacher.
However, this reliance is a double edged sword: it introduces training instability, computational overhead, and, critically, often depends on a costly precomputed dataset for regularization. For large scale models, generating this dataset of teacher student pairs can be prohibitively expensive, exceeding the cost of training itself~\citep{yin2024improved}.
A distinct line of work generates adversarial signals internally. Direct Discriminative Optimization (DDO)~\citep{zheng2025direct} reparameterizes the GAN discriminator using the likelihood ratio between a target model and a fixed reference, operating in \emph{probability space}.
TwinFlow~\citep{cheng2025twinflow} constructs a self adversarial signal by extending the time domain to $t\in[-1,1]$, but requires modifying time embeddings and positional encodings, limiting compatibility with pretrained backbones and parameter efficient tuning.
APEX advances this line by replacing external discriminators with an endogenous adversarial signal derived from \emph{condition shifting}. The shifted condition branch $\vv_{\text{fake}}$ is trained on fake sample trajectories using the same network weights --- requiring no modification to time embeddings or model architecture --- eliminating both discriminator overhead and precomputed teacher datasets while retaining the adversarial correction signal that drives $p_\theta$ toward $p_{\text{real}}$. We further prove that this yields a gradient identical in structure to the GAN update but with constant weight $w\equiv 1$, corresponding to Fisher divergence minimization (see main paper, Section~\ref{subsec:gan_connection}).

\subsection{Scalable Training}
\label{sec:rw_scalable}
The practical implementation of  generative models, including APEX, hinges on scalable system design. 
A key challenge is the need to compute time derivatives to enforce interval consistency. 
Methods like MeanFlow~\citep{geng2025mean} relied on Jacobian-Vector Products (JVP), creating a significant scalability bottleneck. 
JVP is computationally intensive and, more importantly, incompatible with critical training optimizations like FlashAttention~\citep{dao2023flashattention} and FSDP based distributed training~\citep{zhao2023pytorch}, limiting its use in billion parameter models.
To overcome this, the field has converged on finite difference estimators, often termed Differential Derivation Equations (DDE), as a scalable alternative~\citep{lu2024simplifying,wang2025transition}. These estimators rely only on forward passes and are natively compatible with modern training infrastructure. 
APEX's path integrability objective fully embraces this scalable approach. 
This design choice, combined with our efficient endogenous adversarial mechanism and established best practices for large scale training---ensures that APEX maintains 1-NFE fidelity and any-step scaling on large backbones like SDXL, SANA, and Qwen-Image~\citep{podell2023sdxl,xie2024sana,wu2025qwen}, while remaining fully compatible with parameter efficient tuning.

\newpage

\section{Theoretical Analysis and Proofs}
\label{app:theory}

We first establish notation and basic assumptions, then prove the Score--Velocity Duality under the Optimal Transport path, the exact equivalence between endpoint space and velocity space objectives, the gradient equivalence between the mixed consistency loss and the alternative loss, and finally interpret APEX's alternative loss through the lens of Fisher divergence.

\subsection{Setup}
\label{app:notation}

We use bold lowercase letters for vectors like $\xx, \zz, \vv$ and bold uppercase letters for matrices and operators like $\mmF$. 
The identity matrix is denoted by $\mI$, and $\mathbf{0}$ represents the zero vector.
Let $p_{\text{data}}(\xx)$ denote the data distribution over $\xx \in \R^d$, and let $p(\zz) = \cN(\mathbf{0}, \mI)$ be the standard Gaussian prior over $\zz \in \R^d$. 
For conditional generation, we write $p_{\text{data}}(\xx|\cc)$ where $\cc$ is a conditioning variable like text prompt.
Throughout this appendix, we work with the Optimal Transport (OT) interpolation path defined by:
\begin{equation}
\xx_t = \alpha(t)\zz + \gamma(t)\xx, \quad t \in [0,1],
\label{eq:ot_path_app}
\end{equation}
where $\alpha(t) = t$ and $\gamma(t) = 1-t$. 
This satisfies the boundary conditions: $\xx_0 = \xx$ (pure data) and $\xx_1 = \zz$ (pure noise).
Given a time dependent random variable $\xx_t$ following Eq.~\eqref{eq:ot_path_app}, we define the conditional mean velocity.
Throughout the theory section, $\vv(\xx_t,t)$ refers to the \emph{conditional mean velocity} induced by the OT noising construction, i.e.,
\begin{equation}
\vv(\xx_t,t) := \EEb{\zz-\xx\mid \xx_t}{\zz-\xx}.
\end{equation}
The score function is $\ss_t(\xx_t) := \nabla_{\xx_t} \log p_t(\xx_t)$, where $p_t(\xx_t)$ is the marginal density of $\xx_t$ at time $t$.
The target velocity under the OT path is $\vv_{\text{data}}(\xx_t) = \zz - \xx$.
We parameterize a velocity field estimator by a neural network $\mmF_{\mtheta}: \R^d \times [0,1] \times \cC \to \R^d$, where $\mtheta$ denotes the model parameters and $\cC$ is the conditioning space. 
We use the shorthand $\mmF_{\mtheta}(\xx_t, t, \cc) \equiv \mmF_{\mtheta}$ when the arguments are clear from context.
The operator $\sg(\cdot)$ denotes stop gradient, meaning gradients do not flow through the argument. 
The fake velocity $\vv_{\text{fake}}$ is evaluated by querying the same online network $\mmF_{\mtheta}$ under the shifted condition $\cc_{\text{fake}}$ with stop gradient applied (in $\cL_{\text{cons}}$), so no separate teacher parameters are maintained.
We define the endpoint predictor that maps a velocity estimate to its implied clean sample:
\begin{equation}
\mf^{\xx}(\mmF, \xx_t, t) := \xx_t - t \cdot \mmF.
\label{eq:endpoint_pred_app}
\end{equation}
This is motivated by the OT path: if $\xx_t = t\zz + (1-t)\xx$ and $\mmF \approx \zz - \xx$, then $\mf^{\xx}(\mmF, \xx_t, t) \approx \xx$.

\subsection{Score--Velocity Duality under OT Path}
\label{app:score_velocity}

We establish the fundamental relationship between the score function and the optimal velocity field under the OT path.

\begin{proposition}[Score--Velocity Duality]
\label{prop:score_velocity}
Let $\xx_t = t\zz + (1-t)\xx$ for $t\in(0,1)$, where $\zz\sim\cN(\mathbf{0},\mI)$ and $\xx\sim p_{\text{data}}(\xx)$. 
Denote by $p_t(\xx_t)$ the marginal density of $\xx_t$, and define the OT induced conditional mean (least squares optimal) velocity field
\begin{equation}
\vv^*(\xx_t,t) := \EEb{\zz-\xx\mid \xx_t}{\zz-\xx}.
\label{eq:vstar_def}
\end{equation}
Then the score function $\ss_t(\xx_t) := \nabla_{\xx_t}\log p_t(\xx_t)$ satisfies
\begin{equation}
\ss_t(\xx_t) = -\frac{\xx_t + (1-t)\vv^*(\xx_t,t)}{t}.
\label{eq:score_velocity_relation}
\end{equation}
\end{proposition}

\begin{proof}
\textbf{Step 1: Rewrite as an additive Gaussian observation model.}
Define $\xx':=(1-t)\xx$. Then the OT path can be written as
\[
\xx_t=\xx' + t\zz,\qquad \zz\sim\cN(\mathbf{0},\mI).
\]
Conditioned on $\xx'$, the likelihood is $\xx_t\mid \xx' \sim \cN(\xx',t^2\mI)$, since $\xx_t - \xx' = t\zz$ and $\zz \sim \cN(\mathbf{0}, \mI)$ implies $t\zz \sim \cN(\mathbf{0}, t^2\mI)$.

\textbf{Step 2: Apply Tweedie's formula to recover the posterior mean.}
For an additive Gaussian model $\xx_t=\xx' + t\zz$ where $\zz \sim \cN(\mathbf{0}, \mI)$, Tweedie's formula states that the posterior mean can be recovered from the score function:
\begin{equation}
\EEb{\xx' \mid \xx_t}{\xx'} = \xx_t + t^2\,\nabla_{\xx_t}\log p_t(\xx_t) = \xx_t + t^2\,\ss_t(\xx_t).
\label{eq:tweedie_xprime}
\end{equation}
\emph{Justification of Tweedie's formula:} For a Gaussian perturbation model $\yy = \xx' + \sigma \boldsymbol{\epsilon}$ with $\boldsymbol{\epsilon} \sim \cN(\mathbf{0}, \mI)$, we have
\[
\EEb{\xx' \mid \yy}{\xx'} = \yy + \sigma^2 \nabla_{\yy} \log p(\yy).
\]
In our case, $\yy = \xx_t$, $\xx' = (1-t)\xx$, and $\sigma = t$, so Eq.~\eqref{eq:tweedie_xprime} follows directly.

Since $\xx'=(1-t)\xx$, we can recover the conditional expectation of $\xx$:
\begin{align}
\EEb{\xx \mid \xx_t}{\xx} 
&= \frac{1}{1-t}\EEb{\xx'\mid \xx_t}{\xx'} \nonumber\\
&= \frac{1}{1-t}\left[\xx_t + t^2\,\ss_t(\xx_t)\right] \nonumber\\
&= \frac{\xx_t + t^2\,\ss_t(\xx_t)}{1-t}.
\label{eq:cond_x}
\end{align}

\textbf{Step 3: Express the conditional mean of $\zz$.}
From the OT path $\xx_t=t\zz+(1-t)\xx$, we can solve for $\zz$:
\[
\zz=\frac{\xx_t-(1-t)\xx}{t}.
\]
Taking conditional expectations on both sides given $\xx_t$:
\begin{align}
\EEb{\zz\mid \xx_t}{\zz}
&= \EEb{}{\frac{\xx_t-(1-t)\xx}{t} \,\Big|\, \xx_t} \nonumber\\
&= \frac{1}{t}\Bigl[\xx_t - (1-t)\EEb{\xx\mid \xx_t}{\xx}\Bigr] \nonumber\\
&= \frac{1}{t}\Bigl[\xx_t - (1-t)\cdot \frac{\xx_t + t^2\,\ss_t(\xx_t)}{1-t}\Bigr] \quad \text{(substituting Eq.~\eqref{eq:cond_x})} \nonumber\\
&= \frac{1}{t}\Bigl[\xx_t - \bigl(\xx_t + t^2\,\ss_t(\xx_t)\bigr)\Bigr] \quad \text{(simplifying the fraction)} \nonumber\\
&= \frac{1}{t}\bigl[-t^2\,\ss_t(\xx_t)\bigr] \nonumber\\
&= -t\,\ss_t(\xx_t).
\label{eq:cond_z_clean}
\end{align}

\textbf{Step 4: Form the optimal velocity and rearrange.}
By definition, the (least squares) optimal velocity field along the OT path is the conditional expectation of the target velocity $\zz - \xx$:
\[
\vv^*(\xx_t,t):=\EEb{\zz-\xx\mid \xx_t}{\zz-\xx}=\EEb{\zz\mid \xx_t}{\zz}-\EEb{\xx\mid \xx_t}{\xx}.
\]
Substituting Eq.~\eqref{eq:cond_x} and Eq.~\eqref{eq:cond_z_clean}:
\begin{align}
\vv^*(\xx_t,t)
&= -t\,\ss_t(\xx_t) - \frac{\xx_t + t^2\,\ss_t(\xx_t)}{1-t} \nonumber\\
&= \frac{-t(1-t)\,\ss_t(\xx_t) - (\xx_t + t^2\,\ss_t(\xx_t))}{1-t} \quad \text{(common denominator)} \nonumber\\
&= \frac{-t\,\ss_t(\xx_t) + t^2\,\ss_t(\xx_t) - \xx_t - t^2\,\ss_t(\xx_t)}{1-t} \nonumber\\
&= \frac{-t\,\ss_t(\xx_t) - \xx_t}{1-t} \nonumber\\
&= -\frac{\xx_t + t\,\ss_t(\xx_t)}{1-t}.
\label{eq:vstar_in_terms_of_score}
\end{align}

\textbf{Step 5: Rearrange to obtain the score-velocity duality.}
Multiplying both sides of Eq.~\eqref{eq:vstar_in_terms_of_score} by $(1-t)$:
\[
(1-t)\vv^*(\xx_t,t) = -\xx_t - t\,\ss_t(\xx_t).
\]
Rearranging:
\[
\xx_t + (1-t)\vv^*(\xx_t,t) = -t\,\ss_t(\xx_t),
\]
which, upon dividing both sides by $-t$, gives exactly Eq.~\eqref{eq:score_velocity_relation}:
\[
\ss_t(\xx_t) = -\frac{\xx_t + (1-t)\vv^*(\xx_t,t)}{t}.
\]
\end{proof}

\begin{corollary}[Velocity Difference as Score Difference]
\label{cor:velocity_score}
For any two OT noising constructions that induce marginals $p_{1,t},p_{2,t}$ and corresponding conditional mean velocities $\vv_i(\xx_t,t):=\EEb{\zz-\xx\mid \xx_t}{\zz-\xx}$ ($i\in\{1,2\}$) at the same $(\xx_t,t)$, their velocity difference and score difference satisfy
\begin{equation}
\vv_1(\xx_t,t) - \vv_2(\xx_t,t) = -\frac{t}{1-t}\left[\ss_1(\xx_t) - \ss_2(\xx_t)\right].
\label{eq:velocity_score_diff}
\end{equation}
\end{corollary}

\begin{proof}
Applying Proposition~\ref{prop:score_velocity} to both velocity fields:
\begin{align}
\ss_1(\xx_t) &= -\frac{\xx_t + (1-t)\vv_1(\xx_t,t)}{t}, \label{eq:score1} \\
\ss_2(\xx_t) &= -\frac{\xx_t + (1-t)\vv_2(\xx_t,t)}{t}. \label{eq:score2}
\end{align}
Subtracting Eq.~\eqref{eq:score2} from Eq.~\eqref{eq:score1}:
\begin{align}
\ss_1(\xx_t) - \ss_2(\xx_t) 
&= -\frac{1}{t}\left[(\xx_t + (1-t)\vv_1) - (\xx_t + (1-t)\vv_2)\right] \nonumber \\
&= -\frac{1-t}{t}\left[\vv_1(\xx_t,t) - \vv_2(\xx_t,t)\right].
\end{align}
Rearranging yields Eq.~\eqref{eq:velocity_score_diff}.
\end{proof}

\subsection{KL Gradient in Velocity Space}
\label{app:kl_velocity}

We now show how the KL divergence gradient between two flow-induced distributions can be expressed purely in terms of their velocity fields. This result is fundamental to understanding how APEX's training objective connects to distribution matching.
\begin{lemma}[Gradient of KL Divergence via Reparameterization]
    \label{lemma:kl_grad_reparam}
    Let $p_{\mtheta}(\xx)$ be a probability density parameterized by $\mtheta$, defined by the push-forward of a fixed base distribution $p(\zz)$ through a differentiable mapping $\xx = T_{\mtheta}(\zz)$ (the reparameterization trick). Let $q(\xx)$ be a target distribution independent of $\mtheta$. The gradient of the KL divergence $D_{\mathrm{KL}}(p_{\mtheta} \| q)$ with respect to $\mtheta$ satisfies:
    \begin{equation}
    \nabla_{\mtheta} D_{\mathrm{KL}}(p_{\mtheta} \| q) = \E_{\zz \sim p(\zz)}{ (\ss_{\mtheta}(T_{\mtheta}(\zz)) - \ss_{q}(T_{\mtheta}(\zz))) \cdot \nabla_{\mtheta} T_{\mtheta}(\zz) }.
    \label{eq:kl_grad_reparam}
    \end{equation}
    where $\ss_{\mtheta}(\xx) = \nabla_{\xx} \log p_{\mtheta}(\xx)$ and $\ss_{q}(\xx) = \nabla_{\xx} \log q(\xx)$ are the score functions of the model and target distributions, respectively.
\end{lemma}
    
\begin{proof}
    Consider the KL divergence defined as an expectation over the reparameterized variable $\zz$:
    \begin{equation}
    \mathcal{L}(\mtheta) = D_{\mathrm{KL}}(p_{\mtheta} \| q) = \E{\zz \sim p(\zz)}{ \log p_{\mtheta}(T_{\mtheta}(\zz)) - \log q(T_{\mtheta}(\zz)) }.
    \end{equation}
    Since the base distribution $p(\zz)$ does not depend on $\mtheta$, we can move the gradient operator $\nabla_{\mtheta}$ inside the expectation. Applying the total derivative (chain rule) to the terms inside the expectation yields:
    \begin{align}
    \nabla_{\mtheta} \mathcal{L}(\mtheta) 
    &= \E_{\zz}{\left[ \nabla_{\mtheta} \big( \log p_{\mtheta}(\xx) \big)\big|_{\xx=T_{\mtheta}(\zz)} - \nabla_{\mtheta} \big( \log q(\xx) \big)\big|_{\xx=T_{\mtheta}(\zz)} \right] } \nonumber \\
    &= \E_{\zz}{ \left[ \nabla_{\mtheta} \log p_{\mtheta}(\xx) \Big|_{\text{fixed } \xx} + \nabla_{\xx} \log p_{\mtheta}(\xx) \cdot \frac{\partial \xx}{\partial \mtheta} - \nabla_{\xx} \log q(\xx) \cdot \frac{\partial \xx}{\partial \mtheta} \right] }.
    \label{eq:chain_rule_expansion}
    \end{align}
    Note that the first term corresponds to the standard score function estimator identity, which vanishes under expectation:
    \begin{equation}
    \E_{\xx \sim p_{\mtheta}}{ \nabla_{\mtheta} \log p_{\mtheta}(\xx) \Big|_{\text{fixed } \xx} } = \int \nabla_{\mtheta} p_{\mtheta}(\xx) \, \dm\xx = \nabla_{\mtheta} \int p_{\mtheta}(\xx) \, \dm\xx = \nabla_{\mtheta}(1) = 0.
    \end{equation}
    Substituting the definitions of the score functions $\ss_{\mtheta} = \nabla_{\xx} \log p_{\mtheta}$ and $\ss_{q} = \nabla_{\xx} \log q$ into Eq.~\eqref{eq:chain_rule_expansion}, and removing the zero-mean term, we obtain:
    \begin{align}
    \nabla_{\mtheta} \mathcal{L}(\mtheta) 
    &= \E_{\zz}{ \ss_{\mtheta}(\xx) \cdot \frac{\partial \xx}{\partial \mtheta} - \ss_{q}(\xx) \cdot \frac{\partial \xx}{\partial \mtheta} } \nonumber \\
    &= \E_{\xx \sim p_{\mtheta}}{ (\ss_{\mtheta}(\xx) - \ss_{q}(\xx)) \cdot \frac{\partial \xx}{\partial \mtheta} }.
    \end{align}
\end{proof}

\begin{proposition}[KL Gradient via Velocity Difference]
\label{prop:kl_velocity}
Let $p_{\text{fake}}(\xx|\cc)$ be the distribution induced by a flow with velocity field $\vv_{\mtheta}(\xx_t,t,\cc)$, and $p_{\text{real}}(\xx|\cc)$ the data distribution with velocity $\vv_{\text{data}}(\xx_t) = \zz - \xx$ under the OT path. 
Then the gradient of the KL divergence with respect to model parameters $\mtheta$ satisfies
\begin{equation}
\nabla_{\mtheta} D_{\text{KL}}(p_{\text{fake}} \| p_{\text{real}}) 
= -\frac{1}{\omega(t)}\,\EEb{\xx_t,\zz,t}{\bigl(\vv_{\mtheta}(\xx_t,t,\cc) - \vv_{\text{data}}(\xx_t)\bigr) \cdot \frac{\partial\xx_t}{\partial\mtheta}},
\label{eq:kl_velocity_grad}
\end{equation}
where $\omega(t) = \frac{t}{1-t} > 0$ is a positive time dependent weight.
Since $\omega(t) > 0$, this gradient drives $\vv_{\mtheta} \to \vv_{\text{data}}$ under gradient descent,
confirming that minimizing $D_{\mathrm{KL}}$ is equivalent to regressing $\vv_{\mtheta}$ toward the real data velocity.
\end{proposition}

\begin{proof}
    We derive the gradient by directly applying \lemref{lemma:kl_grad_reparam}. Let the model distribution be $p_{\text{fake}}$ (parameterized by $\mtheta$) and the target distribution be $p_{\text{real}}$. By identifying the reparameterization mapping as the flow trajectory $\xx_t$, \lemref{lemma:kl_grad_reparam} implies that the gradient of the KL divergence is the expectation of the dot product between the score difference and the path gradient:
    \begin{equation}
        \nabla_{\mtheta} D_{\text{KL}}(p_{\text{fake}} \| p_{\text{real}}) 
        = \E_{\xx_t \sim p_{\text{fake}}} \left[ (\ss_{\text{fake}}(\xx_t) - \ss_{\text{real}}(\xx_t)) \cdot \frac{\partial \xx_t}{\partial \mtheta} \right],
        \label{eq:kl_grad_lemma_app}
    \end{equation}
    where $\ss_{\text{fake}}(\xx_t) = \nabla_{\xx_t} \log p_{\text{fake}}(\xx_t)$ and $\ss_{\text{real}}(\xx_t) = \nabla_{\xx_t} \log p_{\text{real}}(\xx_t)$.

    Next, we invoke the duality between score and velocity fields for Optimal Transport paths (\corref{cor:velocity_score}). The difference between the model score and the target score is proportional to the difference between their respective velocity fields:
    \begin{equation}
        \ss_{\text{fake}}(\xx_t) - \ss_{\text{real}}(\xx_t) 
        = -\frac{1-t}{t} \bigl( \vv_{\mtheta}(\xx_t, t, \cc) - \vv_{\text{data}}(\xx_t) \bigr).
        \label{eq:score_velocity_relation_app}
    \end{equation}
    Substituting Eq.~\eqref{eq:score_velocity_relation_app} into Eq.~\eqref{eq:kl_grad_lemma_app}, we obtain:
    \begin{equation}
        \nabla_{\mtheta} D_{\text{KL}}(p_{\text{fake}} \| p_{\text{real}}) 
        = \E_{\xx_t, \zz, t} \left[ -\frac{1-t}{t} \bigl(\vv_{\mtheta}(\xx_t, t, \cc) - \vv_{\text{data}}(\xx_t)\bigr) \cdot \frac{\partial \xx_t}{\partial \mtheta} \right].
    \end{equation}
    Defining $\omega(t) := \frac{t}{1-t} > 0$, we identify $-\frac{1-t}{t} = -\frac{1}{\omega(t)}$, giving exactly:
    \begin{equation}
        \nabla_{\mtheta} D_{\text{KL}}(p_{\text{fake}} \| p_{\text{real}}) 
        = -\frac{1}{\omega(t)}\,\E_{\xx_t, \zz, t} \left[ \bigl(\vv_{\mtheta}(\xx_t, t, \cc) - \vv_{\text{data}}(\xx_t)\bigr) \cdot \frac{\partial \xx_t}{\partial \mtheta} \right],
    \end{equation}
    which establishes Eq.~\eqref{eq:kl_velocity_grad}.
\end{proof}

\subsection{Endpoint--Velocity Equivalence}
\label{app:endpoint-proof}

We prove that the endpoint space MSE and velocity space MSE are exactly equivalent up to a scalar factor $t^2$. 
This result establishes that training objectives formulated in either space are mathematically interchangeable.

\begin{proposition}[Endpoint--Velocity Equivalence for Supervised FM]
\label{prop:endpoint_velocity}
Let $\mf^{\xx}(\mmF, \xx_t, t) := \xx_t - t\,\mmF$ be the endpoint predictor defined in Eq.~\ref{eq:endpoint_predictor}, and let $\vv_{\text{data}}(\xx_t) = \zz - \xx$ be the target velocity under the OT path. 
Then for any velocity estimate $\mmF_{\mtheta}$, we have
\begin{equation}
\bigl\|\mf^{\xx}(\mmF_{\mtheta}, \xx_t, t) - \xx\bigr\|_2^2 
= t^2\,\bigl\|\mmF_{\mtheta} - \vv_{\text{data}}(\xx_t)\bigr\|_2^2.
\label{eq:endpoint_velocity_equiv1}
\end{equation}
\end{proposition}

\begin{proof}
\textbf{Step 1: Expand the endpoint predictor.}
By definition of the endpoint predictor in Eq.~\eqref{eq:endpoint_pred_app}, we have
\begin{equation}
\mf^{\xx}(\mmF_{\mtheta}, \xx_t, t) = \xx_t - t\,\mmF_{\mtheta}.
\label{eq:endpoint_expand}
\end{equation}

\textbf{Step 2: Compute the squared error.}
The LHS of Eq.~\eqref{eq:endpoint_velocity_equiv1} is
\begin{align}
\bigl\|\mf^{\xx}(\mmF_{\mtheta}, \xx_t, t) - \xx\bigr\|_2^2
&= \bigl\|(\xx_t - t\,\mmF_{\mtheta}) - \xx\bigr\|_2^2 \nonumber \\
&= \bigl\|(\xx_t - \xx) - t\,\mmF_{\mtheta}\bigr\|_2^2.
\label{eq:endpoint_error_expand}
\end{align}

\textbf{Step 3: Use the OT path identity.}
Under the OT path $\xx_t = t\zz + (1-t)\xx$ from Eq.~\eqref{eq:ot_path_app}, we compute the difference $\xx_t - \xx$ step by step:
\begin{align}
\xx_t - \xx 
&= \bigl[t\zz + (1-t)\xx\bigr] - \xx \quad \text{(substituting the OT path)} \nonumber \\
&= t\zz + (1-t)\xx - \xx \nonumber \\
&= t\zz + (1-t)\xx - 1 \cdot \xx \quad \text{(writing $\xx = 1 \cdot \xx$)} \nonumber \\
&= t\zz + \xx - t\xx - \xx \quad \text{(expanding $(1-t)\xx$)} \nonumber \\
&= t\zz - t\xx \quad \text{(canceling $\xx$)} \nonumber \\
&= t(\zz - \xx). \quad \text{(factoring out $t$)}
\label{eq:ot_diff}
\end{align}
Recall that under the OT path, the target velocity is defined as $\vv_{\text{data}}(\xx_t) := \zz - \xx$, which is the instantaneous rate of change from data $\xx$ to noise $\zz$. 
Therefore, we obtain the key identity:
\begin{equation}
\xx_t - \xx = t\,\vv_{\text{data}}(\xx_t).
\label{eq:xt_minus_x}
\end{equation}
This identity says that the displacement from the clean data $\xx$ to the noised sample $\xx_t$ is exactly $t$ times the target velocity, which makes intuitive sense since we've traveled for "time" $t$ along the trajectory.

\textbf{Step 4: Substitute and simplify.}
Substituting Eq.~\eqref{eq:xt_minus_x} into Eq.~\eqref{eq:endpoint_error_expand}:
\begin{align}
\bigl\|(\xx_t - \xx) - t\,\mmF_{\mtheta}\bigr\|_2^2
&= \bigl\|t\,\vv_{\text{data}}(\xx_t) - t\,\mmF_{\mtheta}\bigr\|_2^2 \quad \text{(using Eq.~\eqref{eq:xt_minus_x})} \nonumber \\
&= \bigl\|t\,(\vv_{\text{data}}(\xx_t) - \mmF_{\mtheta})\bigr\|_2^2 \quad \text{(factoring out $t$)} \nonumber \\
&= t^2\,\bigl\|\vv_{\text{data}}(\xx_t) - \mmF_{\mtheta}\bigr\|_2^2, \quad \text{(using $\|c\vv\|_2^2 = c^2\|\vv\|_2^2$)}
\label{eq:endpoint_final}
\end{align}
which proves Eq.~\eqref{eq:endpoint_velocity_equiv1}. The final step uses the homogeneity property of the squared $\ell_2$ norm.

\textbf{Geometric interpretation:} This result shows that predicting the clean endpoint $\xx$ is equivalent to predicting the velocity $\zz - \xx$, scaled by the time factor $t$. When $t$ is small (near clean data), the endpoint prediction is very sensitive to velocity errors. When $t$ is large (near pure noise), the endpoint prediction is less sensitive, which motivates using time dependent weighting $\omega(t)$ in the loss.
\end{proof}

\begin{proposition}[Endpoint--Velocity Equivalence for Fake Alignment]
\label{prop:endpoint_velocity_fake}
For the fake alignment term, let $\vv_{\text{fake}}(\xx_t, t, \cc_{\text{fake}}) := \sg\!\left(\mmF_{\mtheta}(\xx_t, t, \cc_{\text{fake}})\right)$ be the fake velocity field obtained by querying the same online network $\mmF_{\mtheta}$ under the shifted condition $\cc_{\text{fake}}$ with stop gradient applied. 
Then
\begin{equation}
\bigl\|\mf^{\xx}(\mmF_{\mtheta}, \xx_t, t) - \mf^{\xx}(\vv_{\text{fake}}, \xx_t, t)\bigr\|_2^2 
= t^2\,\bigl\|\mmF_{\mtheta} - \vv_{\text{fake}}(\xx_t,t,\cc_{\text{fake}})\bigr\|_2^2.
\label{eq:endpoint_velocity_equiv2}
\end{equation}
\end{proposition}

\begin{proof}
\textbf{Step 1: Expand both endpoint predictors.}
By definition,
\begin{align}
\mf^{\xx}(\mmF_{\mtheta}, \xx_t, t) &= \xx_t - t\,\mmF_{\mtheta}, \\
\mf^{\xx}(\vv_{\text{fake}}, \xx_t, t) &= \xx_t - t\,\vv_{\text{fake}}(\xx_t, t, \cc_{\text{fake}}).
\end{align}

\textbf{Step 2: Compute the difference.}
\begin{align}
\mf^{\xx}(\mmF_{\mtheta}, \xx_t, t) - \mf^{\xx}(\vv_{\text{fake}}, \xx_t, t)
&= \bigl[\xx_t - t\,\mmF_{\mtheta}\bigr] - \bigl[\xx_t - t\,\vv_{\text{fake}}\bigr] \nonumber \\
&= \xx_t - t\,\mmF_{\mtheta} - \xx_t + t\,\vv_{\text{fake}} \nonumber \\
&= t\,\vv_{\text{fake}} - t\,\mmF_{\mtheta} \nonumber \\
&= t\,(\vv_{\text{fake}} - \mmF_{\mtheta}).
\label{eq:endpoint_diff_fake}
\end{align}

\textbf{Step 3: Square the norm.}
\begin{align}
\bigl\|\mf^{\xx}(\mmF_{\mtheta}, \xx_t, t) - \mf^{\xx}(\vv_{\text{fake}}, \xx_t, t)\bigr\|_2^2
&= \bigl\|t\,(\vv_{\text{fake}} - \mmF_{\mtheta})\bigr\|_2^2 \nonumber \\
&= t^2\,\bigl\|\vv_{\text{fake}} - \mmF_{\mtheta}\bigr\|_2^2 \nonumber \\
&= t^2\,\bigl\|\mmF_{\mtheta} - \vv_{\text{fake}}(\xx_t,t,\cc_{\text{fake}})\bigr\|_2^2,
\label{eq:endpoint_final_fake}
\end{align}
which proves Eq.~\eqref{eq:endpoint_velocity_equiv2}.
\end{proof}

\subsection{Gradient Equivalence of Alternative Loss}
\label{app:alt-loss-proof}

We now prove the key theoretical result: the gradient of the mixed consistency loss $\cL_{\text{mix}}$ is exactly equal to the gradient of the alternative loss $\cG_{\text{APEX}}$. 
This establishes that these two seemingly different objectives induce identical training dynamics in parameter space.

\begin{theorem}[Gradient Equivalence]
\label{thm:gradient_equivalence}
Let $\cL_{\text{mix}}(\mtheta)$ and $\cG_{\text{APEX}}(\mtheta)$ be defined as in Eq.~\ref{eq:l_mix} and Eq.~\ref{eq:g_apex}, respectively. 
Then for any parameter $\mtheta$,
\begin{equation}
\nabla_{\mtheta}\,\cL_{\text{mix}}(\mtheta) = \nabla_{\mtheta}\,\cG_{\text{APEX}}(\mtheta).
\label{eq:gradient_equiv_thm}
\end{equation}
\end{theorem}

\begin{proof}
For notational simplicity, we focus on a single sample and omit the expectation $\EEb{\xx_t,\zz,t}{\cdot}$ and the weighting $\frac{1}{\omega(t)}$ (these are linear operations that commute with gradients). 
We use the shorthand $\mmF_{\mtheta} \equiv \mmF_{\mtheta}(\xx_t,t,\cc)$ and $\vv_{\text{fake}} \equiv \vv_{\text{fake}}(\xx_t,t,\cc_{\text{fake}})$.

\paragraph{Part A: Gradient of the mixed consistency loss.}

\textbf{Step A1: Write the mixed consistency loss.}
From Eq.~\ref{eq:l_mix}, the mixed consistency loss is
\begin{equation}
\cL_{\text{mix}}(\mtheta) = \bigl\|\mf^{\xx}(\mmF_{\mtheta}, \xx_t, t) - \mathbf{T}_{\text{mix}}(\xx_t,t)\bigr\|_2^2,
\label{eq:lmix_def}
\end{equation}
where the mixed target is defined in Eq.~\ref{eq:tmix} as
\begin{equation}
\mathbf{T}_{\text{mix}}(\xx_t,t) = (1-\lambda)\,\xx + \lambda\,\mf^{\xx}(\vv_{\text{fake}}, \xx_t, t).
\label{eq:tmix_def}
\end{equation}

\textbf{Step A2: Expand the endpoint predictors.}
Using the definition $\mf^{\xx}(\mmF, \xx_t, t) = \xx_t - t\,\mmF$ from Eq.~\eqref{eq:endpoint_pred_app}:
\begin{align}
\mf^{\xx}(\mmF_{\mtheta}, \xx_t, t) &= \xx_t - t\,\mmF_{\mtheta}, \label{eq:fx_student} \\
\mf^{\xx}(\vv_{\text{fake}}, \xx_t, t) &= \xx_t - t\,\vv_{\text{fake}}. \label{eq:fx_fake}
\end{align}

\textbf{Step A3: Substitute into the mixed target.}
Substituting Eq.~\eqref{eq:fx_fake} into Eq.~\eqref{eq:tmix_def}:
\begin{align}
\mathbf{T}_{\text{mix}}(\xx_t,t) 
&= (1-\lambda)\,\xx + \lambda\,(\xx_t - t\,\vv_{\text{fake}}) \nonumber \\
&= (1-\lambda)\,\xx + \lambda\,\xx_t - \lambda t\,\vv_{\text{fake}}.
\label{eq:tmix_expanded}
\end{align}

\textbf{Step A4: Compute the error term $\Delta$.}
Define the error as
\begin{equation}
\Delta := \mf^{\xx}(\mmF_{\mtheta}, \xx_t, t) - \mathbf{T}_{\text{mix}}(\xx_t,t).
\label{eq:delta_def}
\end{equation}
Substituting Eq.~\eqref{eq:fx_student} and Eq.~\eqref{eq:tmix_expanded}:
\begin{align}
\Delta 
&= (\xx_t - t\,\mmF_{\mtheta}) - \bigl[(1-\lambda)\,\xx + \lambda\,\xx_t - \lambda t\,\vv_{\text{fake}}\bigr] \nonumber \\
&= \xx_t - t\,\mmF_{\mtheta} - (1-\lambda)\,\xx - \lambda\,\xx_t + \lambda t\,\vv_{\text{fake}} \nonumber \\
&= \xx_t(1 - \lambda) - (1-\lambda)\,\xx - t\,\mmF_{\mtheta} + \lambda t\,\vv_{\text{fake}} \nonumber \\
&= (1-\lambda)(\xx_t - \xx) - t\,\mmF_{\mtheta} + \lambda t\,\vv_{\text{fake}}.
\label{eq:delta_step1}
\end{align}

\textbf{Step A5: Apply the OT path identity.}
From Eq.~\eqref{eq:xt_minus_x} (proven in Section~\ref{app:endpoint-proof}), we have
\begin{equation}
\xx_t - \xx = t\,\vv_{\text{data}}, \quad \text{where } \vv_{\text{data}} = \zz - \xx.
\label{eq:ot_identity_use}
\end{equation}
Substituting into Eq.~\eqref{eq:delta_step1}:
\begin{align}
\Delta 
&= (1-\lambda)\,t\,\vv_{\text{data}} - t\,\mmF_{\mtheta} + \lambda t\,\vv_{\text{fake}} \nonumber \\
&= t\bigl[(1-\lambda)\,\vv_{\text{data}} + \lambda\,\vv_{\text{fake}} - \mmF_{\mtheta}\bigr].
\label{eq:delta_final}
\end{align}

\textbf{Step A6: Compute the gradient using the chain rule.}
The gradient of the squared norm $\cL_{\text{mix}} = \|\Delta\|_2^2$ with respect to $\mtheta$ is
\begin{equation}
\nabla_{\mtheta}\,\cL_{\text{mix}}(\mtheta) = 2\,\langle \Delta, \nabla_{\mtheta}\Delta \rangle,
\label{eq:grad_lmix_chain}
\end{equation}
where $\langle \cdot, \cdot \rangle$ denotes the inner product. This follows from the chain rule for the squared norm:
\[
\nabla_{\mtheta} \|\Delta(\mtheta)\|_2^2 = \nabla_{\mtheta} \langle \Delta, \Delta \rangle = 2\langle \Delta, \nabla_{\mtheta}\Delta \rangle.
\]

Since $\Delta$ depends on $\mtheta$ only through $\mmF_{\mtheta}$ (note that $\vv_{\text{data}} = \zz - \xx$ does not depend on $\mtheta$, and $\vv_{\text{fake}} = \sg(\mmF_{\mtheta}(\xx_t, t, \cc_{\text{fake}}))$ has stop gradient applied, so gradients do not flow through $\vv_{\text{fake}}$), we have
\begin{equation}
\nabla_{\mtheta}\Delta = \nabla_{\mtheta}\bigl[t[(1-\lambda)\,\vv_{\text{data}} + \lambda\,\vv_{\text{fake}} - \mmF_{\mtheta}]\bigr] = -t\,\nabla_{\mtheta}\mmF_{\mtheta}.
\label{eq:grad_delta}
\end{equation}

\textbf{Step A7: Substitute and simplify.}
Substituting Eq.~\eqref{eq:delta_final} and Eq.~\eqref{eq:grad_delta} into Eq.~\eqref{eq:grad_lmix_chain}:
\begin{align}
\nabla_{\mtheta}\,\cL_{\text{mix}}(\mtheta)
&= 2\,\left\langle t\bigl[(1-\lambda)\,\vv_{\text{data}} + \lambda\,\vv_{\text{fake}} - \mmF_{\mtheta}\bigr], -t\,\nabla_{\mtheta}\mmF_{\mtheta} \right\rangle \nonumber \\
&= -2t^2\,\left\langle (1-\lambda)\,\vv_{\text{data}} + \lambda\,\vv_{\text{fake}} - \mmF_{\mtheta}, \nabla_{\mtheta}\mmF_{\mtheta} \right\rangle \nonumber \\
&= 2t^2\,\left\langle \mmF_{\mtheta} - (1-\lambda)\,\vv_{\text{data}} - \lambda\,\vv_{\text{fake}}, \nabla_{\mtheta}\mmF_{\mtheta} \right\rangle.
\label{eq:grad_lmix_step}
\end{align}

\textbf{Step A8: Distribute the inner product.}
Using the bilinearity of the inner product, we expand:
\begin{align}
\nabla_{\mtheta}\,\cL_{\text{mix}}(\mtheta)
&= 2t^2\,\Bigl[\langle \mmF_{\mtheta}, \nabla_{\mtheta}\mmF_{\mtheta} \rangle 
- (1-\lambda)\langle \vv_{\text{data}}, \nabla_{\mtheta}\mmF_{\mtheta} \rangle 
- \lambda\langle \vv_{\text{fake}}, \nabla_{\mtheta}\mmF_{\mtheta} \rangle\Bigr] \nonumber
\end{align}
Now we regroup the terms by factoring out $(1-\lambda)$ and $\lambda$. Note that:
\begin{align*}
\langle \mmF_{\mtheta}, \nabla_{\mtheta}\mmF_{\mtheta} \rangle 
&= (1-\lambda)\langle \mmF_{\mtheta}, \nabla_{\mtheta}\mmF_{\mtheta} \rangle + \lambda\langle \mmF_{\mtheta}, \nabla_{\mtheta}\mmF_{\mtheta} \rangle
\end{align*}
Substituting back:
\begin{align}
\nabla_{\mtheta}\,\cL_{\text{mix}}(\mtheta)
&= 2t^2\,\Bigl[(1-\lambda)\langle \mmF_{\mtheta}, \nabla_{\mtheta}\mmF_{\mtheta} \rangle - (1-\lambda)\langle \vv_{\text{data}}, \nabla_{\mtheta}\mmF_{\mtheta} \rangle \nonumber\\
&\qquad\qquad + \lambda\langle \mmF_{\mtheta}, \nabla_{\mtheta}\mmF_{\mtheta} \rangle - \lambda\langle \vv_{\text{fake}}, \nabla_{\mtheta}\mmF_{\mtheta} \rangle\Bigr] \nonumber \\
&= 2t^2\,\Bigl[(1-\lambda)\langle \mmF_{\mtheta} - \vv_{\text{data}}, \nabla_{\mtheta}\mmF_{\mtheta} \rangle 
+ \lambda\langle \mmF_{\mtheta} - \vv_{\text{fake}}, \nabla_{\mtheta}\mmF_{\mtheta} \rangle\Bigr].
\label{eq:grad_lmix_final}
\end{align}

\paragraph{Part B: Gradient of the alternative loss.}

\textbf{Step B1: Write the alternative loss.}
From Eq.~\ref{eq:g_apex}, the alternative loss is
\begin{equation}
\cG_{\text{APEX}}(\mtheta) = (1-\lambda)\,\cL_{\text{sup}}(\mtheta) + \lambda\,\cL_{\text{cons}}(\mtheta),
\label{eq:gapex_def}
\end{equation}
where $\cL_{\text{sup}}$ and $\cL_{\text{cons}}$ are defined in Eq.~\ref{eq:l_sup} and Eq.~\ref{eq:l_cons}.

\textbf{Step B2: Apply the endpoint-velocity equivalence.}
By Proposition~\ref{prop:endpoint_velocity}, we have
\begin{equation}
\cL_{\text{sup}}(\mtheta) = \bigl\|\mf^{\xx}(\mmF_{\mtheta}, \xx_t, t) - \xx\bigr\|_2^2 
= t^2\,\bigl\|\mmF_{\mtheta} - \vv_{\text{data}}\bigr\|_2^2.
\label{eq:lsup_velocity}
\end{equation}
By Proposition~\ref{prop:endpoint_velocity_fake}, we have
\begin{equation}
\cL_{\text{cons}}(\mtheta) = \bigl\|\mf^{\xx}(\mmF_{\mtheta}, \xx_t, t) - \mf^{\xx}(\vv_{\text{fake}}, \xx_t, t)\bigr\|_2^2 
= t^2\,\bigl\|\mmF_{\mtheta} - \vv_{\text{fake}}\bigr\|_2^2.
\label{eq:lcons_velocity}
\end{equation}

\textbf{Step B3: Compute the gradients of $\cL_{\text{sup}}$ and $\cL_{\text{cons}}$.}
Using the gradient of a squared norm (Lemma from UCGM appendix):
\begin{align}
\nabla_{\mtheta}\,\cL_{\text{sup}}(\mtheta) 
&= \nabla_{\mtheta}\left[t^2\,\bigl\|\mmF_{\mtheta} - \vv_{\text{data}}\bigr\|_2^2\right] \nonumber \\
&= t^2\,\nabla_{\mtheta}\bigl\|\mmF_{\mtheta} - \vv_{\text{data}}\bigr\|_2^2 \nonumber \\
&= t^2 \cdot 2\,\langle \mmF_{\mtheta} - \vv_{\text{data}}, \nabla_{\mtheta}\mmF_{\mtheta} \rangle \nonumber \\
&= 2t^2\,\langle \mmF_{\mtheta} - \vv_{\text{data}}, \nabla_{\mtheta}\mmF_{\mtheta} \rangle.
\label{eq:grad_lsup}
\end{align}
Similarly,
\begin{equation}
\nabla_{\mtheta}\,\cL_{\text{cons}}(\mtheta) 
= 2t^2\,\langle \mmF_{\mtheta} - \vv_{\text{fake}}, \nabla_{\mtheta}\mmF_{\mtheta} \rangle.
\label{eq:grad_lcons}
\end{equation}

\textbf{Step B4: Combine the gradients.}
Substituting Eq.~\eqref{eq:grad_lsup} and Eq.~\eqref{eq:grad_lcons} into the gradient of Eq.~\eqref{eq:gapex_def}:
\begin{align}
\nabla_{\mtheta}\,\cG_{\text{APEX}}(\mtheta)
&= (1-\lambda)\,\nabla_{\mtheta}\,\cL_{\text{sup}}(\mtheta) + \lambda\,\nabla_{\mtheta}\,\cL_{\text{cons}}(\mtheta) \nonumber \\
&= (1-\lambda) \cdot 2t^2\,\langle \mmF_{\mtheta} - \vv_{\text{data}}, \nabla_{\mtheta}\mmF_{\mtheta} \rangle 
+ \lambda \cdot 2t^2\,\langle \mmF_{\mtheta} - \vv_{\text{fake}}, \nabla_{\mtheta}\mmF_{\mtheta} \rangle \nonumber \\
&= 2t^2\,\Bigl[(1-\lambda)\,\langle \mmF_{\mtheta} - \vv_{\text{data}}, \nabla_{\mtheta}\mmF_{\mtheta} \rangle 
+ \lambda\,\langle \mmF_{\mtheta} - \vv_{\text{fake}}, \nabla_{\mtheta}\mmF_{\mtheta} \rangle\Bigr].
\label{eq:grad_gapex_final}
\end{align}

\paragraph{Part C: Conclusion.}

Comparing Eq.~\eqref{eq:grad_lmix_final} and Eq.~\eqref{eq:grad_gapex_final}, we see they are identical:
\begin{equation}
\nabla_{\mtheta}\,\cL_{\text{mix}}(\mtheta) = \nabla_{\mtheta}\,\cG_{\text{APEX}}(\mtheta).
\end{equation}
This completes the proof of Theorem~\ref{thm:gradient_equivalence}.
\end{proof}

\subsection{Fisher Divergence Perspective}
\label{app:fisher}

We provide an interpretation of APEX's alternative loss through the lens of Fisher divergence. 
This analysis reveals that APEX minimizes a score-space distance with uniform weighting, contrasting with GAN based objectives that use sample dependent weights.

\begin{proposition}[APEX as Fisher Divergence Minimization]
\label{prop:fisher}
The alternative loss $\cG_{\text{APEX}}(\mtheta)$ can be interpreted as minimizing a weighted Fisher divergence to a mixed distribution. 
Specifically, define the mixed score function
\begin{equation}
\ss_{\text{mix}}(\xx_t) := (1-\lambda)\,\ss_{\text{data}}(\xx_t) + \lambda\,\ss_{\text{fake}}(\xx_t),
\label{eq:score_mix_def}
\end{equation}
where $\ss_{\text{data}}(\xx_t) = \nabla_{\xx_t}\log p_{\text{data},t}(\xx_t)$ and $\ss_{\text{fake}}(\xx_t) = \nabla_{\xx_t}\log p_{\text{fake},t}(\xx_t)$ are the score functions corresponding to the data distribution and fake distribution at time $t$, respectively. 
Then, up to time dependent weighting $\omega(t)$,
\begin{equation}
\nabla_{\mtheta}\,\cG_{\text{APEX}}(\mtheta) 
\propto \EEb{\xx_t\sim p_{\mtheta,t}}{(\ss_{\mtheta}(\xx_t) - \ss_{\text{mix}}(\xx_t)) \cdot \frac{\partial\xx_t}{\partial\mtheta}},
\label{eq:apex_fisher_grad}
\end{equation}
which corresponds to minimizing the Fisher divergence
\begin{equation}
D_F(p_{\mtheta} \| p_{\text{mix}}) 
:= \int \bigl\|\ss_{\mtheta}(\xx_t) - \ss_{\text{mix}}(\xx_t)\bigr\|_2^2\,p_{\mtheta}(\xx_t)\,\dm\xx_t.
\label{eq:fisher_div_def}
\end{equation}
\end{proposition}

\begin{proof}
\textbf{Step 1: Relate velocity differences to score differences.}
By Corollary~\ref{cor:velocity_score} (Eq.~\eqref{eq:velocity_score_diff}), the velocity-score relationship gives
\begin{align}
\mmF_{\mtheta} - \vv_{\text{data}} 
&= -\frac{t}{1-t}\,(\ss_{\mtheta}(\xx_t) - \ss_{\text{data}}(\xx_t)), \label{eq:vel_score_1} \\
\mmF_{\mtheta} - \vv_{\text{fake}} 
&= -\frac{t}{1-t}\,(\ss_{\mtheta}(\xx_t) - \ss_{\text{fake}}(\xx_t)). \label{eq:vel_score_2}
\end{align}
\emph{Derivation reminder:} These equations follow from applying the score-velocity duality
\[
\ss_t(\xx_t) = -\frac{\xx_t + (1-t)\vv(\xx_t,t)}{t}
\]
to each pair of velocity fields. For instance, for Eq.~\eqref{eq:vel_score_1}:
\begin{align*}
\ss_{\mtheta}(\xx_t) - \ss_{\text{data}}(\xx_t) 
&= -\frac{\xx_t + (1-t)\mmF_{\mtheta}}{t} + \frac{\xx_t + (1-t)\vv_{\text{data}}}{t} \\
&= \frac{(1-t)(\vv_{\text{data}} - \mmF_{\mtheta})}{t} \\
&= -\frac{1-t}{t}(\mmF_{\mtheta} - \vv_{\text{data}}).
\end{align*}
Rearranging gives Eq.~\eqref{eq:vel_score_1}.

\textbf{Step 2: Form the linear combination.}
From the proof of Theorem~\ref{thm:gradient_equivalence} (Eq.~\eqref{eq:grad_gapex_final}), the gradient of $\cG_{\text{APEX}}$ involves the weighted sum
\begin{equation}
(1-\lambda)\,(\mmF_{\mtheta} - \vv_{\text{data}}) + \lambda\,(\mmF_{\mtheta} - \vv_{\text{fake}}).
\label{eq:velocity_combo}
\end{equation}
Now we substitute the velocity-score relationships from Step 1. Substituting Eq.~\eqref{eq:vel_score_1} and Eq.~\eqref{eq:vel_score_2}:
\begin{align}
&(1-\lambda)\,(\mmF_{\mtheta} - \vv_{\text{data}}) + \lambda\,(\mmF_{\mtheta} - \vv_{\text{fake}}) \nonumber \\
&= (1-\lambda)\,\bigl[-\frac{t}{1-t}\,(\ss_{\mtheta} - \ss_{\text{data}})\bigr] 
+ \lambda\,\bigl[-\frac{t}{1-t}\,(\ss_{\mtheta} - \ss_{\text{fake}})\bigr] \nonumber \\
&= -\frac{t}{1-t}\,\bigl[(1-\lambda)\,(\ss_{\mtheta} - \ss_{\text{data}}) 
+ \lambda\,(\ss_{\mtheta} - \ss_{\text{fake}})\bigr] \quad \text{(factor out $-\frac{t}{1-t}$)} \nonumber \\
&= -\frac{t}{1-t}\,\bigl[(1-\lambda)\,\ss_{\mtheta} - (1-\lambda)\,\ss_{\text{data}} 
+ \lambda\,\ss_{\mtheta} - \lambda\,\ss_{\text{fake}}\bigr] \quad \text{(expand)} \nonumber \\
&= -\frac{t}{1-t}\,\bigl[[(1-\lambda) + \lambda]\,\ss_{\mtheta} - (1-\lambda)\,\ss_{\text{data}} - \lambda\,\ss_{\text{fake}}\bigr] \nonumber \\
&= -\frac{t}{1-t}\,\bigl[\ss_{\mtheta} - \bigl((1-\lambda)\,\ss_{\text{data}} + \lambda\,\ss_{\text{fake}}\bigr)\bigr] \quad \text{(since $(1-\lambda) + \lambda = 1$)} \nonumber \\
&= -\frac{t}{1-t}\,\bigl[\ss_{\mtheta}(\xx_t) - \ss_{\text{mix}}(\xx_t)\bigr],
\label{eq:velocity_to_score_mix}
\end{align}
where in the last line we used the definition of the mixed score function from Eq.~\eqref{eq:score_mix_def}:
\[
\ss_{\text{mix}}(\xx_t) := (1-\lambda)\,\ss_{\text{data}}(\xx_t) + \lambda\,\ss_{\text{fake}}(\xx_t).
\]

\textbf{Step 3: Write the gradient in score-space form.}
From Eq.~\eqref{eq:grad_gapex_final}, the gradient of $\cG_{\text{APEX}}$ is
\begin{align}
\nabla_{\mtheta}\,\cG_{\text{APEX}}(\mtheta)
&= 2t^2\,\EEb{\xx_t,\zz,t}{\left\langle (1-\lambda)\,(\mmF_{\mtheta} - \vv_{\text{data}}) 
+ \lambda\,(\mmF_{\mtheta} - \vv_{\text{fake}}), \nabla_{\mtheta}\mmF_{\mtheta} \right\rangle}.
\end{align}
Substituting Eq.~\eqref{eq:velocity_to_score_mix}:
\begin{align}
\nabla_{\mtheta}\,\cG_{\text{APEX}}(\mtheta)
&= 2t^2\,\EEb{\xx_t,\zz,t}{\left\langle -\frac{t}{1-t}\,(\ss_{\mtheta} - \ss_{\text{mix}}), \nabla_{\mtheta}\mmF_{\mtheta} \right\rangle} \nonumber \\
&= -\frac{2t^3}{1-t}\,\EEb{\xx_t,\zz,t}{\left\langle (\ss_{\mtheta}(\xx_t) - \ss_{\text{mix}}(\xx_t)), \nabla_{\mtheta}\mmF_{\mtheta} \right\rangle}.
\label{eq:grad_apex_score}
\end{align}

\textbf{Step 4: Relate to Fisher divergence.}
The Fisher divergence between the model distribution $p_{\mtheta}$ and a target distribution $p_{\text{mix}}$ is defined as
\begin{equation}
D_F(p_{\mtheta} \| p_{\text{mix}}) 
= \int \bigl\|\ss_{\mtheta}(\xx_t) - \ss_{\text{mix}}(\xx_t)\bigr\|_2^2\,p_{\mtheta}(\xx_t)\,\dm\xx_t.
\label{eq:fisher_integral}
\end{equation}
Taking the gradient with respect to $\mtheta$ using the score identity $\nabla_{\xx}\log p_{\mtheta} = \ss_{\mtheta}$ and the path-wise gradient estimator:
\begin{equation}
\nabla_{\mtheta}\,D_F 
\propto \EEb{\xx_t\sim p_{\mtheta}}{(\ss_{\mtheta}(\xx_t) - \ss_{\text{mix}}(\xx_t)) \cdot \frac{\partial\xx_t}{\partial\mtheta}}.
\label{eq:fisher_grad}
\end{equation}

\textbf{Step 5: Absorb time dependent factors.}
The coefficient $-\frac{2t^3}{1-t}$ in Eq.~\eqref{eq:grad_apex_score} depends only on time $t$, not on the spatial position $\xx_t$ or the sample. 
This factor can be absorbed into the time weighting $\omega(t)$ used in the expectation. 
Thus, up to a time dependent proportionality constant,
\begin{equation}
\nabla_{\mtheta}\,\cG_{\text{APEX}}(\mtheta) 
\propto \EEb{\xx_t\sim p_{\mtheta,t}}{(\ss_{\mtheta}(\xx_t) - \ss_{\text{mix}}(\xx_t)) \cdot \frac{\partial\xx_t}{\partial\mtheta}},
\end{equation}
which matches the form of the Fisher divergence gradient in Eq.~\eqref{eq:fisher_grad}.
\end{proof}

\paragraph{Contrast with GAN objectives.}
For reference, we note that classical GAN objectives involve sample dependent weights. 
The non saturating GAN gradient takes the form
\begin{equation}
\nabla_{\mtheta}\,\cL_{\text{NS-GAN}} 
\propto \EEb{\xx_t\sim p_{\mtheta}}{w_{\text{NS}}(\xx_t)\,(\ss_{\mtheta}(\xx_t) - \ss_{\text{data}}(\xx_t)) \cdot \frac{\partial\xx_t}{\partial\mtheta}},
\label{eq:gan_ns_grad}
\end{equation}
where the weight $w_{\text{NS}}(\xx_t) = 1 - D^*(\xx_t) = \frac{p_{\mtheta}(\xx_t)}{p_{\text{data}}(\xx_t) + p_{\mtheta}(\xx_t)}$ depends on the optimal discriminator $D^*(\xx_t)$. 
This sample dependent weight can become very small (when $D^* \approx 1$, i.e., generated samples are perfect) or very large (when $D^* \approx 0$, i.e., generated samples are easily distinguished), leading to gradient instability.
In contrast, APEX's gradient in Eq.~\eqref{eq:apex_fisher_grad} has a \emph{uniform} weight across samples (the time dependent factor $\omega(t)$ is constant for all $\xx_t$ at a given $t$). 
This structural property ensures stable training signals throughout the learning process, independent of the current quality of generated samples.

\newpage

\section{Visualizations Part I}
\label{app:visualizations}

This section provides additional qualitative results to complement the quantitative analysis in the main paper.

\begin{figure}[!h]
  \centering
  \begin{subfigure}{\textwidth}
      \centering
      \includegraphics[width=\linewidth]{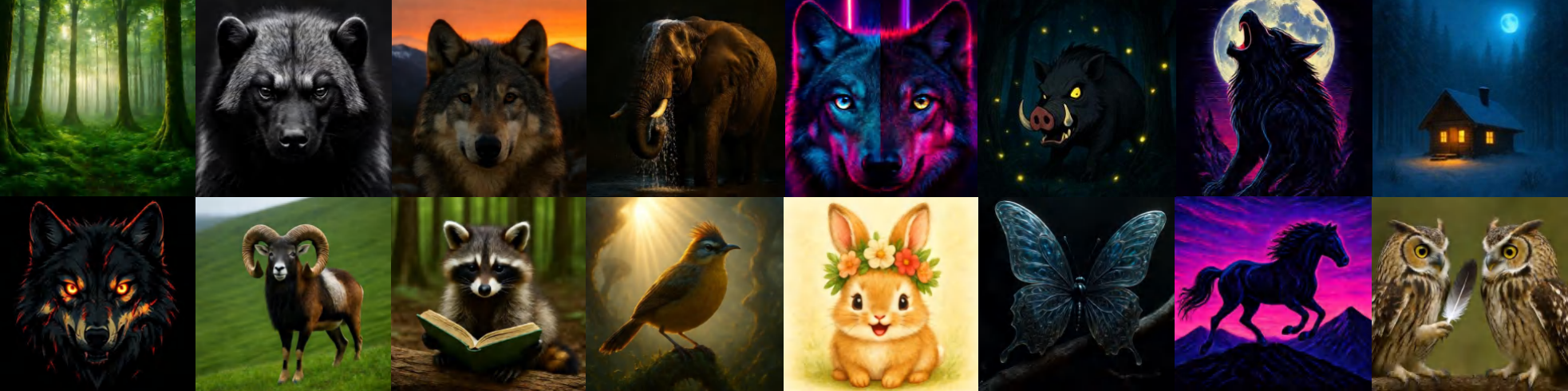}
  \end{subfigure}
  \caption{\small
      \textbf{Qualitative Comparison of 512x512 in APEX 20B LoRA for NFE=1.} 
  }
  \label{fig:nfe_comparison_appendix}
  \vspace{-10pt}
\end{figure}

\begin{figure}[!h]
    \centering
    \begin{subfigure}{\textwidth}
        \centering
        \includegraphics[width=\linewidth]{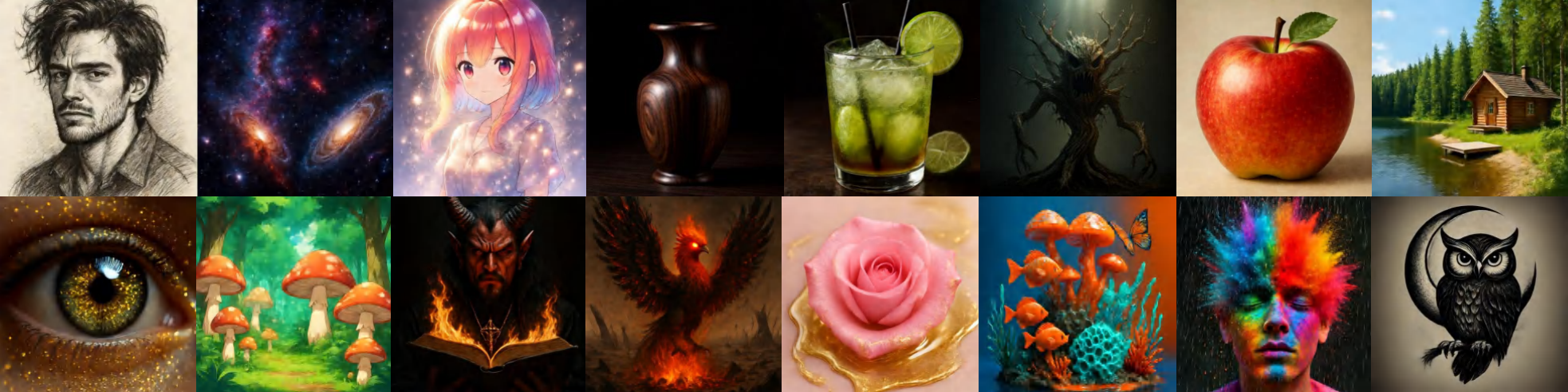}
    \end{subfigure}
    \caption{\small
        \textbf{Qualitative Comparison of 512x512 in APEX 20B LoRA for NFE=1.} 
    }
    \vspace{-10pt}
\end{figure}

\begin{figure}[!htbp]
    \centering
    \begin{subfigure}{\textwidth}
        \centering
        \includegraphics[width=\linewidth]{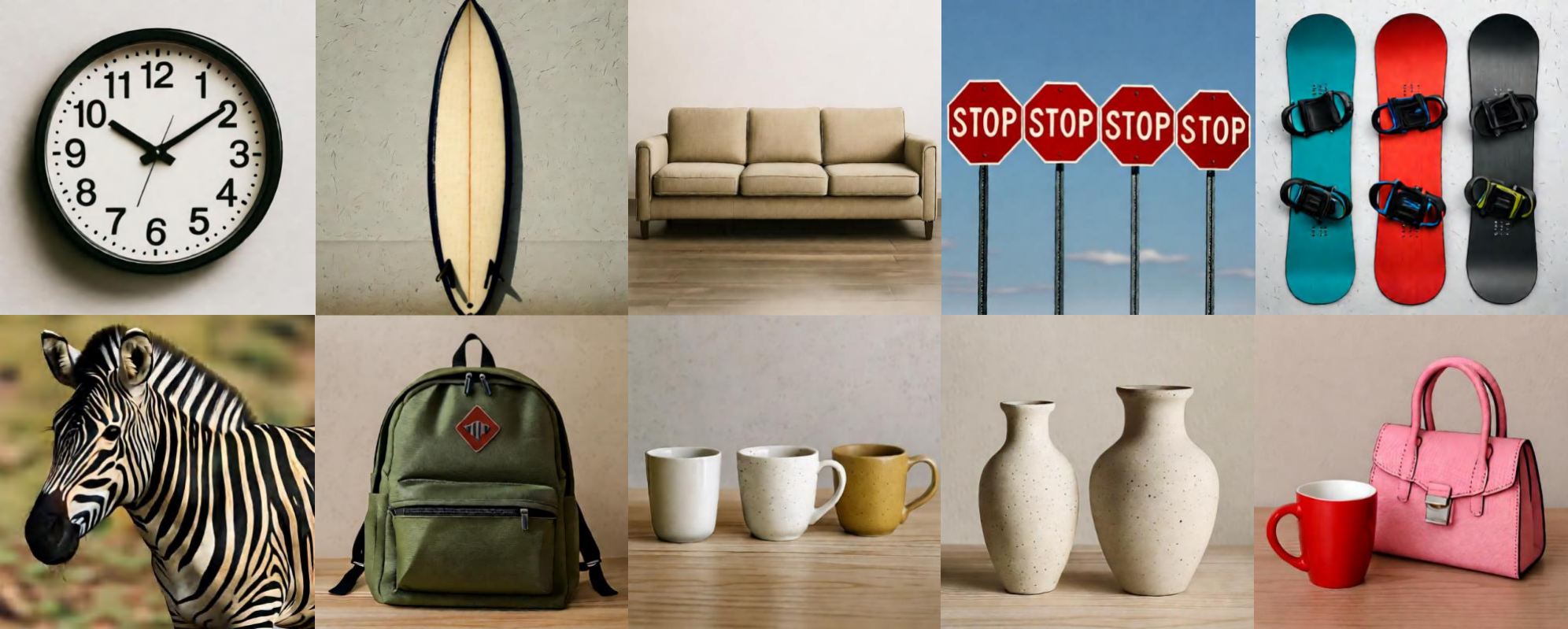}
    \end{subfigure}
    \caption{\small
        \textbf{Qualitative Comparison of 512x512 in APEX 20B LoRA for NFE=1.} 
    }
    \vspace{-10pt}
\end{figure}



\begin{figure}[!t]
    \centering
    \begin{subfigure}{\textwidth}
        \centering
        \includegraphics[width=\linewidth]{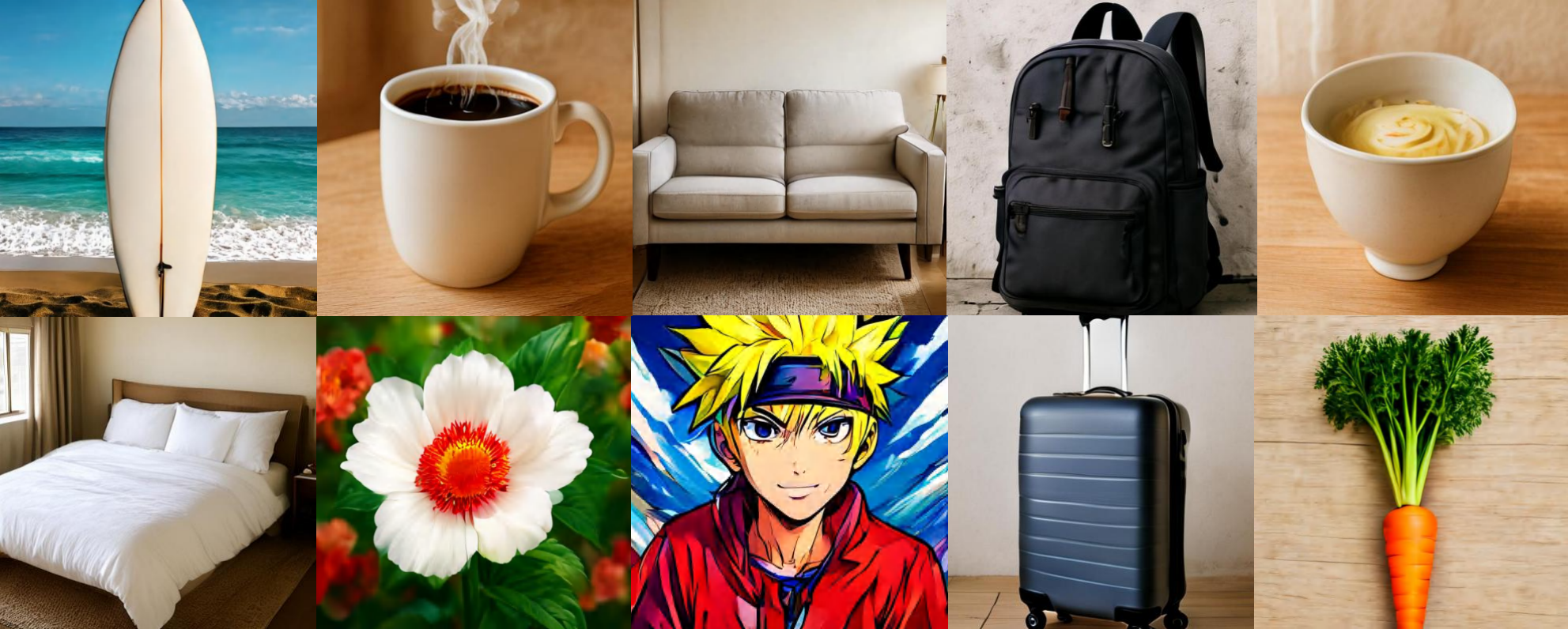}
    \end{subfigure}
    \caption{\small
        \textbf{Qualitative Comparison of 512x512 in APEX 0.6B LoRA for NFE=1.}
    }
    \vspace{-10pt}
\end{figure}

\newpage

\section{Visualizations Part II}
\label{app:visualizations_ii}

\begin{figure}[!h]
    \centering
    \begin{subfigure}{\textwidth}
        \centering
        \includegraphics[width=\linewidth]{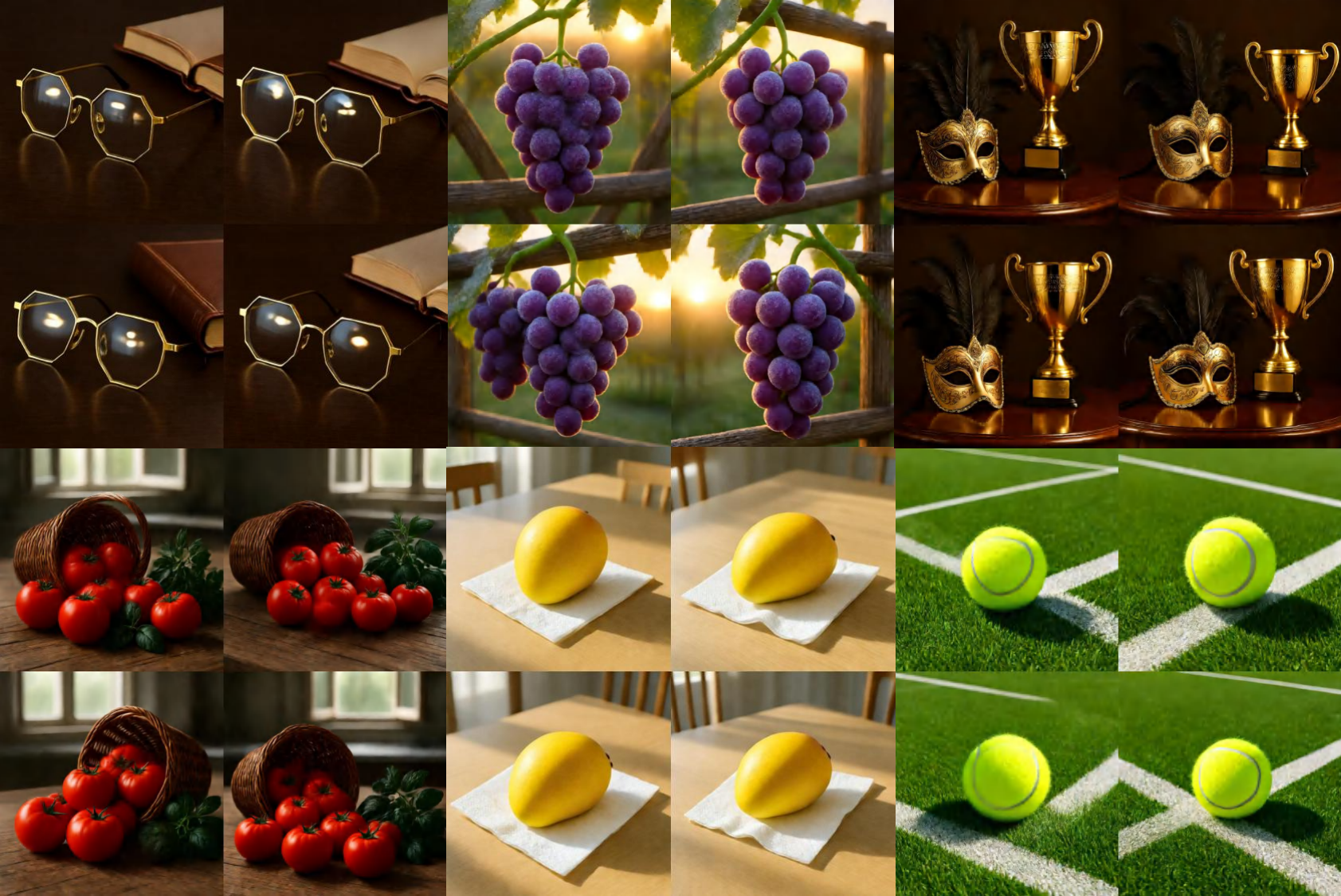}
    \end{subfigure}
    \caption{\small
        \textbf{Qualitative Comparison of 512x512 in APEX 20B LoRA for NFE=1.} 
    }
    \vspace{-10pt}
\end{figure}

\begin{figure}[!htbp]
    \centering
    \begin{subfigure}{\textwidth}
        \centering
        \includegraphics[width=\linewidth]{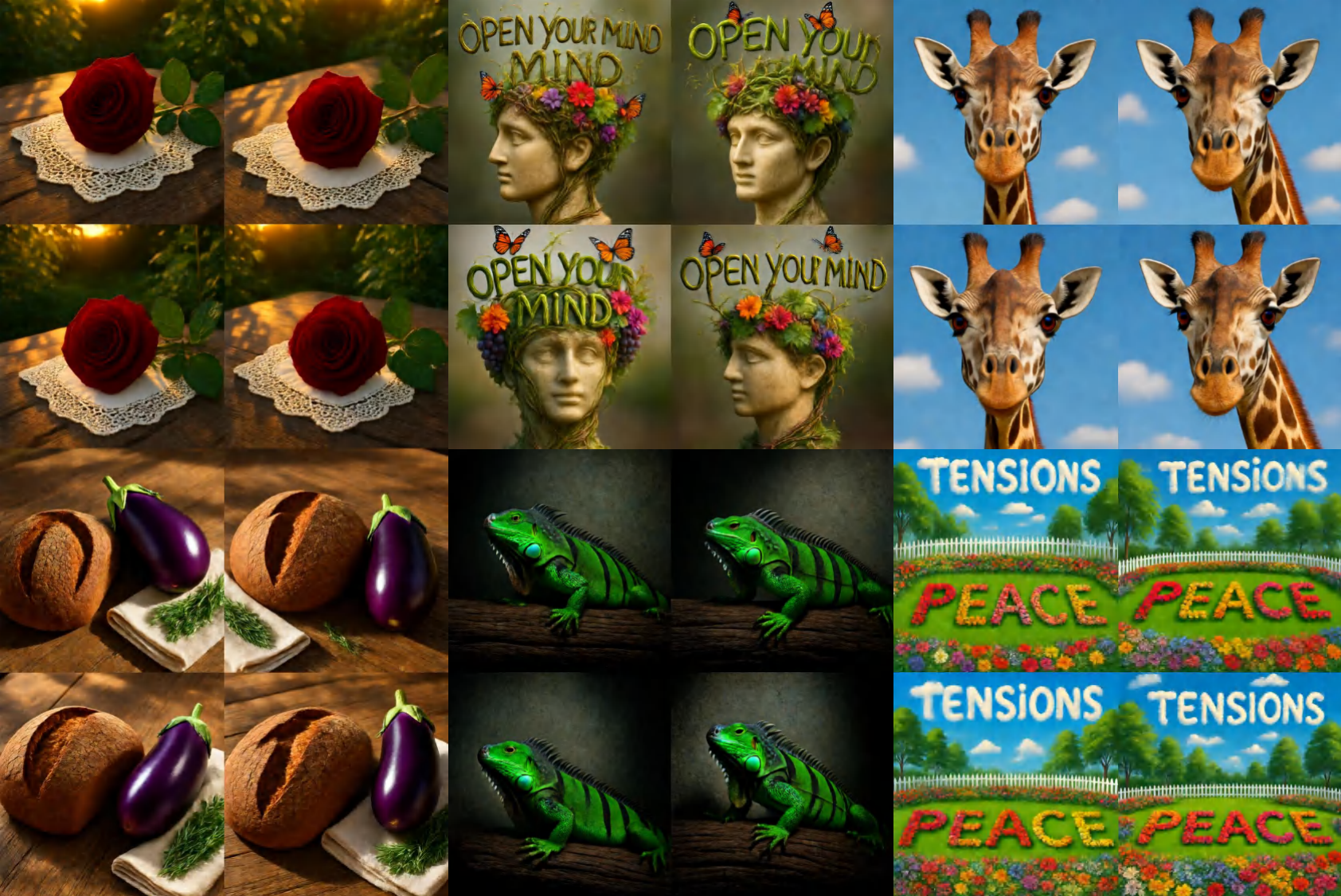}
    \end{subfigure}
    \caption{\small
        \textbf{Qualitative Comparison of 512x512 in APEX 20B LoRA for NFE=1.} 
    }
    \vspace{-10pt}
\end{figure}

\begin{figure}[!htbp]
    \centering
    \begin{subfigure}{\textwidth}
        \centering
        \includegraphics[width=\linewidth]{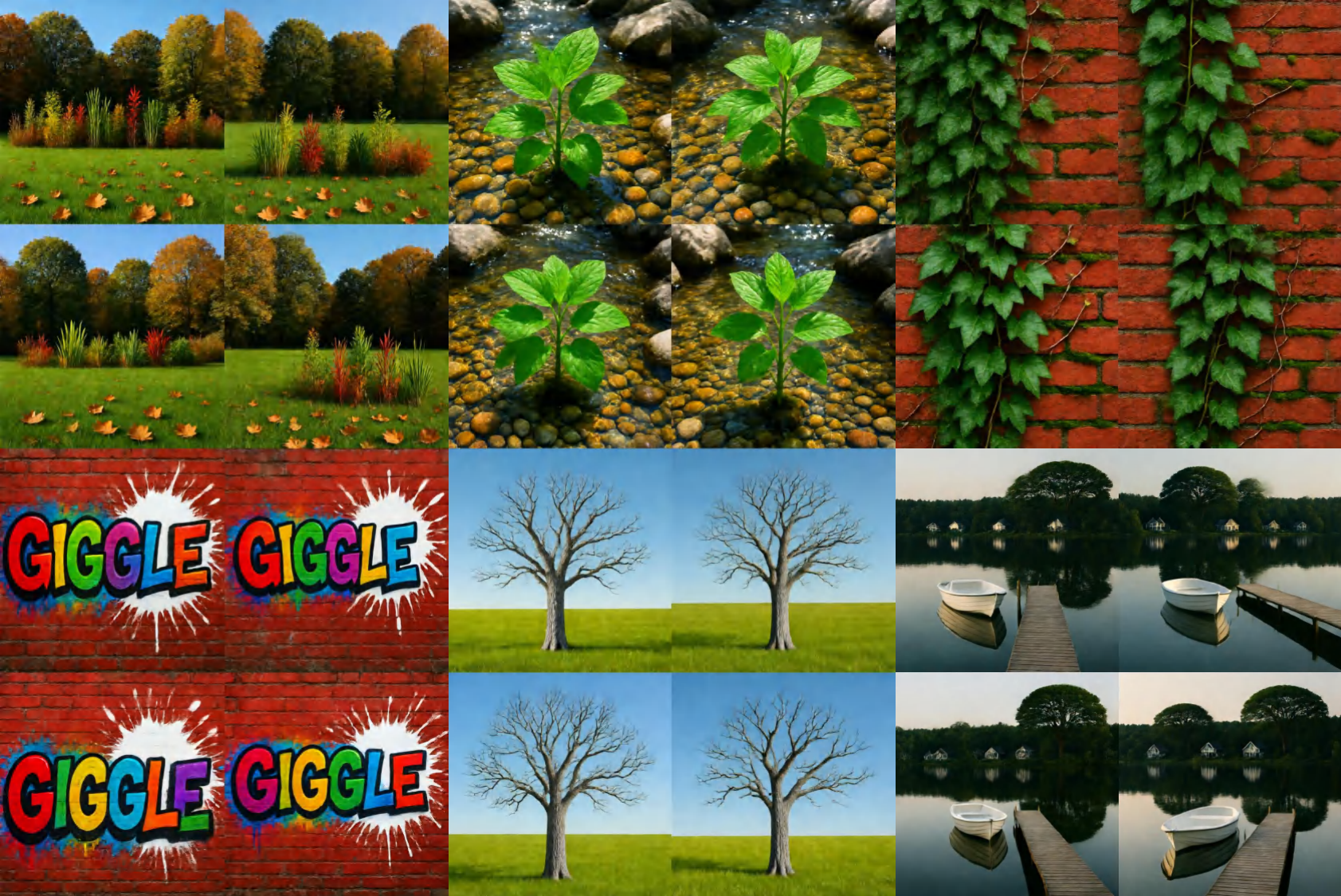}
    \end{subfigure}
    \caption{\small
        \textbf{Qualitative Comparison of 512x512 in APEX 20B LoRA for NFE=1.} 
    }
    \vspace{-10pt}
\end{figure}

\begin{figure}[!htbp]
    \centering
    \begin{subfigure}{\textwidth}
        \centering
        \includegraphics[width=\linewidth]{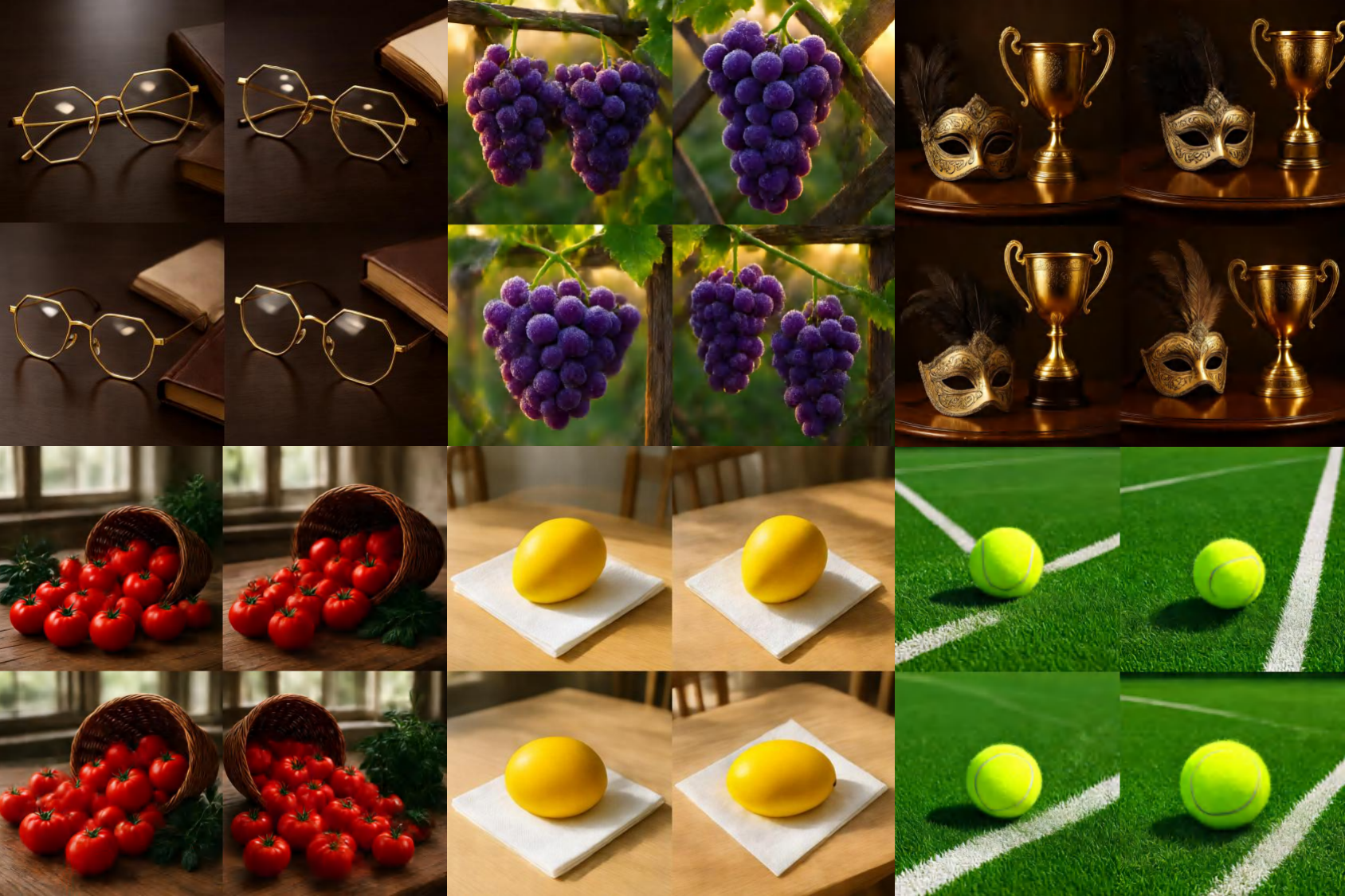}
    \end{subfigure}
    \caption{\small
        \textbf{Qualitative Comparison of 512x512 in APEX 20B Full Parameter Tuning for NFE=1.} 
    }
    \vspace{-10pt}
\end{figure}

\begin{figure}[!htbp]
    \centering
    \begin{subfigure}{\textwidth}
        \centering
        \includegraphics[width=\linewidth]{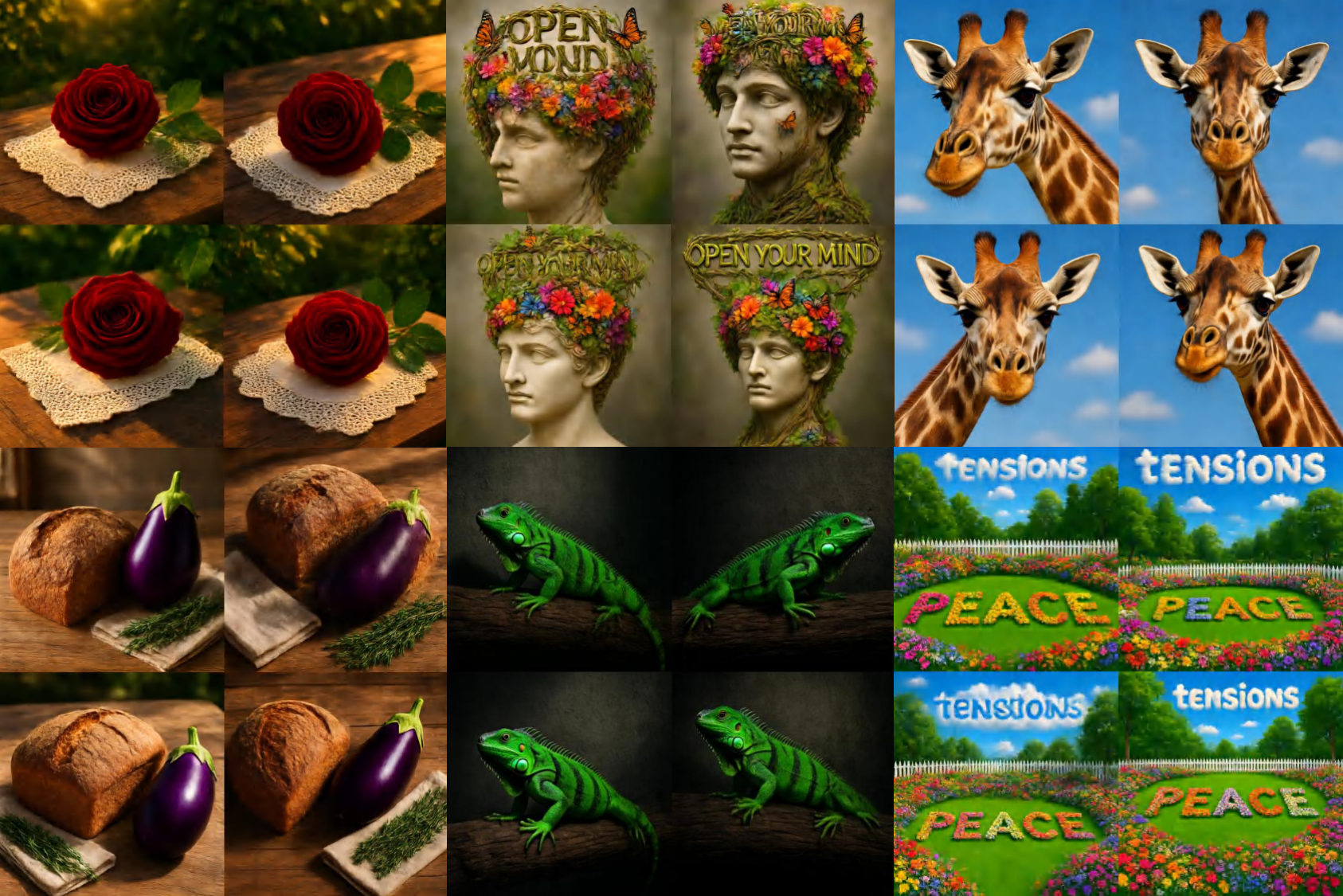}
    \end{subfigure}
    \caption{\small
        \textbf{Qualitative Comparison of 512x512 in APEX 20B Full Parameter Tuning for NFE=1.} 
    }
    \vspace{-10pt}
\end{figure}

\begin{figure}[!htbp]
    \centering
    \begin{subfigure}{\textwidth}
        \centering
        \includegraphics[width=\linewidth]{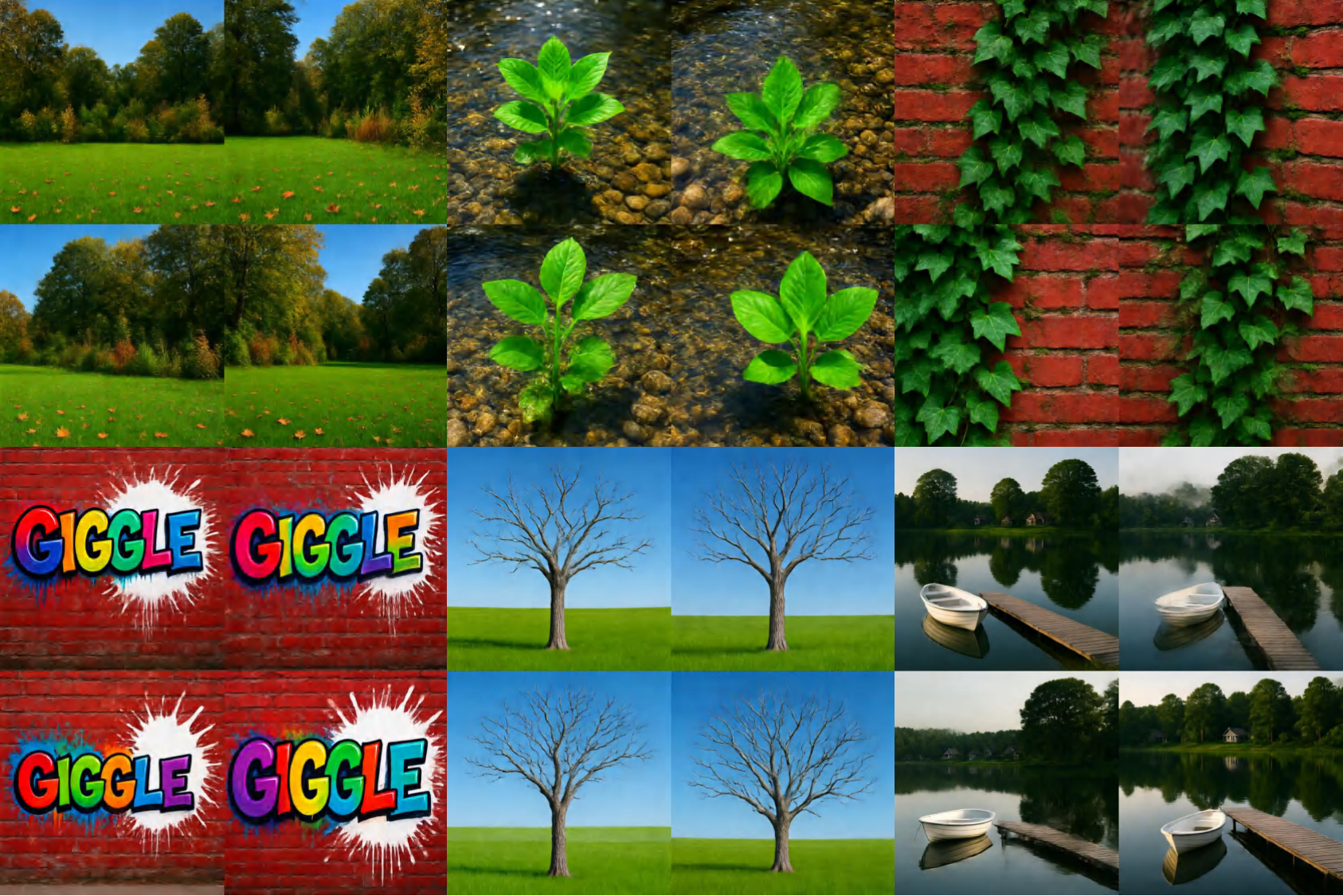}
    \end{subfigure}
    \caption{\small
        \textbf{Qualitative Comparison of 512x512 in APEX 20B Full Parameter Tuning for NFE=1.} 
    }
    \vspace{-10pt}
\end{figure}

\begin{figure}[!htbp]
    \centering
    \begin{subfigure}{\textwidth}
        \centering
        \includegraphics[width=\linewidth]{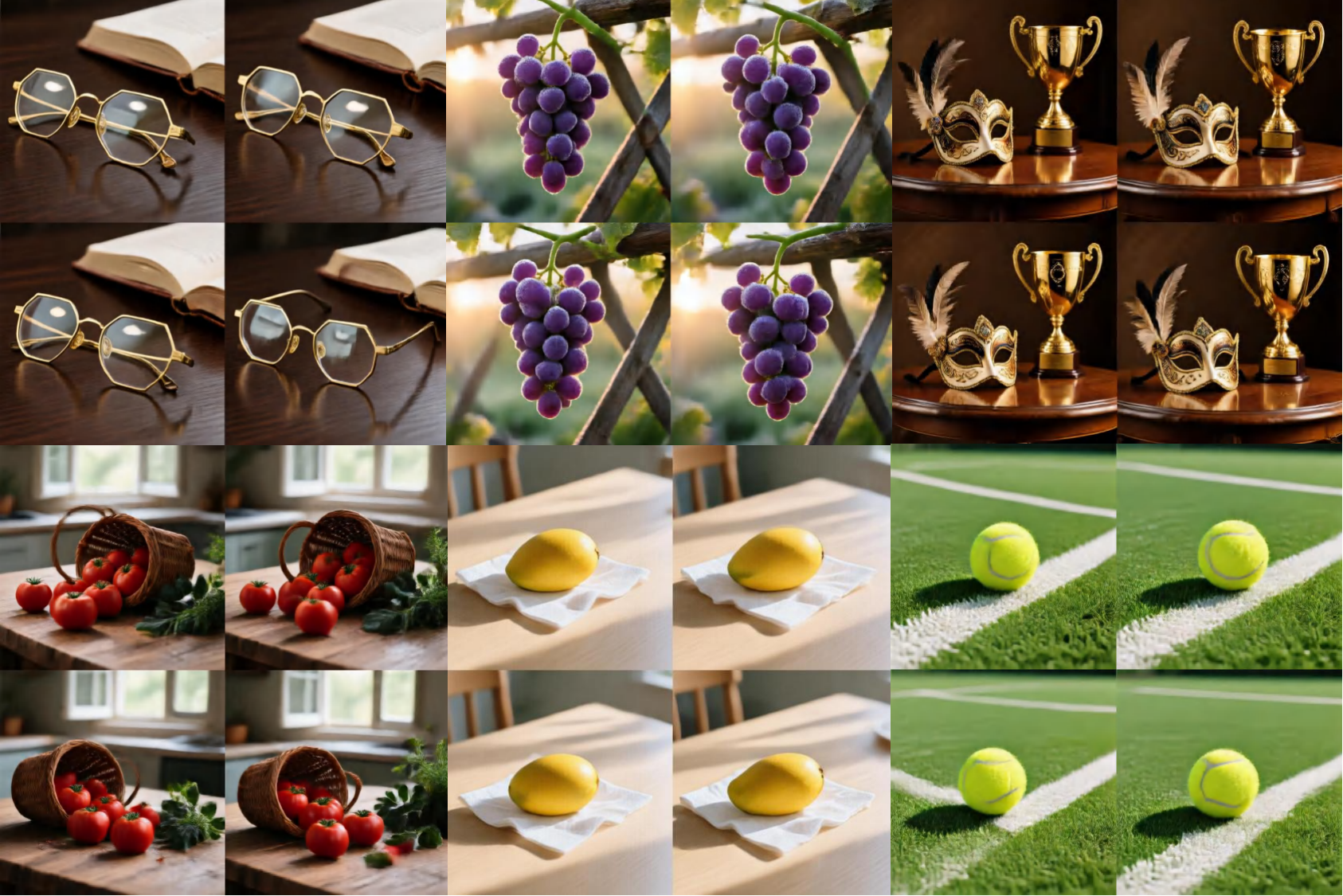}
    \end{subfigure}
    \caption{\small
        \textbf{Qualitative Comparison of 512x512 in Qwen-Image Lightning LoRA for NFE=1.} 
    }
    \vspace{-10pt}
\end{figure}

\begin{figure}[!htbp]
    \centering
    \begin{subfigure}{\textwidth}
        \centering
        \includegraphics[width=\linewidth]{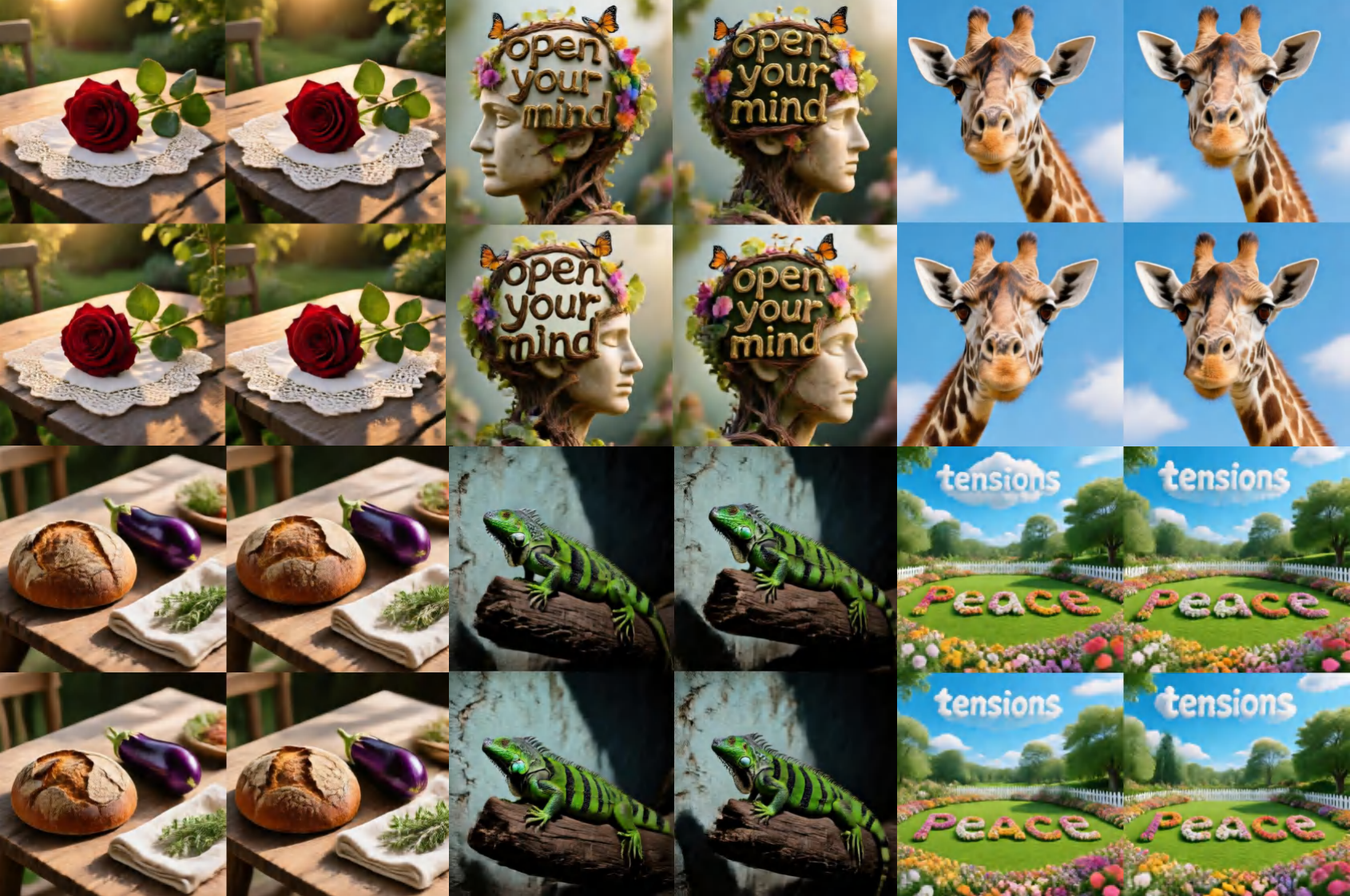}
    \end{subfigure}
    \caption{\small
        \textbf{Qualitative Comparison of 512x512 in Qwen-Image Lightning LoRA for NFE=1.} 
    }
    \vspace{-10pt}
\end{figure}

\begin{figure}[!htbp]
    \centering
    \begin{subfigure}{\textwidth}
        \centering
        \includegraphics[width=\linewidth]{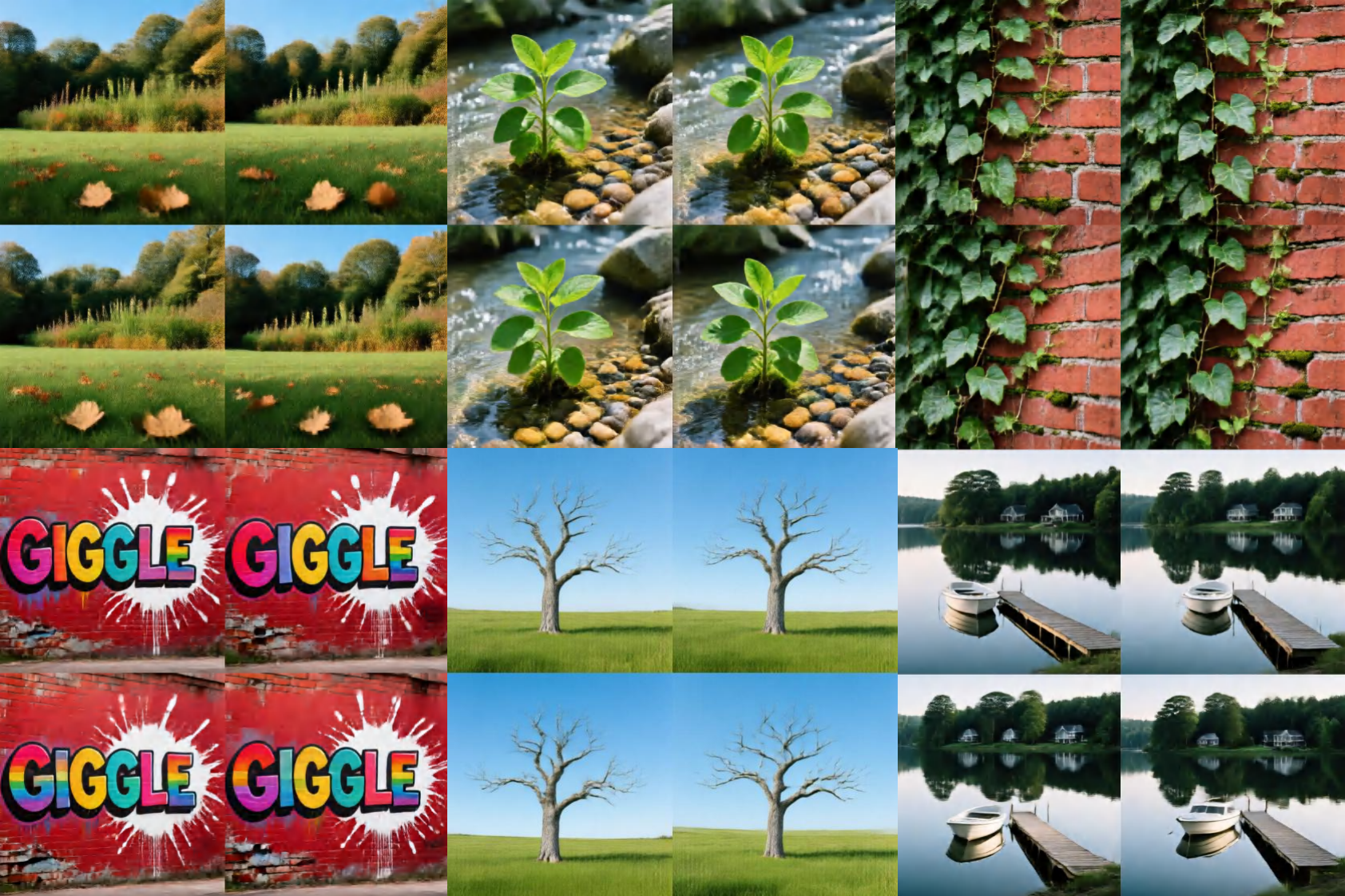}
    \end{subfigure}
    \caption{\small
        \textbf{Qualitative Comparison of 512x512 in Qwen-Image Lightning LoRA for NFE=1.} 
    }
    \vspace{-10pt}
\end{figure}

\section{Visualizations Part III}
\label{app:visualizations_iii}

\newpage

\begin{figure}[!htbp]
    \centering
    \begin{subfigure}{\textwidth}
        \centering
        \includegraphics[width=\linewidth]{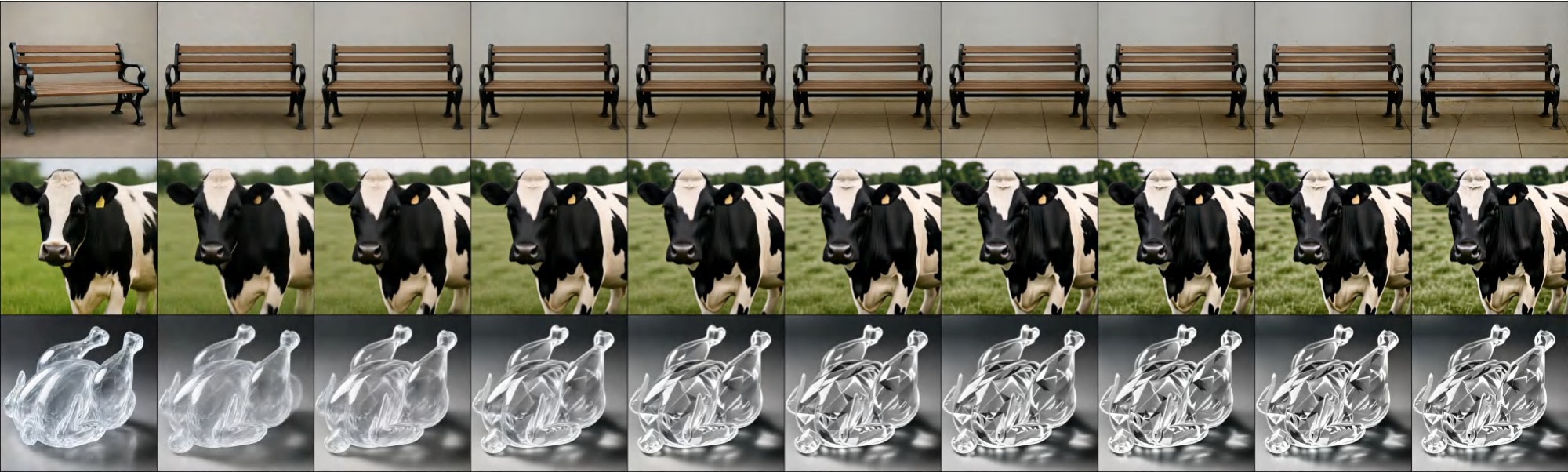}
    \end{subfigure}
    \caption{\small
        \textbf{Qualitative Comparison of 512x512 in 20B Full Parameter Tuning of APEX methods and Synthetic dataset from NFE=1 to NFE=20.}
    }
    \vspace{-10pt}
\end{figure}

\begin{figure}[!htbp]
    \centering
    \begin{subfigure}{\textwidth}
        \centering
        \includegraphics[width=\linewidth]{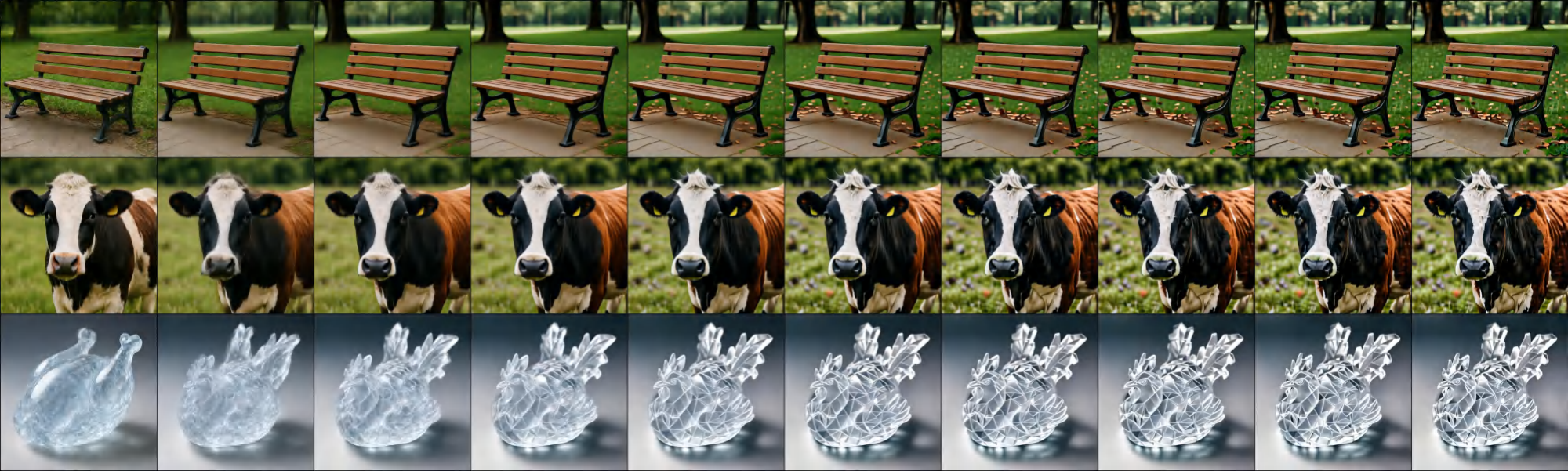}
    \end{subfigure}
    \caption{\small
        \textbf{Qualitative Comparison of 512x512 in 20B Full Parameter Tuning of APEX methods and BLIP-3o dataset from NFE=1 to NFE=20.} 
    }
    \vspace{-10pt}
\end{figure}

\begin{figure}[!htbp]
    \centering
    \begin{subfigure}{\textwidth}
        \centering
        \includegraphics[width=\linewidth]{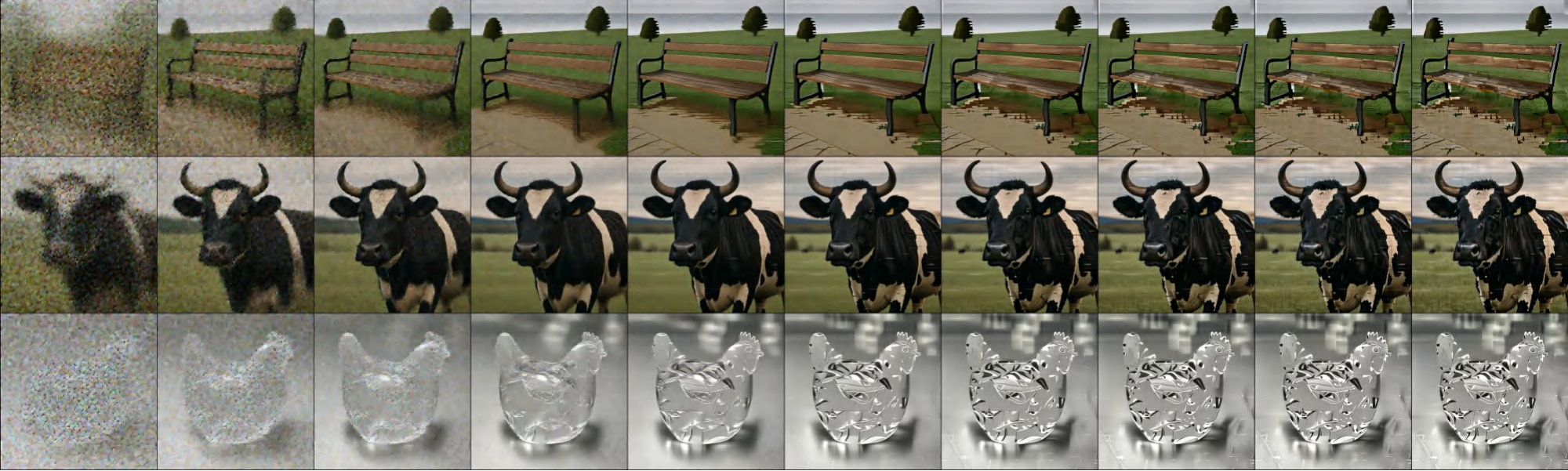}
    \end{subfigure}
    \caption{\small
        \textbf{Qualitative Comparison of 512x512 in 20B Full Parameter Tuning of sCM methods and BLIP-3o dataset from NFE=1 to NFE=20.}
    }
    \vspace{-10pt}
\end{figure}

\begin{figure}[!h]
    \centering
    \begin{subfigure}{\textwidth}
        \centering
        \includegraphics[width=\linewidth]{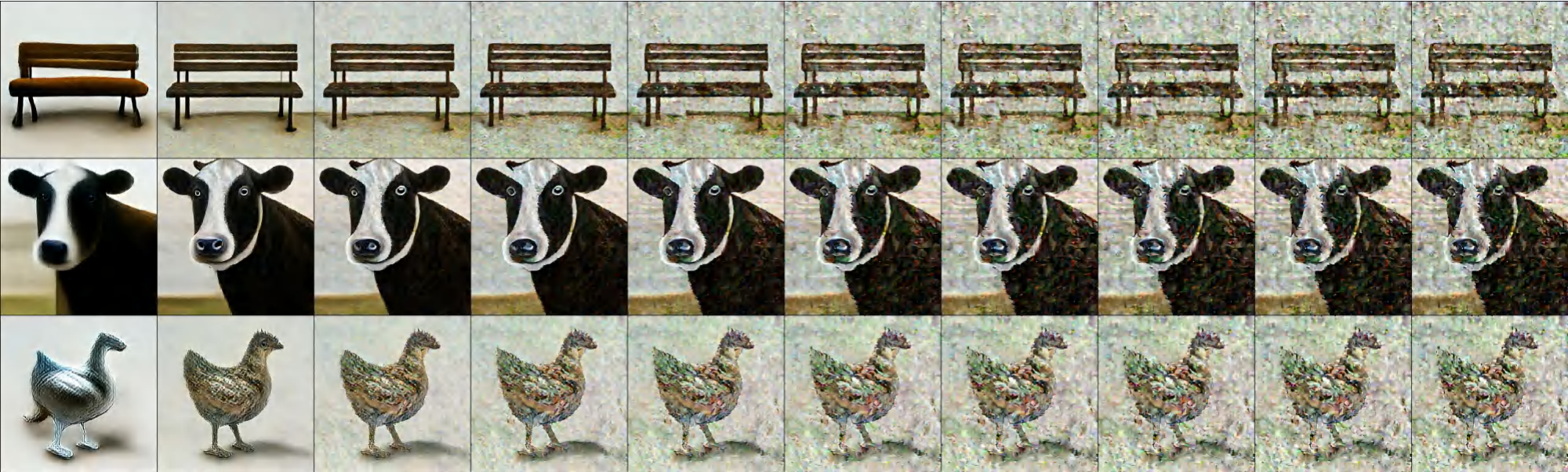}
    \end{subfigure}
    \caption{\small
        \textbf{Qualitative Comparison of 512x512 in 20B Full Parameter Tuning of CTM methods and BLIP-3o dataset from NFE=1 to NFE=20.}
    }
    \vspace{-10pt}
\end{figure}

\begin{figure}[!htbp]
    \centering
    \begin{subfigure}{\textwidth}
        \centering
        \includegraphics[width=\linewidth]{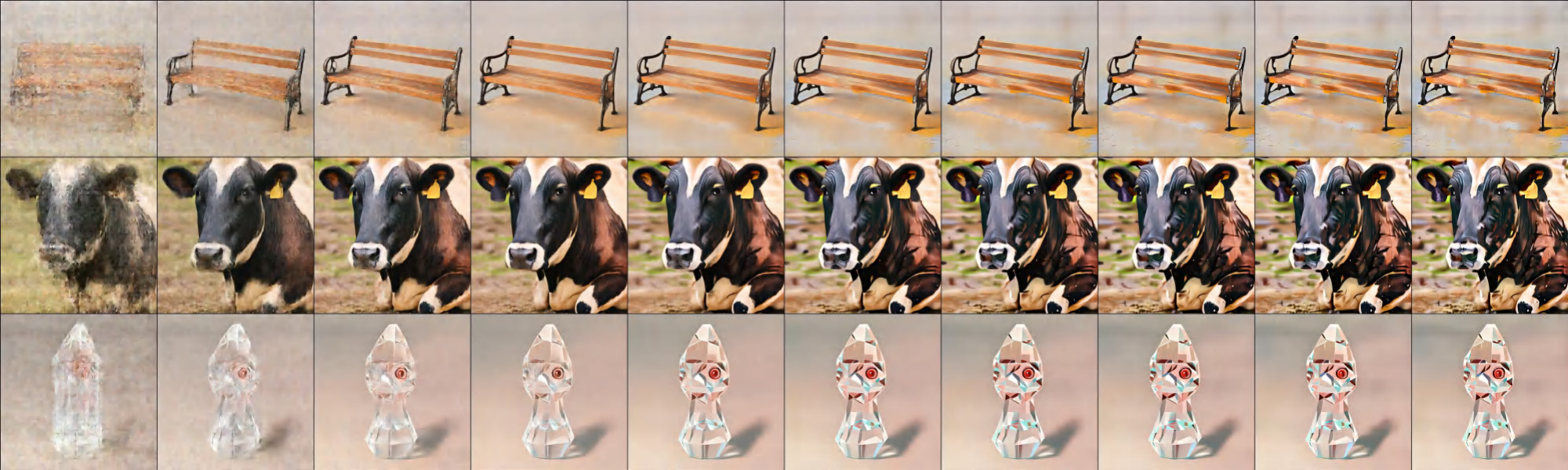}
    \end{subfigure}
    \caption{\small
        \textbf{Qualitative Comparison of 512x512 in 20B Full Parameter Tuning of MeanFlow methods and BLIP-3o dataset from NFE=1 to NFE=20.}
    }
    \vspace{-10pt}
\end{figure}

\end{document}